\documentclass{elsarticle}
\usepackage{hyperref}
\usepackage{subfigure}
\usepackage{amssymb}
\usepackage{color}
\usepackage{fullpage}

\newtheorem{theorem}{Theorem}

\newtheorem{definition}{Definition}

\newcommand\shrink[1]{}

\def\n(#1){\bar{#1}}

\def\pr{{\it Pr}}

\def\E{{\bf E}}
\def\e{{\bf e}}

\def\U{{\bf U}}
\def\u{{\bf u}}

\def\X{{\bf X}}
\def\x{{\bf x}}
\def\Y{{\bf Y}}
\def\y{{\bf y}}
\def\Z{{\bf Z}}

\def\eql(#1,#2){{#1\!\!=\!#2}}

\def\eql(#1,#2){{#1\!=\!#2}}

\def\eproof{{\hfill$\Box$}}
\newenvironment{proof}[1][Proof]
			   {\begin{trivlist}\item[\hskip \labelsep {\bfseries #1}]}
			   {~\eproof \end{trivlist}}

\def\pr{{P}}
\def\pstar{P^\star}
\def\Pstar{{\cal P}^\star}
\def\sumnode(#1,#2){{+(#1,#2)}}
\def\prodnode(#1,#2){{*(#1,#2)}}
\def\testnode(#1,#2,#3,#4){{[#1 \geq  #2 \: ?\:  #3 : #4]}}

\begin{document}

\begin{frontmatter}

\title{On the Relative Expressiveness of Bayesian and Neural Networks$^\star$\footnote{${}^\star$ This article is an extended and revised version of \cite{ChoiDarwiche18}, with more theoretical and experimental results.}}

\author[uclaaddress]{Arthur Choi}
\ead{aychoi@cs.ucla.edu}
\author[zjuaddress]{Ruocheng Wang}
\ead{rchwang@outlook.com}
\author[uclaaddress]{Adnan Darwiche}
\ead{darwiche@cs.ucla.edu}

\address[uclaaddress]{Computer Science Department, University of California, Los Angeles, USA}
\address[zjuaddress]{College of Computer Science and Technology, Zhejiang University, Hangzhou, Zhejiang, China}

\begin{abstract}
A neural network computes a function.  A central property of neural networks is that they are ``universal approximators:'' for a given continuous function, there exists a neural network that can approximate it arbitrarily well, given enough neurons (and some additional assumptions).  In contrast, a Bayesian network is a model, but each of its queries can be viewed as computing a function.  In this paper, we identify some key distinctions between the functions computed by neural networks and those by marginal Bayesian network queries, showing that the former are more expressive than the latter.  Moreover, we propose a simple augmentation to Bayesian networks (a testing operator), which enables their marginal queries to become ``universal approximators.''
\end{abstract}

\begin{keyword}
Bayesian networks, neural networks, arithmetic circuits, function approximation
\end{keyword}

\end{frontmatter}

\setcounter{footnote}{0}

\section{Introduction}

The field of artificial intelligence (AI) has seen two major milestones throughout its history. Shortly after the field was born in the 1950s, the focus turned to {\em symbolic, model-based} approaches, which were premised on the need to represent and reason with domain knowledge, and exemplified by the use of logic to represent such knowledge~\citep{McCarthy59}.  In the 1980s, the focus turned to {\em probabilistic, model-based} approaches, as exemplified by Bayesian networks and probabilistic graphical models more generally (first major milestone)~\citep{Pearl88b}.  Starting in the 1990s, and as data became abundant, these probabilistic models provided the foundation for much of the research in machine learning, where models were learned either generatively or discriminatively from data.  Recently, the field shifted its focus to {\em numeric, function-based} approaches, as exemplified by neural networks, which are trained discriminatively using labeled data (deep learning, second major milestone)~\citep{Goodfellow-et-al-2016,HintonOT06,BengioLPL06,RanzatoPCL06,rosenblatt1958perceptron,mcculloch1943logical}.  Perhaps the biggest surprise with the second milestone is the extent to which certain tasks, associated with perception or limited forms of cognition, can be approximated using functions (i.e., neural networks) that are learned purely from labeled data, without the need for modeling or reasoning~\citep{darwicheCACM18}.

While this evolution of the field has increased our abilities, the emerging techniques have been pursued by somewhat independent research communities.  The price has been a lack of enough integration and fusion of the various methods, which hinders the realization of their collective benefits.  {\em Logic} provides a rich framework for representing knowledge in the form of domain constraints and comes with profound reasoning mechanisms.  {\em Probabilistic graphical models} excel at capturing uncertainty, causal knowledge, and independence information.  Both of these frameworks provide a foundation for capturing domain knowledge of various types and for reasoning
with such knowledge. 
{\em Neural networks} are effective function approximators, allowing one to approximate specific and narrow tasks by simply fitting a complex function (i.e., deep neural network) to data in the form of input-output pairs---again,
without the need for modeling or reasoning.

Each of these frameworks has its shortcomings though.  
Symbolic models are too coarse to capture certain phenomena and, in their pure form, miss out on exploiting data and the wealth of information it may contain.  
Probabilistic graphical models, in their generative form, address this limitation especially when they integrate symbolic knowledge. However, in their discriminative form,
these models have been outperformed by neural networks as a realization of the function-based approach to AI. 
While most of the recent success in AI has been due to neural networks, we are now starting to realize their limits too: 
they are data hungry, may not generalize beyond the given data, can be quite brittle, and pose challenges for explanation and verification. 
Interestingly, it is these shortcomings that models, whether symbolic or probabilistic, can help alleviate.  

Hence, a key challenge and opportunity for AI today lies in {\em fusing} these approaches to realize their collective benefits.\footnote{We distinguish between integration and fusion. {\em Integration} may refer to an intelligent agent architecture in which the components are based on different approaches, but work together to complement each other. {\em Fusion} may refer to a model-based approach that is empowered by functions, or a function-based approach that is empowered by models. Our focus in this paper is on fusion, particularly the empowerment of function-based approaches with domain knowledge (i.e., models).}
We take some initial steps towards this goal in this paper by showing how models can empower the function-based approach to AI.
We focus on probabilistic graphical models in the form of Bayesian networks, but our interest is ultimately in models that combine probabilistic and symbolic 
knowledge; see, e.g., \citep{ShenChoiDarwiche18,KisaVCD14,Poole2003,halpern1990analysis,nilsson1986probabilistic}. 
We start with a set of observations that are known in the literature but that together lead to a dilemma. 
Our main contribution is a resolution to this dilemma, which we argue has major implications on the quest of fusing model-based and function-based approaches to AI.

Our observations are as follows.
First, a query posed to a Bayesian network model can be viewed as inducing (and evaluating) a function, which can be represented using an Arithmetic Circuit (AC) \citep{Darwiche03,ChoiDarwiche17}.  
Next, Bayesian network queries (and, hence, ACs) can be trained discriminatively using labeled data and gradient descent methods, leading to a realization of the function-based approach currently 
practiced using neural networks---except that a neural network represents only one function, while a Bayesian network represents many functions (one for each query). 
Finally, the functions induced by Bayesian network queries can integrate background knowledge of various forms, suggesting a more principled, if not more sophisticated, function-based approach compared to neural networks.
Now the dilemma: recent developments in AI indicate that neural networks outperform Bayesian networks, and probabilistic graphical models more generally, as a realization of the function-based approach,
despite the ability of the latter to integrate background knowledge into the learned functions.

Our first contribution towards resolving this dilemma is highlighting the following observation:
Bayesian and neural networks induce classes of functions that differ in expressiveness.  It is known that neural networks are {\em universal approximators,} which means they can approximate any function to an arbitrary small error. However, the functions induced by marginal Bayesian network queries are less expressive as they correspond to multi-linear functions or quotients of such functions (this is also known but was never discussed in this context). Our second contribution: We address this expressiveness gap by proposing a simple extension to Bayesian networks and showing that this extension can also induce functions that are universal approximators. When the newly induced functions are represented by circuits and trained using labeled data, we obtain a function-based approach that is as expressive as neural networks, but that can also 
integrate the knowledge embodied in a Bayesian network. 
In other words, we can now synthesize function structures from models and then learn their parameters from labeled data. 
The resulting function-based approach can therefore benefit from what models have to offer: 
less dependence on data, improved generality and robustness, better prospects for explainability and the potential for more direct formal guarantees on behavior. 

This paper is structured as follows. We review in Section~\ref{sec:functions} the class of functions induced by neural and Bayesian networks, while identifying the corresponding gap in expressiveness. We then propose a new class of Bayesian networks in Section~\ref{sec:tbn}, {\em Testing Bayesian Networks~(TBNs),} whose queries can be computed using {\em Testing Arithmetic Circuits~(TACs)} that we discuss in
Section~\ref{sec:tac}. We show in Section~\ref{sec:uat} that these functions are universal approximators and reveal their specific functional form in Section~\ref{sec:f-form}.  
We then show some experimental results in Section~\ref{sec:exp} and conclude in Section~\ref{sec:conclusion}.
The appendices include the method we used to train TACs from labeled data, the full proof of the universal approximation theorem, and an algorithm for compiling TBN queries into TACs. 

\section{Technical Background}
\label{sec:functions}

We next review the class of functions represented by neural networks. We also highlight previous results, allowing us to view Bayesian network queries as inducing and evaluating functions. 
The goal is to pinpoint an expressiveness gap between the two classes of functions, which we address in Section~\ref{sec:tbn}.

\subsection{Neural Networks as Functions}

\begin{figure}[tb]
 \centering
 \subfigure[A neural network structure]{\label{fig:structure}
   \includegraphics[width=0.3\linewidth]{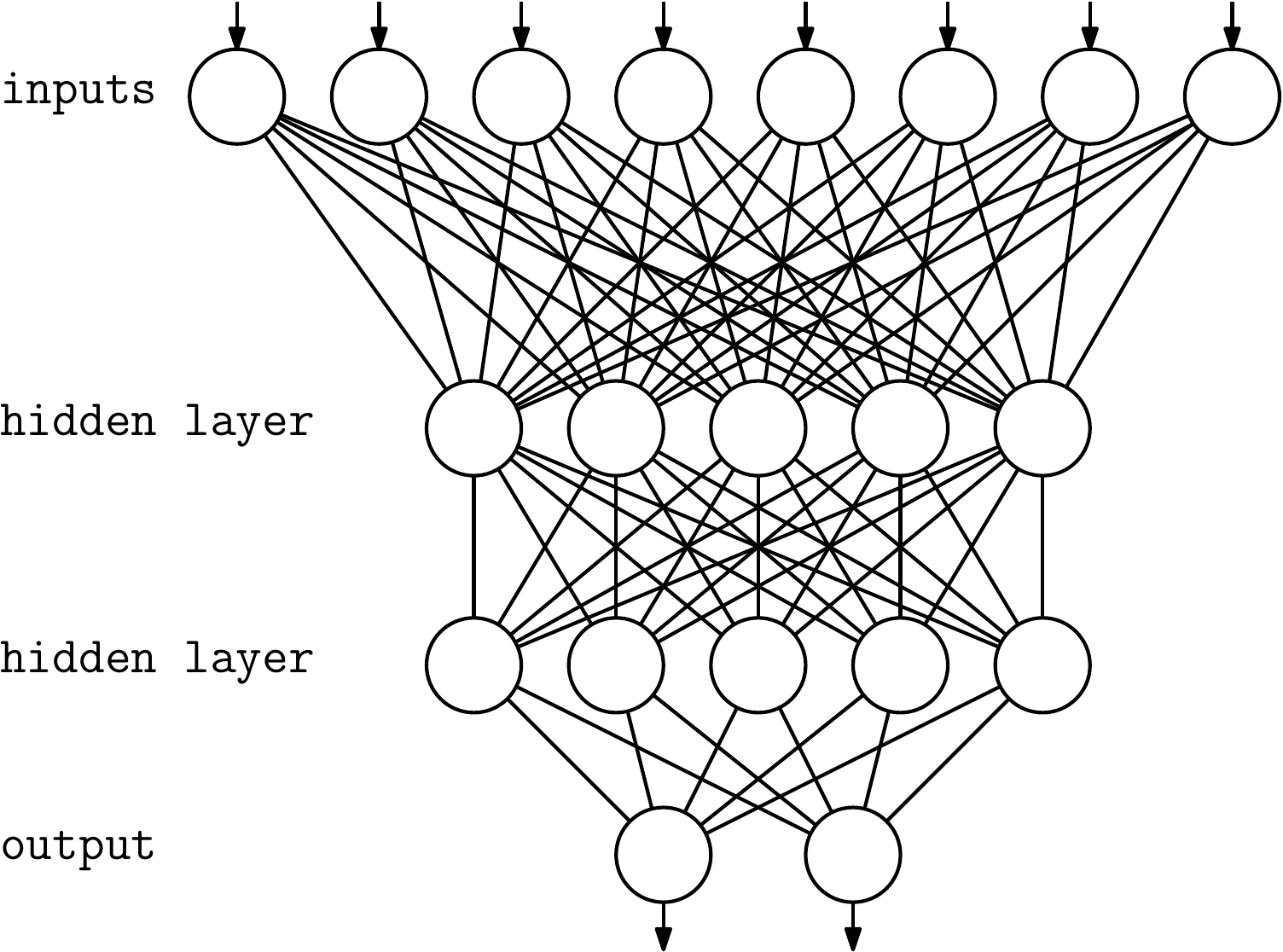}}
   \qquad
 \subfigure[A mathematical model of a neuron]{\label{fig:neuron}
   \raisebox{.5\height}{
     \includegraphics[width=0.3\linewidth]{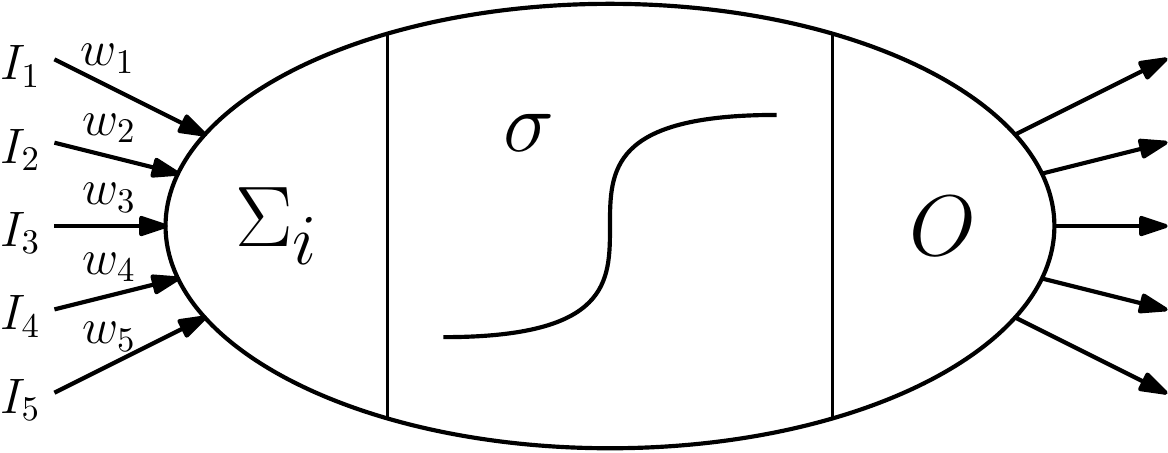}}}
\caption{A neural network and a neuron. \label{fig:neural-network}}
\end{figure}

\begin{figure}[tb]
  \centering
 \includegraphics[width=0.25\linewidth]{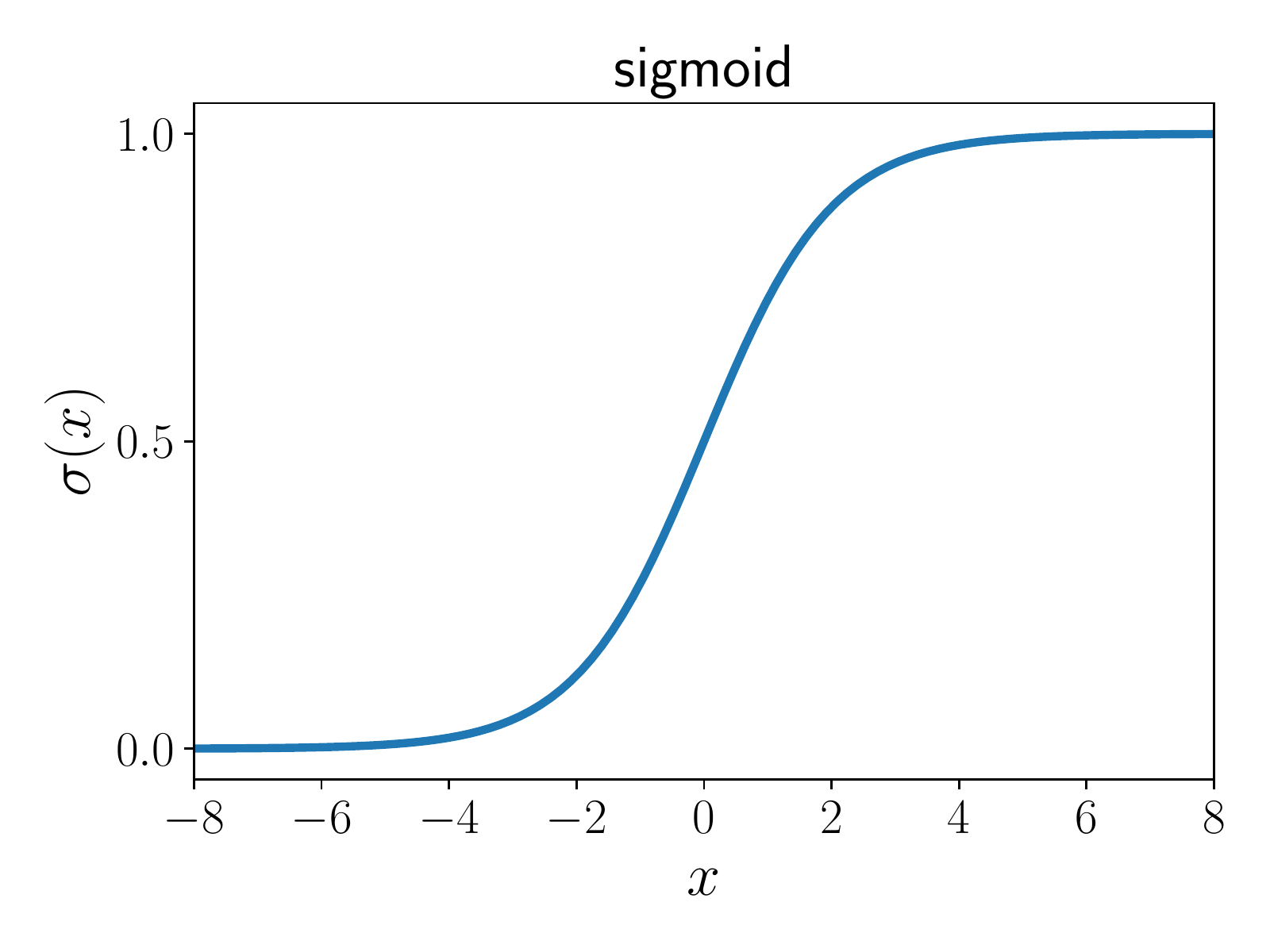}
 \qquad \qquad
 \includegraphics[width=0.25\linewidth]{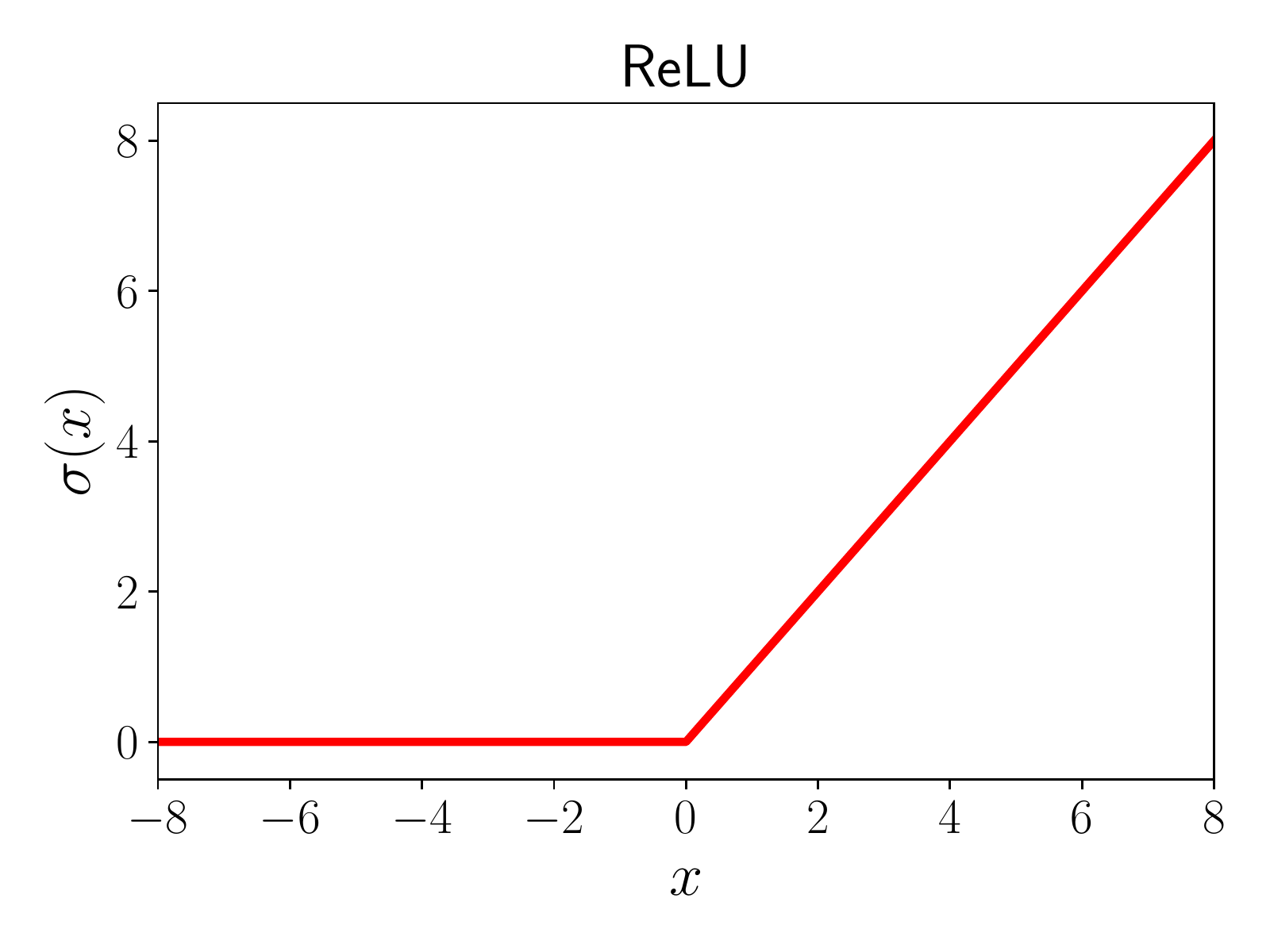}
\caption{Two activation functions.  A sigmoid \(\sigma(x) = \frac{1}{1+\exp\{-x\}}\) acts as a soft threshold which tends to 0 as \(x\) goes to \(-\infty\) and tends to 1 as \(x\) goes to \(\infty\).  A ReLU \(\sigma(x) = \max(0,x)\) is equal to 0 if \(x < 0\) and is equal to \(x\) otherwise. \label{fig:activation}}
\end{figure}

A (feedforward) neural network is a directed acyclic graph (DAG); see Figure~\ref{fig:structure}. The roots of the DAG are the neural network inputs, call them \(X_1, \ldots, X_n\).  The leaves of the DAG are the neural network outputs, call them \(Y_1, \ldots, Y_m\).  Each node in the DAG is called a {\em neuron} and contains an {\em activation function} \(\sigma\); see Figure~\ref{fig:neuron}.  Each edge \(I\) in the DAG has a {\em weight} \(w\) attached to it. The weights of a neural network are its {\em parameters,} which are learned from data. Consider a neuron with activation function \(\sigma\), inputs \(I_{i}\) and  corresponding weights \(w_{i}\). The output of this neuron is simply \(\sigma(\sum_{i} w_{i} \cdot I_{i})\). Thus, one can compute the output \(Y_j\) of a neural network by simply evaluating neurons, parents before children, which can be done in time linear in the neural network size.

To simplify the discussion, we will assume that a neural network has a single output \(Y\). Hence, a neural network represents a function \(Y = f(X_1, \ldots, X_n)\).  A key question here relates to the class of functions that can be represented, or approximated well, by neural networks that use a certain class of activation functions. 
One example is the {\em sigmoid} activation function; see Figure~\ref{fig:activation}.  A neural network with only sigmoid activation functions can approximate any continuous function to within an arbitrary error.\footnote{A {\em shallow} neural network with a single hidden layer is sufficient for universal approximation, but may require exponentially many neurons. A {\em deep} neural network may be more succinct for this purpose though.} 
Such neural networks are called {\em universal approximators} of continuous functions  \citep{HornikSW89,Cybenko89,LeshnoLPS93}.
Most of the recent neural networks are based on the {\em ReLU} activation function, which is simpler than the sigmoid; see Figure~\ref{fig:activation}. Neural networks with ReLUs are also universal approximators of continuous functions~\citep{LeshnoLPS93}.

\subsection{Bayesian Network Queries as Functions}

\begin{figure}[tb]
 \centering
 \subfigure[A Bayesian network with logical constraints (\(0/1\) parameters)]{\label{fig:p-model}
 \includegraphics[width=.3\linewidth]{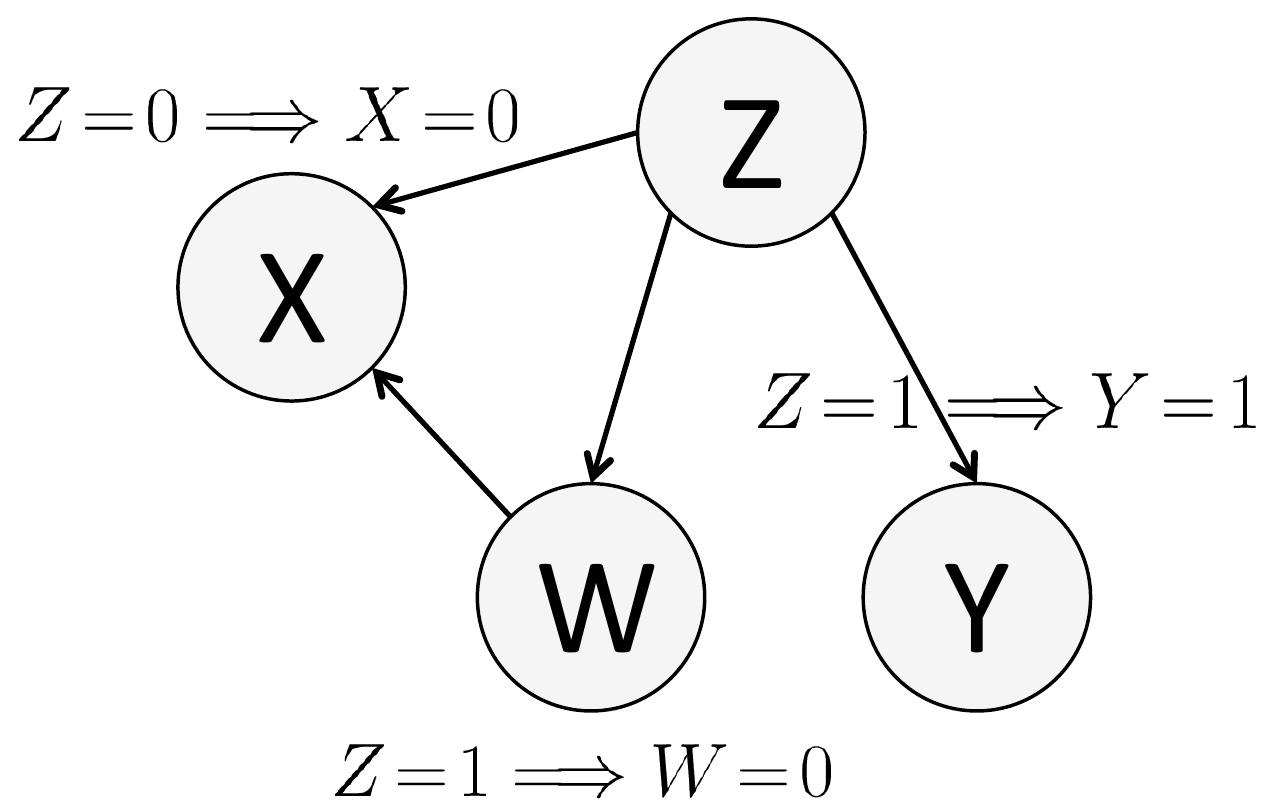}}
\qquad \qquad
 \subfigure[A function that computes the query \(O=\pr(y|x,w)\).  Here, \(\alpha = \pr(y,x,w)\) and \(\beta = \pr(\bar{y},x,w)\)]{\label{fig:s-function}
 \includegraphics[width=.5\linewidth]{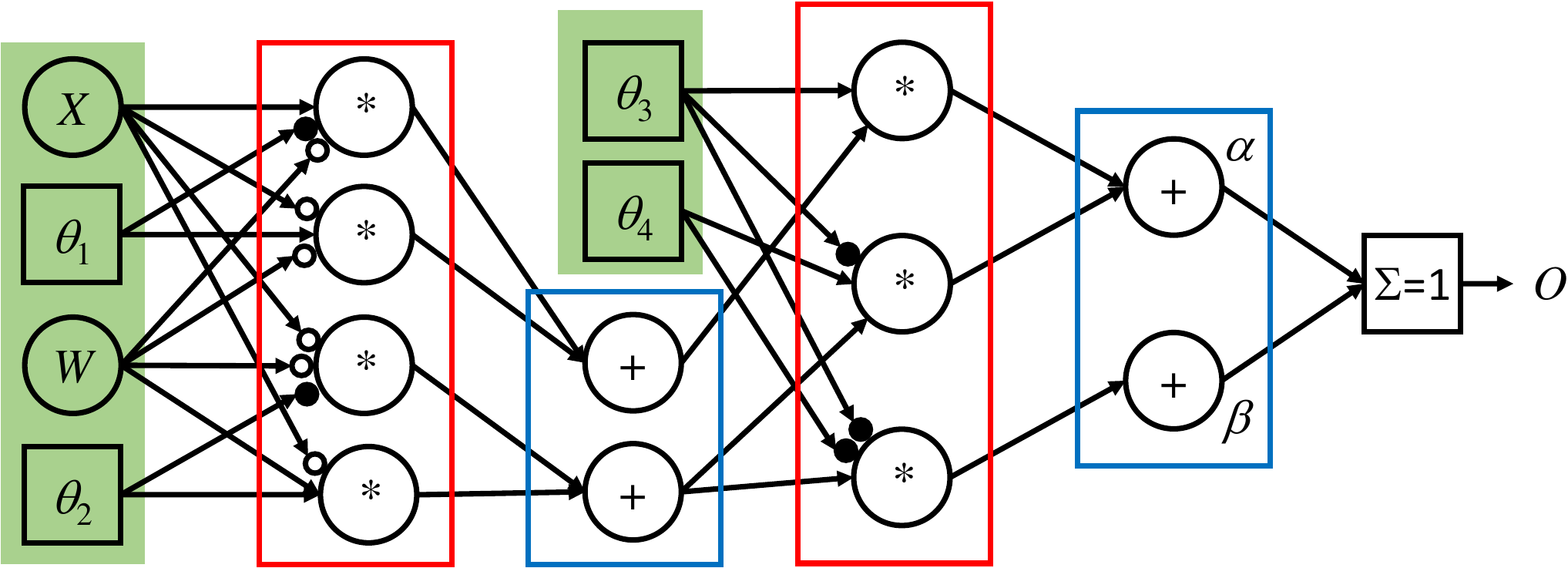}} 
\caption{The function uses adders (\(+\)), multipliers (\(*\)), inverters (\(\circ\)), \(1-\theta\) units (\(\bullet\)), and normalizing units (\(\Sigma = 1\)).  Excluding the \(\Sigma\) unit (division), and emulating \(\circ\) and \(\bullet\) units using adders, we obtain an {Arithmetic Circuit (AC)}~\citep{Darwiche03,ChoiDarwiche17}. The function implicitly integrates the Bayesian network's \(0/1\) parameters. Moreover, \(\theta_1,\ldots,\theta_4\) are Bayesian network parameters that have specific interpretations. \label{fig:function}}
\end{figure}

A Bayesian network is a directed acyclic graph (DAG), where each node is associated with a conditional probability table (CPT). The CPT for node \(X\) with parents \(\U\) contains a parameter 
\(\theta_{x|\u} \in [0,1]\) for each state \(x\) of node \(X\) and state \(\u\) of parents \(\U\), such that \(\sum_x \theta_{x|\u} = 1\). A Bayesian network specifies a unique probability distribution over its nodes~\citep{Pearl88b}.

Consider now a Bayesian network, some evidence \(\e\) on variables \(\E\) (e.g., symptoms), and let \(Y\) be a query variable (e.g., a disease). The probability \(\pr(y,\e)\) can be viewed as the output of a function \(f(\E)\) that maps evidence \(\e\) into a number in \([0,1]\). The function inputs are discrete values of variables \(\E\), but can be continuous values in \([0,1]\) if one uses soft evidence \citep{chanAI05} (universal approximation results for neural networks assume that inputs/outputs are in \([0,1]\)).

The function \(f(\E)\) can be represented by an {\em Arithmetic Circuit (AC)} containing only multipliers and adders~\citep{Darwiche03,ChoiDarwiche17}; see Figure~\ref{fig:function}. Classical inference algorithms for Bayesian networks implicitly construct and evaluate this circuit on the fly~\cite[Chapter~12]{Darwiche09}, while other approaches explicitly construct this circuit through a compilation process~\citep{Darwiche03,Darwiche09}. Consider now the function corresponding to a Bayesian network query \(\pr(\alpha)\).  It is known that  this function is  {\em multi-linear} (more on this later),
regardless of the query \(\alpha\) and the underlying Bayesian network~\citep{Darwiche03,Darwiche09}.\footnote{The conditional probability \(\pr(\alpha|\e)\) is the quotient of two multi-linear functions.}
Hence, the functions induced by Bayesian network queries (and ACs) are less expressive than the ones represented by neural networks. This explains why a Bayesian network that is trained discriminatively using labeled data may not outperform a neural network trained on the same data.
This is particularly so when the data is generated by a function that is not multi-linear.

Here is our major insight for addressing this expressiveness gap, which is based on a simple but consequential observation. 
It is known that if each activation function of a neural network is a polynomial, then the neural network can only represent polynomials~\citep{LeshnoLPS93}. 
Moreover, neural networks with linear activation functions can only represent linear functions.
Consider now the ReLU activation function \(\sigma(x) = \max(0,x)\), which leads to a universal approximator. This function equals 0 if \(x < 0\) and equals \(x\) otherwise; see Figure~\ref{fig:activation}.  Hence, it is two linear functions augmented with a simple {\em test,} \(x < 0,\) for {\em selecting} one of them. What this tells us is that we can turn ACs into universal approximators by only augmenting them with {\em testing units.} 
We show this later, leading to {\em Testing Arithmetic Circuits (TACs)}.  
In fact, the notion of testing can be integrated directly into Bayesian networks, leading to {\em Testing Bayesian Networks (TBNs).}  
A TAC computes a TBN query, just like an AC computes a Bayesian network query.

\section{Testing Bayesian Networks}
\label{sec:tbn}

The concept of a Testing Bayesian Network (TBN) is relatively simple. In a nutshell: it is a Bayesian network whose CPTs are selected dynamically based on the given evidence.

Consider a Bayesian network that contains a binary node \(X\) having a single binary parent \(U\).
The CPT for node \(X\) contains  {\em one} distribution on \(X\) for each state \(u\) of its parent:
\begin{center}
\footnotesize
\[
\begin{array}{cc|c}
U & X & \\ \hline
u & x & \theta_{x|u} \\
u & \n(x) & \theta_{\n(x)|u} \\ \hline
\n(u) & x & \theta_{x|\n(u)} \\
\n(u) & \n(x) & \theta_{\n(x)|\n(u)} \\
\end{array}
\]
\end{center}
In a TBN, node \(X\) may be {\em testing.} In this case, we need {\em two} distributions on \(X\) for each state \(u\) of its parent. Moreover, we need a {\em threshold}
for each state \(u\), which is used to select one of these distributions:
\begin{center}
\footnotesize
\[
\begin{array}{cc|c|cc}
U & X          &  &  &  \\ \hline
u & x           & T_{u} & \theta^+_{x|u}  & \theta^-_{x|u} \\
u & \n(x)      & & \theta^+_{\n(x)|u}  & \theta^-_{\n(x)|u} \\ \hline
\n(u) & x      & T_{\n(u)} & \theta^+_{x|\n(u)}  & \theta^-_{x|\n(u)} \\
\n(u) & \n(x) & & \theta^+_{\n(x)|\n(u)} & \theta^-_{\n(x)|\n(u)} \\
\end{array}
\]
\end{center}
The selection of distributions utilizes the posterior on parent \(U\) given evidence on \(X\)'s {\em ancestors.}
For parent state \(u\), the selected distribution on \(X\) is \((\theta^+_{x|u},\theta^+_{\n(x)|u})\) if the posterior on \(u\) is \(\geq T_{u}\); otherwise, it is \((\theta^-_{x|u},\theta^-_{\n(x)|u})\).  
For parent state \(\n(u)\), the distribution is \((\theta^+_{x|\n(u)},\theta^+_{\n(x)|\n(u)})\) if the posterior on \(\n(u)\) is \(\geq T_{\n(u)}\); otherwise, it is \((\theta^-_{x|\n(u)},\theta^-_{\n(x)|\n(u)})\).   
Thus, the CPT for node \(X\) is determined {\em dynamically} based on the two thresholds and the posterior over  parent \(U\), leading to one of the following four CPTs:\footnote{In general,
testing can take other forms such as \(> T\), \(\le T\) or \(< T\).}
\begin{center}
\footnotesize
\[
\begin{array}{cc|c}
U & X & \\ \hline
u & x & \theta^+_{x|u} \\
u & \n(x) & \theta^+_{\n(x)|u} \\ \hline
\n(u) & x & \theta^+_{x|\n(u)} \\
\n(u) & \n(x) & \theta^+_{\n(x)|\n(u)} \\
\end{array}
\qquad
\begin{array}{cc|c}
U & X & \\ \hline
u & x & \theta^+_{x|u} \\
u & \n(x) & \theta^+_{\n(x)|u} \\ \hline
\n(u) & x & \theta^-_{x|\n(u)} \\
\n(u) & \n(x) & \theta^-_{\n(x)|\n(u)} \\
\end{array}
\qquad
\begin{array}{cc|c}
U & X & \\ \hline
u & x & \theta^-_{x|u} \\
u & \n(x) & \theta^-_{\n(x)|u} \\ \hline
\n(u) & x & \theta^+_{x|\n(u)} \\
\n(u) & \n(x) & \theta^+_{\n(x)|\n(u)} \\
\end{array}
\qquad
\begin{array}{cc|c}
U & X & \\ \hline
u & x & \theta^-_{x|u} \\
u & \n(x) & \theta^-_{\n(x)|u} \\ \hline
\n(u) & x & \theta^-_{x|\n(u)} \\
\n(u) & \n(x) & \theta^-_{\n(x)|\n(u)} \\
\end{array}
\]
\end{center}
In general, if the parents of testing node \(X\) have \(n\) states, the selection process may yield \(2^n\) distinct CPTs.
We will now give two illustrative examples of TBNs before we define their syntax and semantics formally.

\begin{figure}[t]
\centering
\subfigure[]{\label{fig:noisy-or}
  \includegraphics[width=0.17\linewidth]{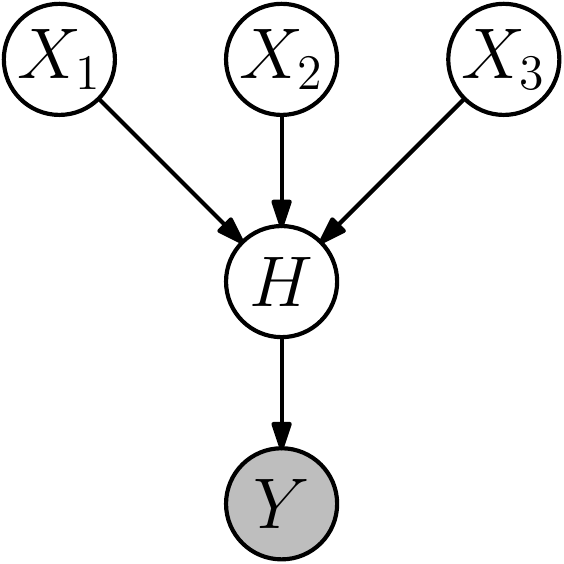}}
  \qquad \qquad
\subfigure[]{\label{fig:diamond}
  \includegraphics[width=0.17\linewidth]{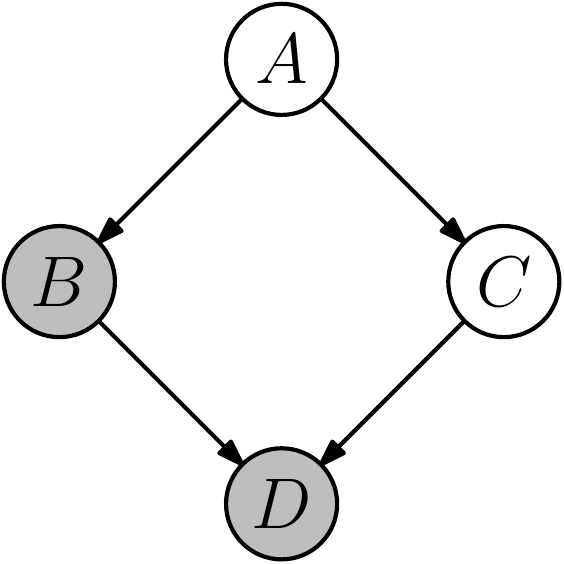}}
\caption{Two testing Bayesian Networks (TBNs). All nodes are binary. Testing nodes are shaded.} \label{fig:networks}
\end{figure}

Consider Figure~\ref{fig:noisy-or}, where all nodes are binary. Node \(Y\) is testing and has the following {\em testing CPT:}
\begin{center}
\small
\[
\begin{array}{cc|rl|rlrl}
H & Y          &  &  &  \\ \hline
h & y           & T_{h}= & t & \theta^+_{y|h}=  & 1 & \theta^-_{y|h}= & 0 \\
h & \n(y)      & & & \theta^+_{\n(y)|h}=  & 0 & \theta^-_{\n(y)|h}= & 1 \\ \hline
\n(h) & y      & T_{\n(h)}= & 1-t & \theta^+_{y|\n(h)}= & 0 & \theta^-_{y|\n(h)}= & 1 \\
\n(h) & \n(y) &  &  & \theta^+_{\n(y)|\n(h)}= & 1 & \theta^-_{\n(y)|\n(h)}= & 0\\
\end{array}
\]
\end{center}
Suppose now we have evidence \(\e = x_1,x_2,x_3\), which pertains to the ancestors of testing node \(Y\). In this example, the selected CPT for node \(Y\) will be based 
on the tests \(\pr(h \mid \e) \ge T_{h}\) and \(\pr(\n(h) \mid \e) > T_{\n(h)}\).
For parent state \(h\), the selected distribution on \((y,\n(y))\) is \((1,0)\) if \(\pr(h \mid \e) \ge T_{h}=t\); otherwise, it is \((0,1)\).
For parent state \(\n(h)\), the selected distribution on \((y,\n(y))\) is \((0,1)\) if \(\pr(\n(h) \mid \e) > T_{\n(h)}=1-t\); otherwise, it is \((1,0)\).

Nodes \(X_1, X_2, X_3\) and \(H\) can be viewed as the basis of a noisy-or classifier as in~\cite{Vomlel06}. That is, we classify an instance \(x_1,x_2,x_3\) positively iff \(\pr(h | x_1,x_2,x_3) \geq t\), where \(t\) is the classification threshold. The TBN in Figure~\ref{fig:noisy-or} can then be viewed as implementing this classifier since  
\(\pr(y | x_1,x_2,x_3) = 1\) if instance \(x_1,x_2,x_3\) is positive, and \(\pr(y | x_1,x_2,x_3) = 0\) if the instance is negative. 

More generally, we may have multiple testing nodes in a TBN. Figure~\ref{fig:diamond} depicts a TBN with two testing nodes, \(B\) and \(D\), where all variables are binary.  In a classical Bayesian network, we need \(18\) parameters to fully specify the network: \(2\) for \(A\), \(4\) for each of \(B,C\) and \(8\) for \(D\). For the TBN, we need \(30\) parameters: \(4\) additional parameters for \(B\) and \(8\) additional parameters for \(D\). We also need \(2\) thresholds for \(B\) and \(4\) thresholds for \(D\).

From now on, we will use BN to denote a classical Bayesian network and TBN for a testing one.

\subsection{TBN Syntax}

A TBN is a directed acyclic graph (DAG) with two types of nodes: {\em regular} and {\em testing,} each having a conditional probability table (CPT). 
Root nodes are always regular. Consider a node \(X\) with parents \(\U\).
\begin{itemize}
\item[--] If \(X\) is a regular node, its CPT is said to be {\em regular} and has a parameter \(\theta_{x|\u} \in [0,1]\) for each state \(x\) of node \(X\) and state \(\u\) of its parents \(\U\),
such that \(\sum_x \theta_{x|\u} = 1\) (these are the CPTs used in BNs).
\item[--] If \(X\) is a testing node, its CPT is said to be {\em testing} and has a threshold \(T_{X|\u} \in [0,1]\) for each state \(\u\) of parents \(\U\).
It also has two parameters \(\theta^+_{x|\u} \in [0,1]\) and \(\theta^-_{x|\u} \in [0,1]\) for each state \(x\) of node \(X\) and state \(\u\) of its parents \(\U\),
such that \(\sum_x \theta^+_{x|\u} = 1\) and \(\sum_x \theta^-_{x|\u} = 1\).
\end{itemize}
The parameters of a regular CPT are said to be {\em static} and the ones for a testing CPT are said to be {\em dynamic.}

Consider a node that has \(m\) states and its parents have \(n\) states. If the node is regular, its CPT will have \(m \cdot n\) static parameters.
If it is a testing node, its CPT will have \(n\) thresholds and \(2\cdot m \cdot n\) dynamic parameters.
As we shall discuss later, the thresholds and parameters of a TBN can be learned discriminatively from labeled data (as in deep learning).

\subsection{TBN Semantics}

A testing CPT corresponds to a set of regular CPTs, one of which is selected based on the given evidence.
Once a regular CPT is selected from each testing CPT, the TBN transforms into a BN.
In other words, a TBN over DAG \(G\) represents a set of BNs over DAG \(G\), one of which is selected based on the given evidence. 
It is this selection process that determines the semantics of TBNs. We define this process next based on soft evidence, which includes hard evidence as 
a special case.\footnote{\cite{ChoiDarwiche18} used soft evidence on root nodes only, which is sufficient for the universal approximation theorem.}  

Soft evidence on a variable \(X\) with states \(x_1, \ldots, x_k\) is specified using {\em likelihood ratios} \(\lambda_1, \ldots, \lambda_k\)~\citep{Pearl88b}.  
Without loss of generality, we require \(\lambda_{1} + \ldots + \lambda_{k} = 1\) so \(\lambda_{i}=1\) corresponds to hard evidence \(\eql(X,x_i)\).
Moreover, when node \(X\) is binary, soft evidence reduces to a single number \(\lambda_x \in [0,1]\) since \(\lambda_{\n(x)} = 1-\lambda_x\).

We will emulate soft evidence by hard evidence on an auxiliary, binary child \(S\) with the following CPT: 
\begin{center}
\small
\[
\begin{array}{cc|c}
X & S & \Theta_{S|x} \\ \hline
x_1 & s & \lambda_1 \\
x_1 & \n(s) & 1-\lambda_1 \\ 
\vdots & \vdots & \vdots \\
x_k & s & \lambda_k \\
x_k & \n(s) & 1-\lambda_k \\ 
\end{array}
\]
\end{center}
We can now take \(\pr(.|s)\) as the result of asserting soft evidence on node \(X\) since
\[
\frac{\pr(x_1|s)}{\pr(x_1)} \: : \: \ldots \: : \: \frac{\pr(x_k|s)}{\pr(x_k)} = \lambda_{1} \: : \: \ldots \: : \: \lambda_{k}.
\]
Let \({\bf s}\) denote all available soft evidence.
We will use \(\Pstar(.)\) to denote the conditional distribution \(\pr(.|{\bf s})\) and \(\pstar(.)\) to denote the joint distribution \(\pr(.,{\bf s})\).

We now show how a TBN can be converted into a BN, thereby defining the semantics of TBNs. We start with an empty BN
and traverse the TBN nodes, parents before children. Suppose we are visiting TBN node \(X\) with parents \(\U\).
If node \(X\) is regular, we add it to the BN as a child of nodes \(\U\), while copying its regular CPT to the BN.
If node \(X\) is testing, we first use the partially constructed BN to compute the posterior \(\Pstar(\U)\) using soft evidence on the ancestors of node \(X\). 
Using this posterior, we then select a regular CPT for node \(X\) from its testing CPT, e.g., as follows:\footnote{The selection can be based
on other tests such as \(\Pstar(\u) > T_{X|\u}\), \(\Pstar(\u) \le T_{X|\u}\) or \(\Pstar(\u) < T_{X|\u}\).}
\[
\theta_{x|\u} = 
\left\{
\begin{array}{ll}
\theta^+_{x|\u} & \mbox{if \(\Pstar(\u) \ge T_{X|\u}\)} \\
\theta^-_{x|\u} & \mbox{otherwise.}
\end{array}
\right.
\]
We finally add node \(X\) to the BN as a child of nodes \(\U\) and copy its selected, regular CPT to the BN. 

After visiting all nodes in the TBN, the constructed BN will have the same structure as the TBN. We can now use this BN to answer queries using all available evidence, as is normally done.

The dependence of selection on only ancestral evidence is not strictly needed and can sometimes be easily relaxed, but this makes the semantics of TBNs less transparent. 
We prefer simplicity for now given that this is all we need for our main result on universal approximation.\footnote{The dependence of CPT selection on only ancestral evidence
can be limiting practically though since not all TBN queries can fully benefit from the power of CPT selection. In the extreme case of evidence laying below testing nodes in a TBN, 
CPT selection will not be impacted by the available evidence. CPT selection can be easily made to depend on more evidence, beyond ancestral, as long as it does not lead to selection 
ambiguities (i.e., the selection of one CPT impacting the selection of another). Resolving such potential ambiguities, however, requires a more thorough treatment.}

\section{Testing Arithmetic Circuits}
\label{sec:tac}

\begin{figure}[t]
 \centering
 \subfigure[TBN]{\label{fig:c-TBN}
 \includegraphics[width=0.14\linewidth]{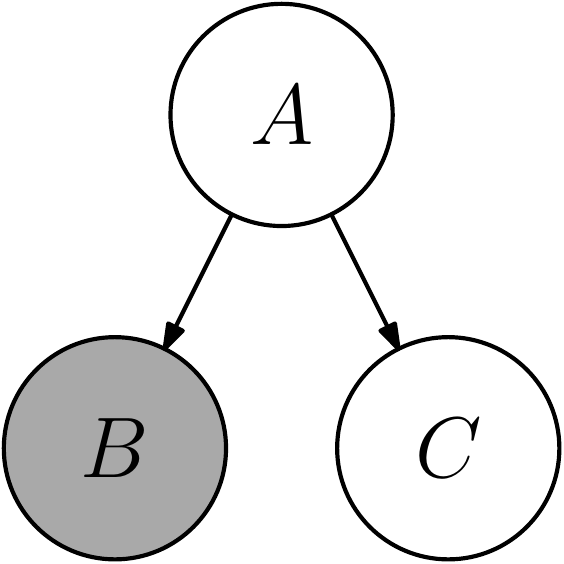}} 
\qquad
 \subfigure[TAC]{\label{fig:c-TAC}
    \includegraphics[width=0.75\linewidth]{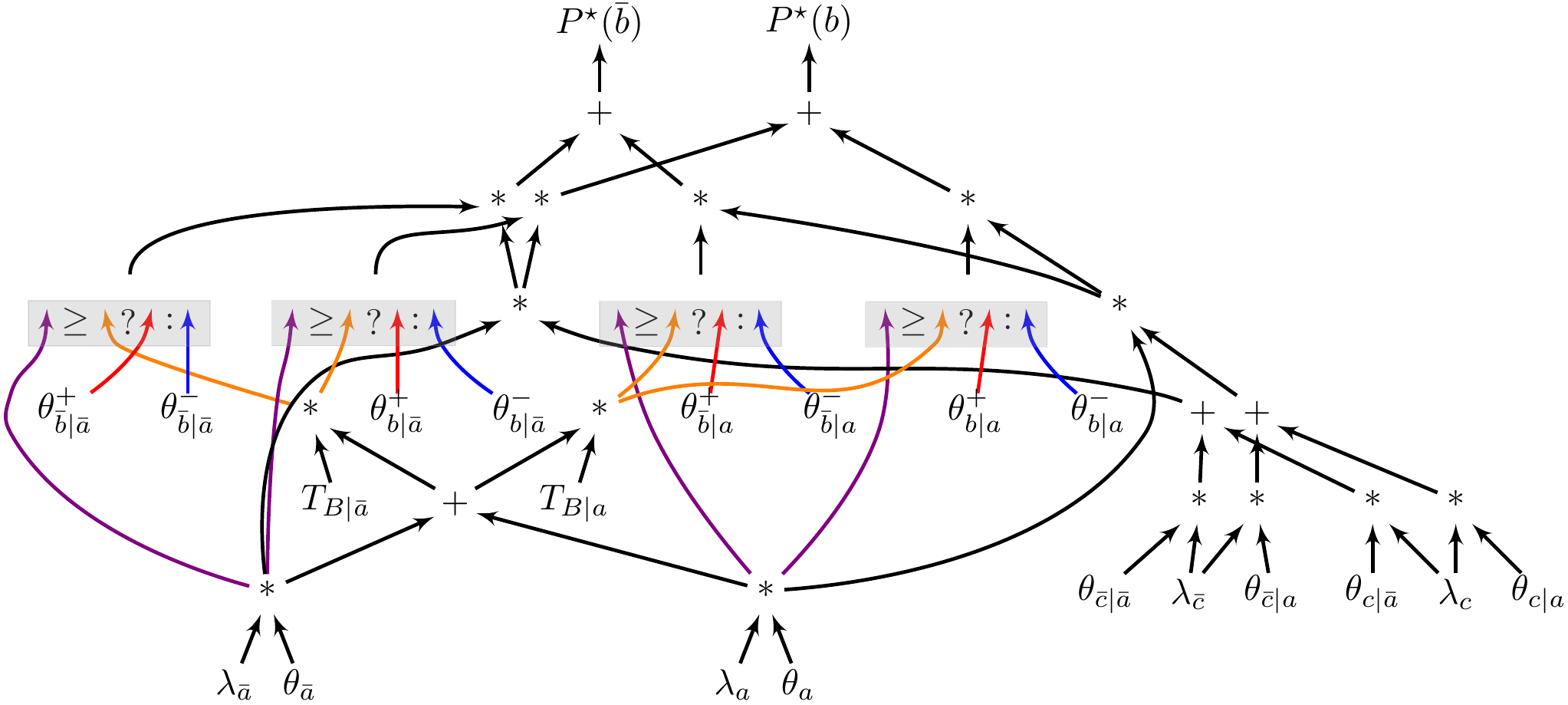}}
\caption{Nodes \(A\), \(B\) and \(C\) are binary and node \(B\) is testing. 
Nodes \fbox{\(x \ge\ T \: ? \: \theta^+ : \theta^-\)} represent testing units.
 \label{fig:compiled tac}}
\end{figure}

A Testing Arithmetic Circuit (TAC) is an Arithmetic Circuit (AC) that includes {\em testing units.}  A testing unit has two inputs, \(x\) and \(T\), and two parameters, \(\theta^+\) and \(\theta^-\).
Its output is computed as follows:\footnote{A testing unit may employ other tests such as \(x > T\), \(x \le T\) or \(x < T\).}
\[
f(x,T)
= 
\left\{
\begin{tabular}{rl}
$\theta^+$ & \mbox{if \(x \ge T\)} \\
$\theta^-$ & \mbox{otherwise.}
\end{tabular}
\right.
\]
Just like an AC computes a BN query (see Figure~\ref{fig:function}), a TAC computes a TBN query.  
\ref{sec:compiling tacs} provides an algorithm for compiling a TAC that computes a given TBN query. 

This algorithm was used to compile the TAC in Figure~\ref{fig:c-TAC}, which has four testing units and computes a query on the TBN in Figure~\ref{fig:c-TBN}. 
The TAC inputs \((\lambda_a,\lambda_{\n(a)})\)
and \((\lambda_c,\lambda_{\n(c)})\) represent soft evidence on nodes \(A\) and \(C\), respectively. Its outputs \(\pstar(b)\) and \(\pstar(\n(b))\) represent the
marginal distribution on node \(B\). All other TAC inputs correspond to TBN parameters and thresholds: \(2\) static parameters for node \(A\),
\(4\) static parameters for node \(C\), in addition to \(8\) dynamic parameters and \(2\) thresholds for node \(B\).

As discussed earlier, TBNs are motivated by the recent success of neural networks in learning functions from labeled data. 
When using a neural network in this context, the function structure is usually handcrafted while its parameters are learned using gradient descent methods. 
While neural networks have been very successful in this context, they have also been subject to scrutiny due to their opaqueness and inability to integrate domain knowledge 
in a principled manner. Opaqueness makes neural networks hard to explain and verify. 
The inability to integrate domain knowledge diminishes their robustness and implies that we may need a massive amount of data to train them successfully. 

The structure of a function computed by a TBN query is {\em compiled} from the TBN. This structure takes the form of a TAC, which 
parameters and thresholds can be learned from labeled data using gradient descent methods.
The advantage of compiling a function structure from a model, in contrast to handcrafting it, is that we can integrate some forms of background knowledge into the 
function. For example, all the independence assumptions encoded by the TBN will be respected by the compiled TAC, regardless of how we
train it. Moreover, some TBN parameters may already be known and these will be carried into the compiled TAC, without the need to train them from data,
therefore reducing the dependence on data.
This particularly includes \(0/1\) parameters, which correspond to logical domain constraints.
Finally, since a TAC is compiled from a TBN, one stands a better chance at explaining and verifying the TAC behavior.
One must note though that the expressiveness of a TAC is mandated by the underlying TBN query so it cannot be arbitrarily controlled, as in handcrafted neural
networks. We discuss this issue in the next two sections.

\section{Expressiveness of TBN Queries and TACs}
\label{sec:uat}

A central concept relating to expressiveness is that of universal approximation. Consider the class of continuous functions \(f(x_1, \ldots,x_n)\) from 
 \([0,1]^n\) to \([0,1]\). A representation is said to be a universal approximator if it can approximate any function in this class to an arbitrary small error.
 As mentioned earlier, neural networks with appropriate activation functions are universal approximators. 
 
 TBN queries and, hence, TACs are also universal approximators. 
 The following theorem shows this for continuous and monotonic functions \(f(x)\). The general (and slightly more involved) case of multivariate, non-monotonic functions is delegated to \ref{sec:uat2}. 
 
 \begin{theorem}\label{theo:uat}
Given a continuous, monotonic function \(f(x)\) from \([0,1]\) to \([0,1]\) and error \(\varepsilon\), there exists a TBN that contains \(O(\lceil \frac{1}{\varepsilon} \rceil)\) nodes and
that satisfies the following. The TBN contains two binary nodes, \(Z\) and \(Y\), such that 
if \(x\) is the soft evidence on node \(Z\), then \(|\Pstar(y)-f(x) \leq \varepsilon|\).
 \end{theorem}

\begin{figure}[t]
 \centering
 \subfigure[TBN]{\label{fig:TBN}
   \includegraphics[width=0.25\linewidth]{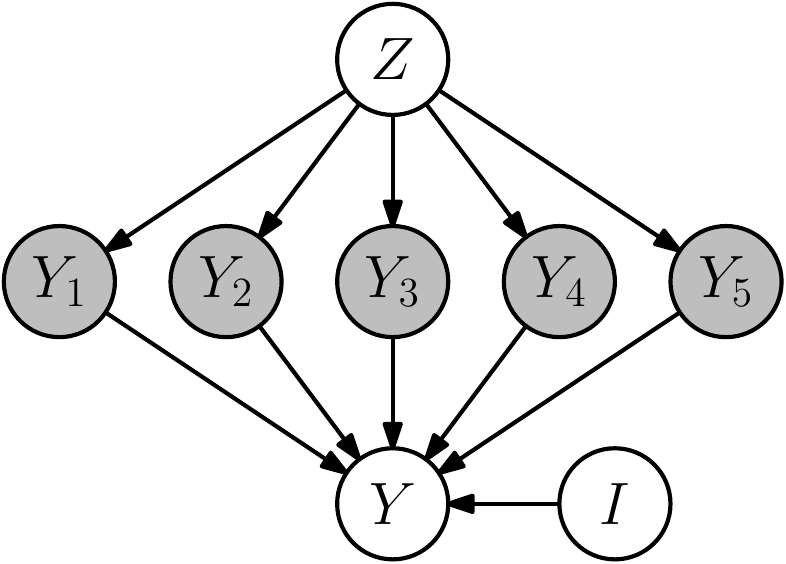}}
   \qquad
 \subfigure[TAC for query \(\Pstar(y)\)]{\label{fig:TAC}
   \includegraphics[width=0.25\linewidth]{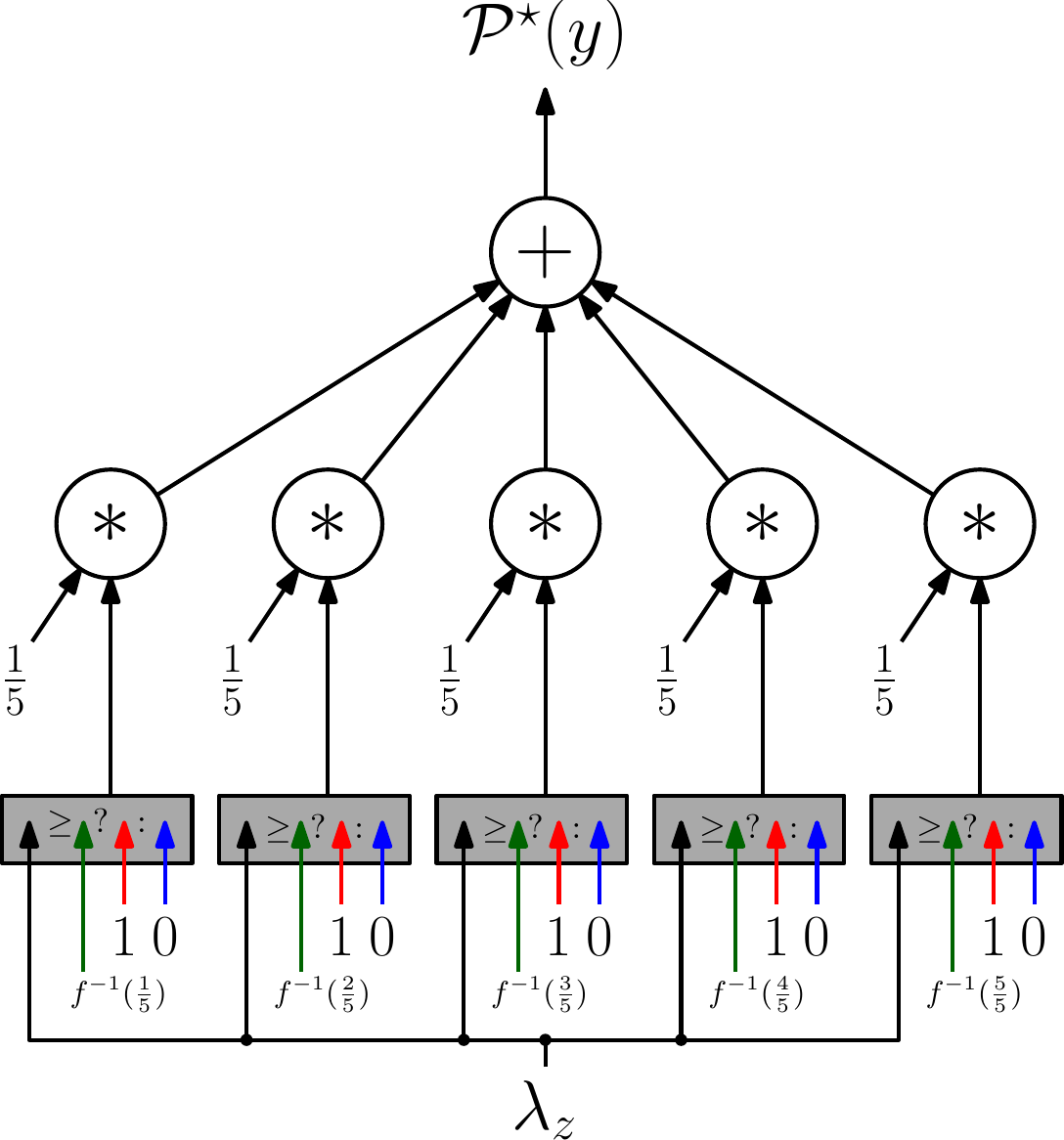}}
   \qquad
 \subfigure[BN after CPT selection]{\label{fig:instance}
   \includegraphics[width=0.25\linewidth]{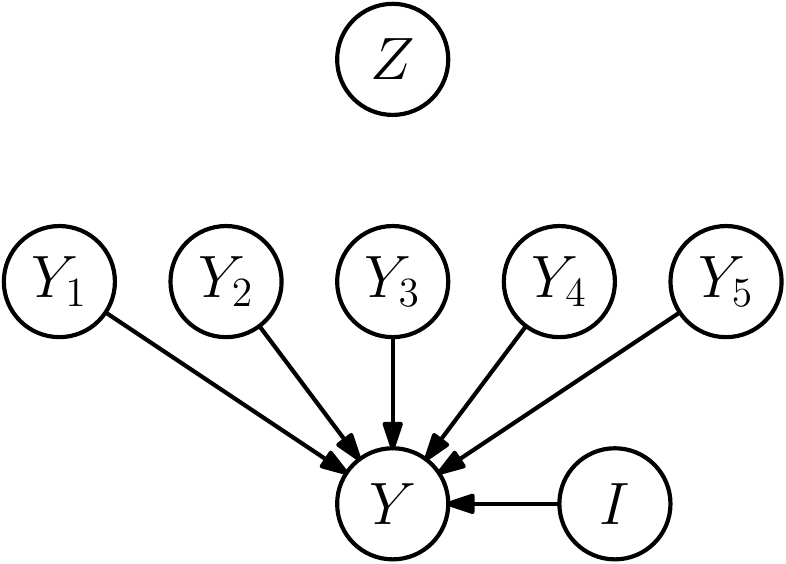}}
 \caption{Approximating a continuous, monotonic function \(f(x)\) using a TBN query for \(N=5\). \label{fig:tbn-tac}
 }
\end{figure}

\begin{proof}
Our constructive proof is based on~\cite{Jones90} and uses the TBN in Figure~\ref{fig:TBN}.  
The TBN has regular nodes \(Z, Y\) and \(I\), which are discrete. Nodes \(Z\) and \(Y\) are binary with states \(z,\n(z)\) and \(y,\n(y)\). 
Node \(I\) has values \(1, \ldots, N\) and controls the quality of approximation (better approximation for larger \(N\)). 
The TBN has testing nodes \(\Y = Y_1, \ldots, Y_N\), which are binary with states \(y_i,\n(y)_i\).  
The TBN query is to compute the probability of \(\eql(Y,y)\) given soft evidence on node \(Z\).
The TAC in~Figure~\ref{fig:TAC} computes this query.

Nodes \(Z\) and \(I\) have uniform priors. Node \(Y\) has the following CPT:
\[
\theta_{y\mid \y,i} 
=
\left\{
\begin{tabular}{rl}
1 & \mbox{if \(\y\) sets variable \(Y_i\) to value \(y\)} \\
0 & \mbox{otherwise}
\end{tabular}
\right.
\]
That is, node \(Y\) is equivalent to node \(Y_i\) when \(I=i\). The testing CPT for node \(Y_i\) is as follows:\footnote{We assume, without loss of generality, that \(f(0)=0\) and \(f(1)=1\).}
\begin{center}
\small
\[
\begin{array}{cc|rl|rlrl}
Z & Y_i          &  &  &  \\ \hline
z & y_i           & T_{z}= & f^{-1}(\frac{i}{N}) & \theta^+_{y_i|z}=  & 1 & \theta^-_{y_i|z}= & 0 \\
z & \n(y)_i      & & & \theta^+_{\n(y)_i|z}=  & 0 & \theta^-_{\n(y)_i|z}= & 1 \\ \hline
\n(z) & y_i      & T_{\n(z)}= & 1-f^{-1}(\frac{i}{N}) & \theta^+_{y_i|\n(z)}= & 0 & \theta^-_{y_i|\n(z)}= & 1 \\
\n(z) & \n(y)_i &  &  & \theta^+_{\n(y)_i|\n(z)}= & 1 & \theta^-_{\n(y)_i|\n(z)}= & 0\\
\end{array}
\]
\end{center}
Given soft evidence \(\lambda_z=x\) on node \(Z\), the posterior marginal \(\Pstar(z)\) is then \(x\). We can now select a regular CPT for each testing
node \(Y_i\), which must be one of\footnote{The used tests are \(\Pstar(z) \ge T_{z}\) and \(\Pstar(\n(z)) > T_{\n(z)}\). We have
\(\Pstar(z) \ge T_{z}\) iff \(1-\Pstar(z) \le 1-T_{z}\) iff \(\Pstar(\n(z)) \le T_{\n(z)}\).}

\begin{center}
\small
\[
\begin{array}{cc|cc}
Z & Y_i          &  &  \\ \hline
z & y_i           &  \theta_{y_i|z}=  & 1  \\
z & \n(y)_i      & \theta_{\n(y)_i|z}=  & 0 \\ \hline
\n(z) & y_i      & \theta_{y_i|\n(z)}= & 1 \\
\n(z) & \n(y)_i & \theta_{\n(y)_i|\n(z)}= & 0 \\
\end{array}
\qquad \qquad
\begin{array}{cc|cc}
Z & Y_i          &  &  \\ \hline
z & y_i           &  \theta_{y_i|z}=  & 0  \\
z & \n(y)_i      & \theta_{\n(y)_i|z}=  & 1 \\ \hline
\n(z) & y_i      & \theta_{y_i|\n(z)}= & 0 \\
\n(z) & \n(y)_i & \theta_{\n(y)_i|\n(z)}= & 1 \\
\end{array}
\]
\end{center}
In either case, \(Y_i\) will no longer depend on \(Z\), leading to the BN structure in Figure~\ref{fig:instance}.
Using \(\theta_{y_i}\) to denote \(\theta_{y_i|z} = \theta_{y_i|\n(z)}\), and noting that \(\Pstar(y)=P(y)\) in the selected BN, we get: 
\begin{eqnarray*}
\Pstar(y) 
& = & \sum_{\y} \sum_{i=1}^N \Pstar(y \mid \eql(I,i), \y) \Pstar(\eql(I,i),\y) \\
& = & \sum_{\y} \sum_{i=1}^N \Pstar(y \mid \eql(I,i), \y) \cdot \theta_{\eql(I,i)} \cdot \prod_{i=1}^N \theta_{y_i} \\
&  =&  \sum_{i=1}^N \theta_{\eql(I,i)} \cdot \theta_{y_i} =  \frac{1}{N} \sum_{i=1}^N \theta_{y_i}.
\end{eqnarray*}
Intuitively, each node \(Y_i\) is either \emph{activated} (\(\theta_{y_i}=1\)) or \emph{de-activated} (\(\theta_{y_i}=0\)), where \(\Pstar(y)\) is the proportion of activated nodes.  
Moreover, by choice of thresholds \(T_z\) and \(T_{\n(z)}\), as \(x\) increases, more nodes \(Y_i\) get activated.  
In fact, exactly \(\lfloor N \cdot f(x) \rfloor\) are activated, leading to \(\Pstar(y)  = \frac{1}{N} \lfloor N \cdot f(x) \rfloor\) and \(|f(x) - \Pstar(y)| \le \varepsilon\).
\end{proof}

\section{The Functional Form of TACs}
\label{sec:f-form}

Consider a TBN that contains some binary nodes, \(E_1, \ldots, E_n, Q\). Soft evidence on nodes \(E_i\) corresponds to a vector \(\lambda_1, \ldots, \lambda_n\) 
in \([0,1]^n\). Hence, a query that asks for the probability of \(\eql(Q,q)\) given soft evidence will then correspond to a function that maps \([0,1]^n\) into \([0,1]\). 
In Section~\ref{sec:uat} and~\ref{sec:uat2}, we provided a universal approximation theorem, showing how each continuous function from \([0,1]^n\) into \([0,1]\) 
can be approximated to an arbitrary error by a TBN query. 
The construction used a TBN and query that are based on the given function and error. 
In practice though, the TBN and query are mandated by modeling and task considerations, which may restrict the class of functions that can be learned from data. 
We next show, however, that regardless of the TBN and query, the induced function must be {\em piecewise multi-linear.}

\begin{definition}
We say that function \(f(\lambda_1,\ldots,\lambda_n)\) is \underline{multi-linear} if it has the form \(\sum_{I \subseteq \{1,\ldots,n\}} C_I \prod_{i \in I} \lambda_i\) for 
some constants \(C_I\). We say it is \underline{linear} if it has the form \(C_0 + \sum_{i=1}^n C_i \lambda_i\) for some constants \(C_i\). 
\end{definition}
For example, \(f(\lambda_1,\lambda_2)=a \lambda_1 \lambda_2 + b \lambda_1 + c \lambda_2 + d\) is multi-linear and \(f(\lambda_1,\lambda_2)=a \lambda_1 + b \lambda_2 + c\) is linear.

\begin{figure}[t]
  \centering
  \begin{subfigure}[BN function]{\label{fig:plot-bn}
    \includegraphics[width=.25\linewidth]{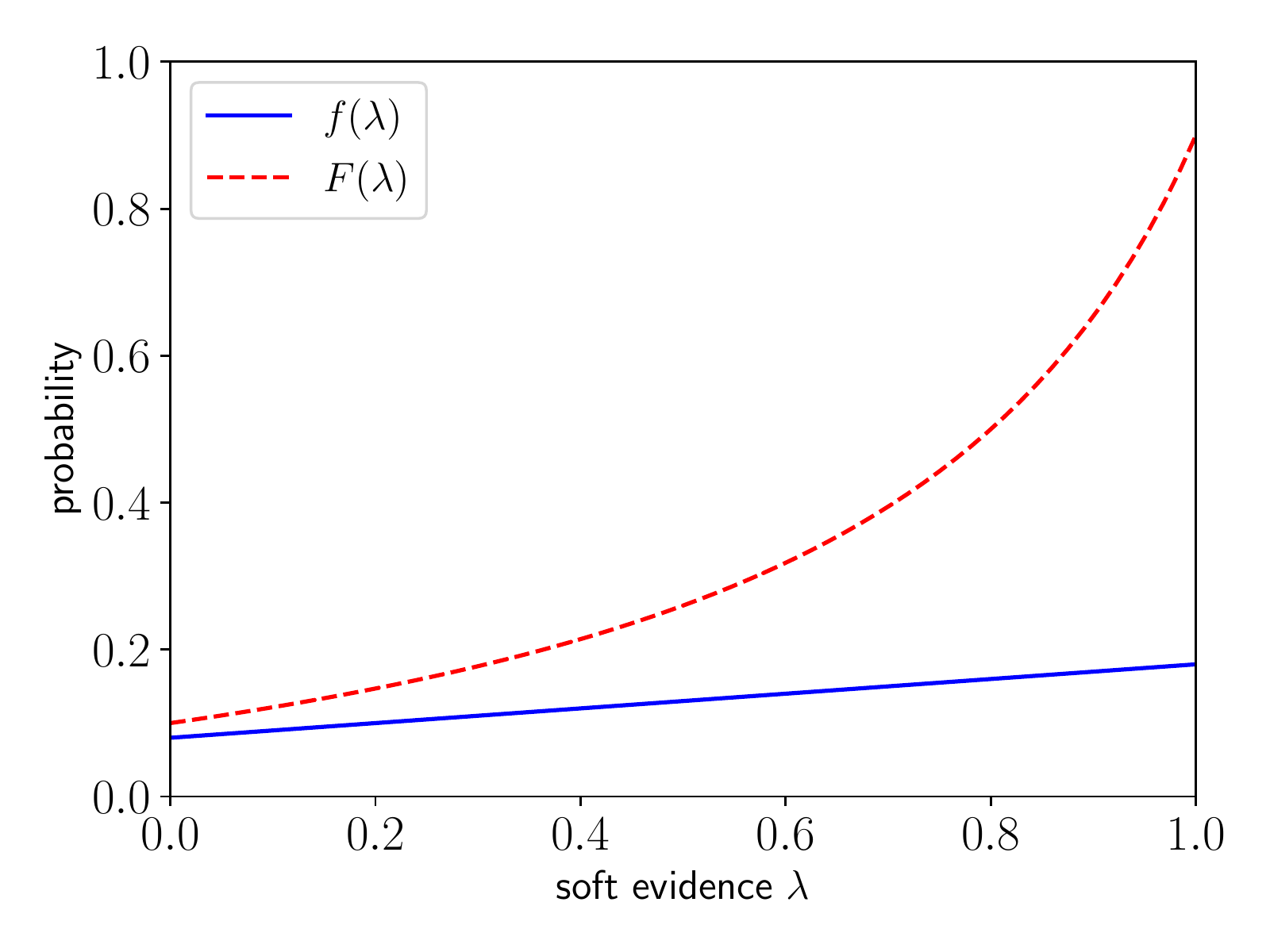}}
  \end{subfigure}
  \quad
  \begin{subfigure}[TBN function]{\label{fig:plot-tbn} 
    \includegraphics[width=.25\linewidth]{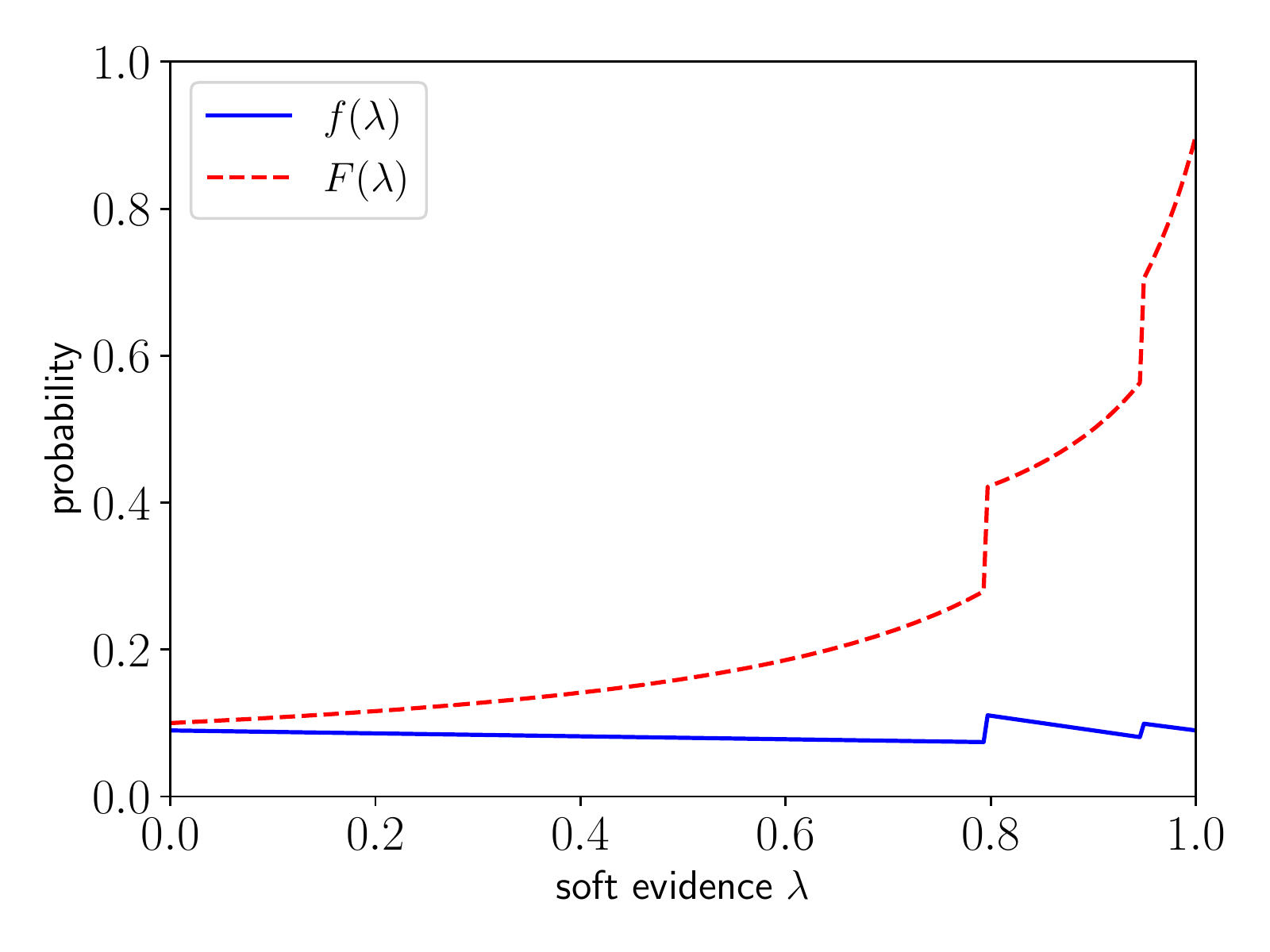}}
  \end{subfigure}
  \caption{Functions induced by BN and TBN queries on the structure \(E \rightarrow T \rightarrow Q\). Node \(T\) is testing in the TBN.}
  \label{fig:bn-tbn}
\end{figure}

We will next illustrate the functional form of BN and TBN queries using concrete examples, then follow by the formal results. We will start by
assuming only one evidence node \(E\) with soft evidence \(\lambda \in [0,1]\). Consider now the following queries:
\begin{itemize}
\item[---] \(f(\lambda)\): the joint probability of \(\eql(Q,q)\) and evidence \(\lambda\), \(\pstar(q)\).
\item[---] \(F(\lambda)\): the conditional probability of \(\eql(Q,q)\) given evidence \(\lambda\), \(\Pstar(q)\).
\end{itemize}
Figure~\ref{fig:plot-bn} depicts these functions for a BN \(E \rightarrow T \rightarrow Q\) with some arbitrary parameters. 
Figure~\ref{fig:plot-tbn} depicts these functions for a TBN with the same structure, but where node \(T\) is testing (again, the parameters and thresholds are arbitrary). 
For the BN, \(f(\lambda)\) is a linear function and \(F(\lambda)\) is a quotient of two linear functions. 
For the TBN, \(f(\lambda)\) is a piecewise linear function and \(F(\lambda)\) is a piecewise quotient of two linear functions.

More precisely, for BNs with one evidence node, these functions have the following form \citep{CastilloGH96,Jensen99,KjaerulffG00,Darwiche00}:
\[
f(\lambda)  = a \lambda + b
\quad\qquad F(\lambda) = \frac{a \lambda + b}{c \lambda + d}
\]
where the constants \(a, b, c, d\) depend only on the BN parameters. For TBNs with one evidence node, the input space \([0,1]\) is partitioned into {\em regions,} where
functions \(f\) and \(F\) take the above form but with different constants in each region; see Figure~\ref{fig:plot-tbn}.

\begin{figure}[t]
  \centering
  \begin{subfigure}[BN function]{\label{fig:plot-2bn}
    \includegraphics[width=.23\linewidth]{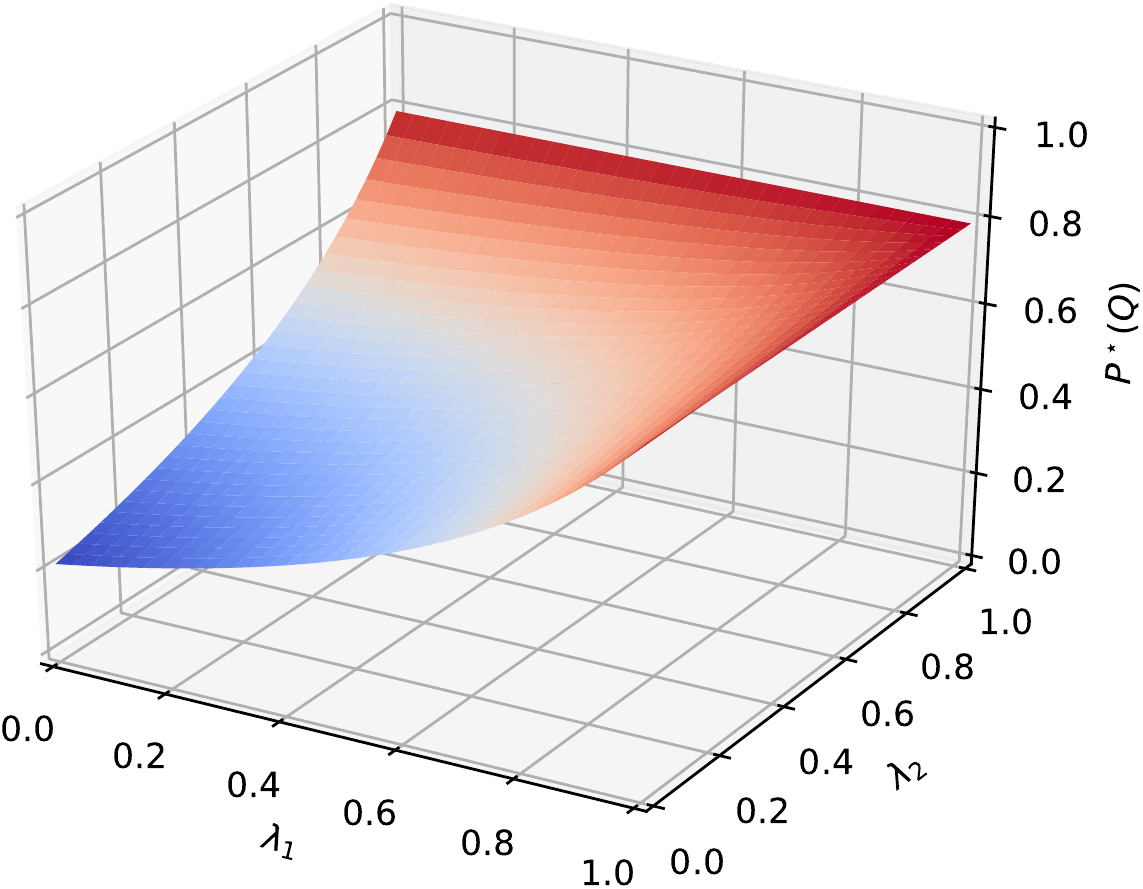}}
  \end{subfigure}
  \qquad\qquad
  \begin{subfigure}[TBN function]{\label{fig:plot-2tbn}
    \includegraphics[width=.23\linewidth]{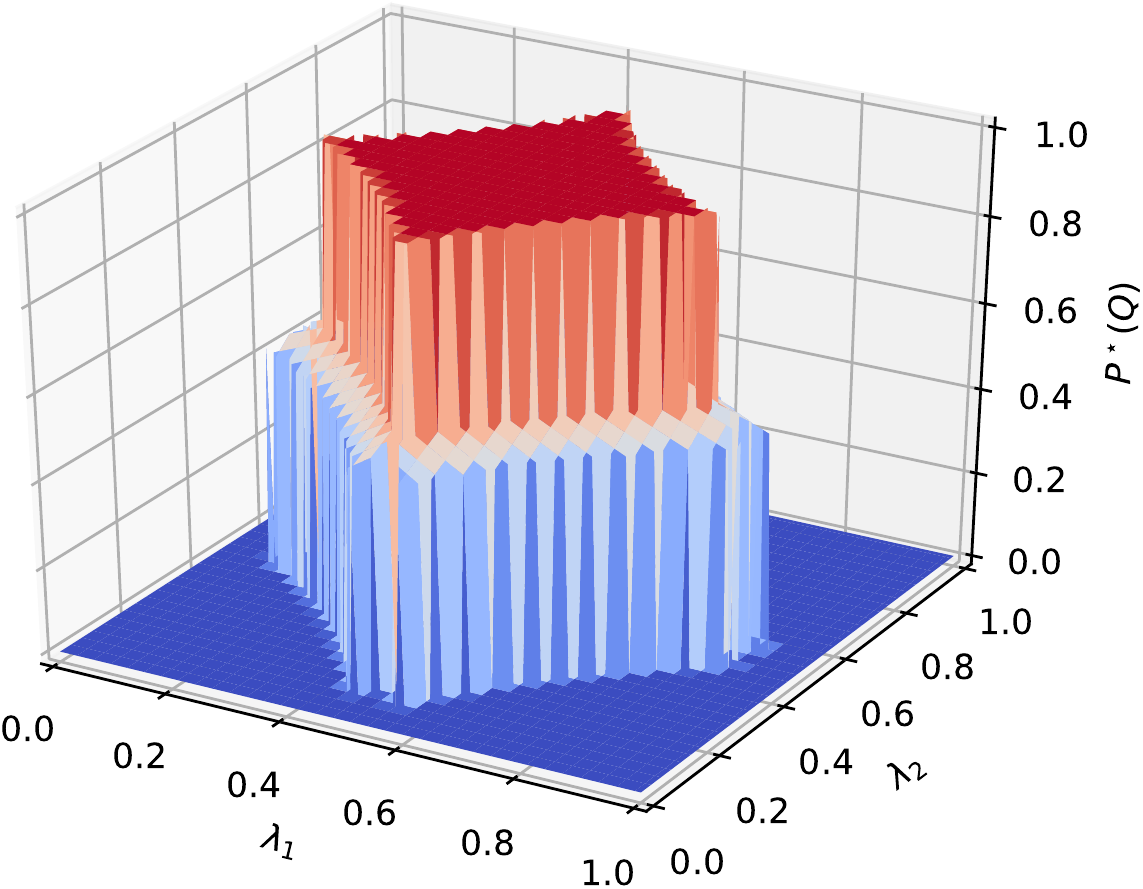}}
  \end{subfigure}
\\
\bigskip
  \begin{subfigure}[BN/TBN structure]{\label{fig:2tbn-dag} 
    \includegraphics[width=.22\linewidth]{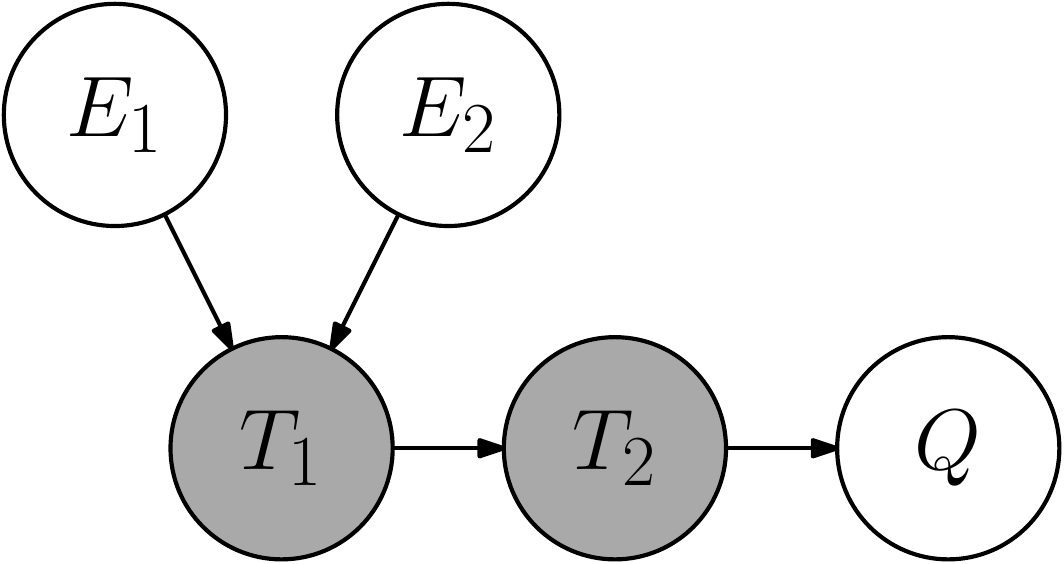}}
  \end{subfigure}
  \qquad\qquad
  \begin{subfigure}[TBN function]{\label{fig:plot-step}
    \includegraphics[width=.23\linewidth]{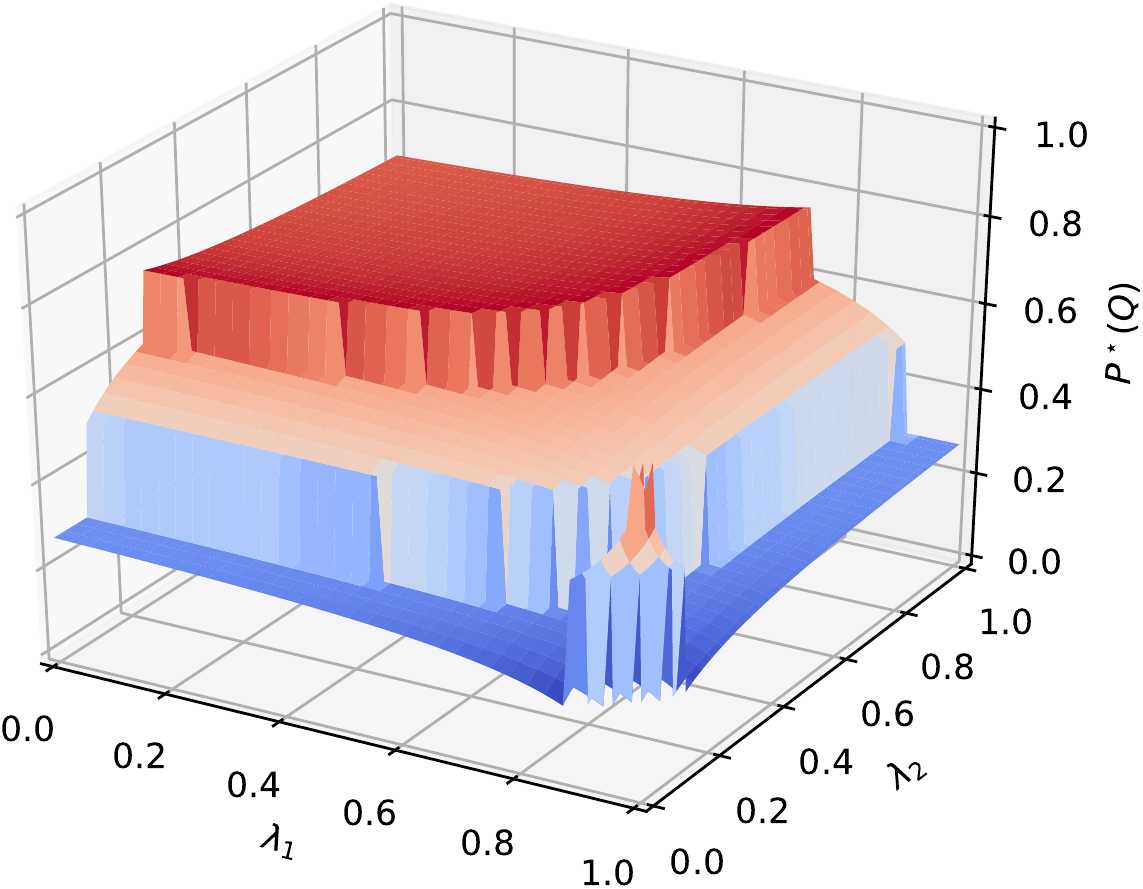}}
  \end{subfigure}

  \caption{Conditional probability functions, \(F(\lambda_1,\lambda_2)\), induced by BN and TBN queries. The TBN functions are for different parameters and thresholds,
  where nodes \(T_1\) and \(T_2\) are testing.
  \label{fig:2tbn}}
\end{figure}

Suppose we now have two evidence nodes \(E_1\) and \(E_2\) with soft evidence \(\lambda_1\) and \(\lambda_2\). 
For TBNs, the input space \([0,1] \times [0,1]\) is partitioned into two-dimensional regions. The functional form in each region is:
\[
f(\lambda_1,\lambda_2)  = a \lambda_1\lambda_2 + b\lambda_1 + c\lambda_2 + d
\quad\qquad F(\lambda_1,\lambda_2) = \frac{a \lambda_1\lambda_2 + b\lambda_1 + c\lambda_2 + d}{e \lambda_1\lambda_2 + f\lambda_1 + g \lambda_2 + h}
\]
Figure~\ref{fig:2tbn} depicts examples of \(F(\lambda_1,\lambda_2)\) corresponding to BN and TBN queries (only one region for the BN).

We now have the following result, showing that BN queries compute multi-linear functions. 
\begin{theorem}\label{theo:ML}
Consider a BN that contains some binary nodes \(E_1, \ldots, E_n, Q\) and let \(\lambda_1, \ldots, \lambda_n\) denote a soft evidence vector on nodes \(E_1, \ldots, E_n\). 
There are constants \(C_I\) for \(I \subseteq \{1, \ldots, n\}\) such that
\[
\pstar(q) = \sum_{I} C_I \prod_{i \in I} \lambda_i.
\]
\end{theorem}
\begin{proof}
Consider an instantiation \(\x\) over all network variables. We write \(\x \models x\u\) to mean that instantiation \(\x\) sets variable \(X\) and its parents \(\U\) to \(x\) and \(\u\).
Let \(I^+\) be the indices \(i\) of evidence variables \(E_i\) set positively by \(\x\), and let \(I^-\) be the indices \(i\) of evidence variables \(E_i\) set negatively by \(\x\).  We have:
\begin{eqnarray*}
\pstar(\x) 
& = & \prod_{\x \models x\u} \theta_{x|\u} \cdot \prod_{i \in I^+} \lambda_i \cdot \prod_{i \in I^-} (1-\lambda_i) \\
& = & \prod_{\x \models x\u} \theta_{x|\u} \cdot \prod_{i \in I^+} \lambda_i \cdot \Big[ \sum_{J \subseteq I^-} (-1)^{|J|} \cdot \prod_{j \in J} \lambda_j \Big] \\
& = & \sum_{J \subseteq I^-} (-1)^{|J|} \cdot \prod_{\x \models x\u} \theta_{x|\u} \cdot \prod_{i \in I^+} \lambda_i \cdot \prod_{j \in J} \lambda_j \\
& = & \sum_{I \subseteq \{1,\ldots,n\}} D^{\x}_I \prod_{i \in I} \lambda_i
\end{eqnarray*}
where \(D^{\x}_I\) is either \(0\), \(P(\x) = \prod_{\x \models x\u} \theta_{x|\u}\) or \(-P(\x)\).  
Hence, \(\pstar(\x)\) is a multi-linear function.  Since \(\pstar(q) = \sum_{\x \models q} \pstar(\x)\), it follows that \(\pstar(q)\) is also multi-linear as it is the sum
of multi-linear functions.
\end{proof}

The following result shows that TBN queries compute piecewise multi-linear functions.
\begin{theorem}\label{theo:PML}
Consider a TBN that contains some binary nodes \(E_1, \ldots, E_n, Q\) and let \(\lambda_1, \ldots, \lambda_n\) denote a soft evidence vector on nodes \(E_1, \ldots, E_n\). 
The space \([0,1]^n\) can be partitioned into a finite set of regions \(R\) that satisfy the following. 
For each region \(r \in R\), there are constants \(C_{I,r}\) for \(I \subseteq \{1, \ldots, n\}\) such that
\[
\pstar(q) = \sum_{I} C_{I,r} \prod_{i \in I} \lambda_i, \mbox{  for vectors \(\lambda_1, \ldots, \lambda_n\) in region \(r\).}
\]
\end{theorem}
\begin{proof}
For a given soft evidence vector \(\lambda_1, \ldots, \lambda_n\), the TBN selects a unique set of CPTs, leading to the selection of a unique BN.  
There is a finite number of possible selections that a TBN can make, leading to a finite number of BNs. 
Hence, the soft evidence space of \([0,1]^n\) can be partitioned into a finite set of regions \(R\), with each region \(r \in R\) leading to a unique BN. 
By Theorem~\ref{theo:ML}, \(\pstar(q)\) takes the form in Theorem~\ref{theo:PML}.
\end{proof}

While TBN queries represent piecewise {\em multi-linear} functions, neural networks with ReLU activation functions represent piecewise {\em linear} functions \citep{Pascanu2014,MontufarPCB14}.
Moreover, there has been work on bounding the number of regions for such functions, depending on the size and depth of neural 
networks, e.g., \cite{Pascanu2014,MontufarPCB14,Raghu2017,Serra2018}.

\begin{figure}[t]
  \centering
 \includegraphics[width=0.12\linewidth]{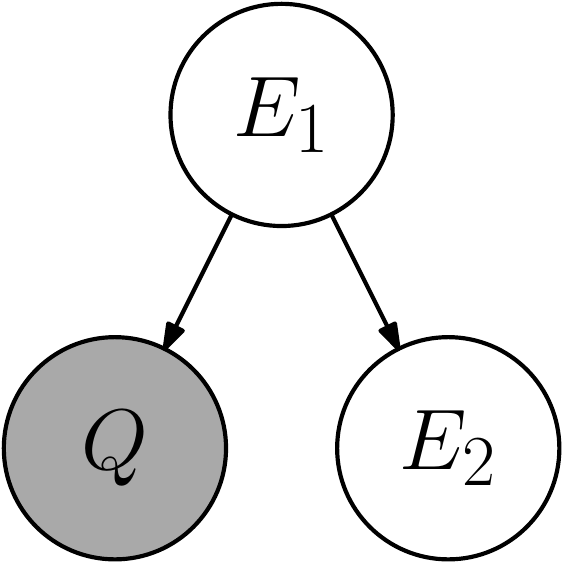}
 \qquad\qquad
 \includegraphics[width=0.23\linewidth]{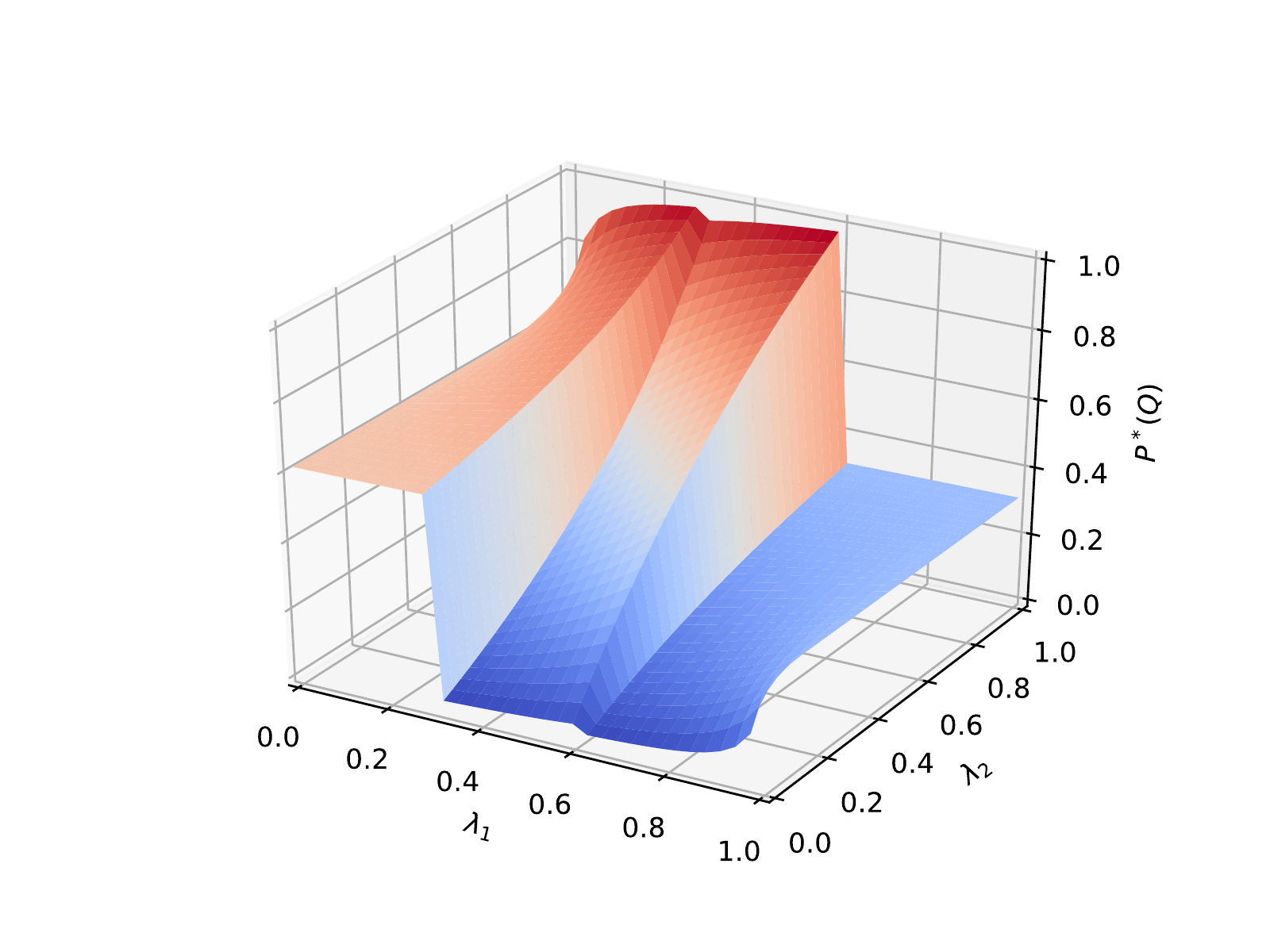}
 \caption{A conditional probability function, \(F(\lambda_1,\lambda_2)\), induced by a TBN query. Node \(Q\) is testing.
 \label{fig:tbn-impact}}
\end{figure}

A TBN query can induce different piecewise functions, depending on the TBN parameters and thresholds that are specified or learned from data (Figures~\ref{fig:plot-2tbn} and~\ref{fig:plot-step} 
depict two functions induced by the same TBN query).
In particular, these functions may have regions that differ in their nature 
and count, which impacts how well the learned function fits the training data. 
\begin{definition}
For a piecewise function, the number of regions is called the \underline{function granularity.}
For a TBN query, the \underline{query granularity} is the maximum granularity attained by any function inducible by that query.
\end{definition}
The granularity of a TBN query---and, to an extent, its expressiveness---depends on the location of evidence, query, and testing nodes.
For example, in Figure~\ref{fig:tbn-impact}, evidence \(\lambda_2\) on node \(E_2\) does not impact CPT selection at testing node \(Q\) since \(E_2\) is not an ancestor of \(Q\). 
Hence, the regions are one-dimensional, across evidence \(\lambda_1\). Contrast this with the TBN query in Figure~\ref{fig:2tbn-dag}. In this case, CPT selection is impacted by 
evidence on both \(E_1\) and \(E_2\), leading to two-dimensional regions as in Figures~\ref{fig:plot-2tbn} and~\ref{fig:plot-step}.

For completeness, we close this section by the following immediate result.
\begin{theorem}\label{theo:nn-tac}
The function computed by a neural network with ReLU and step activation functions can be computed by a TAC with size proportional to the neural network size.\footnote{The resulting TAC may have negative parameters and thresholds.}
\end{theorem}

\begin{proof}
A testing unit \(f(x,T)\) with parameters \(\theta^+ =1\) and \(\theta^-=0\) corresponds to a step activation function with threshold \(T\).
Moreover, such a testing unit can emulate a ReLU \(g(x)\) as follows: \(g(x) = x \cdot f(x,0) = \max(0,x)\).  A neuron with a linear activation function can be simulated by an AC using multipliers and adders.  A neuron with a ReLU or step activation functions can be simulated by a TAC using the above AC and the TAC fragment of the activation function.
\end{proof}

\section{Generalized CPT Selection}
\label{sec:exp}

\begin{figure}[t]
  \centering
 \includegraphics[width=0.25\linewidth]{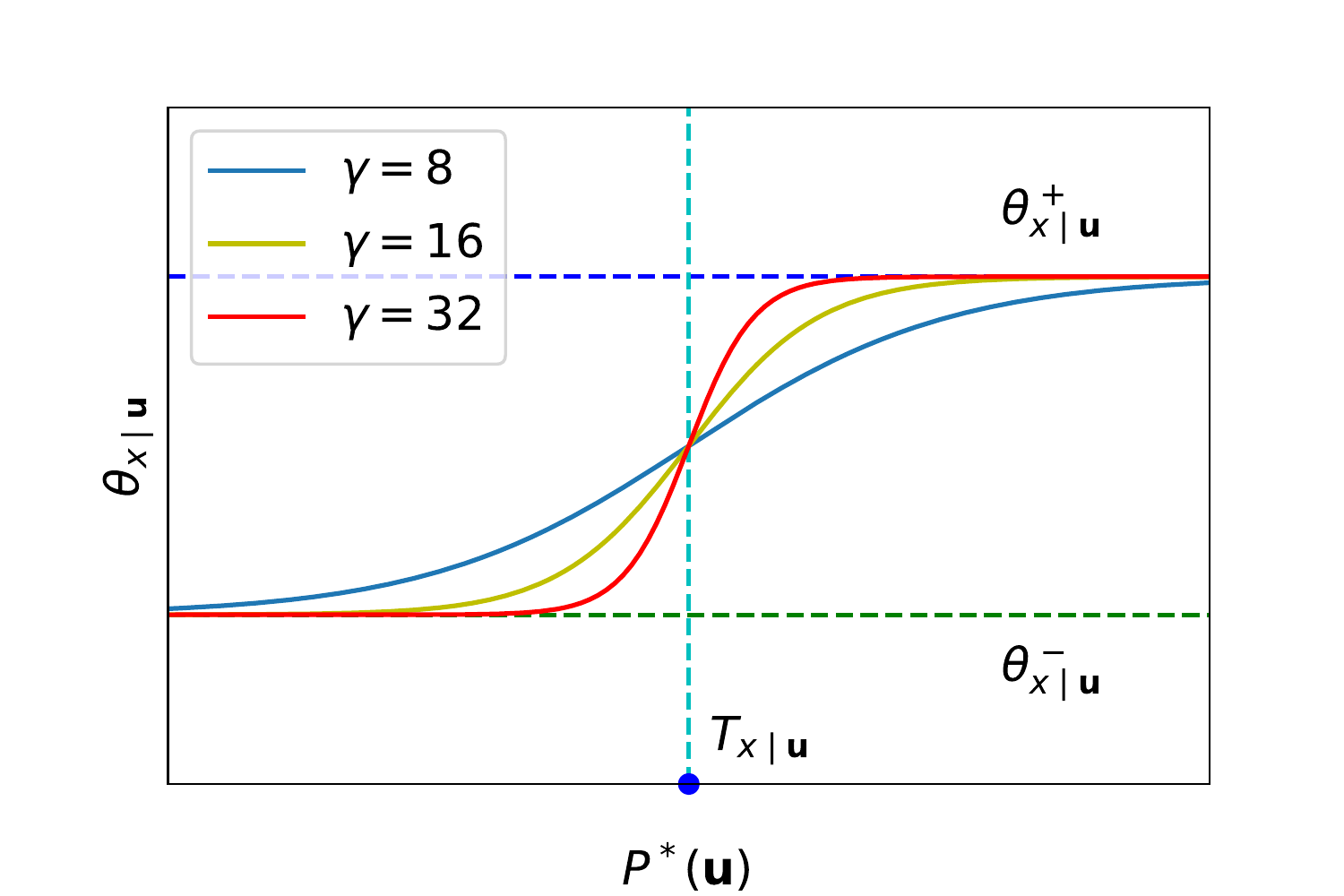}
 \caption{CPT selection using a sigmoid function.
 \label{fig:sigmoid}}
\end{figure}

TBNs get their expressiveness from the ability to select CPTs based on the available evidence. Selecting CPTs based on threshold tests, e.g., \(\Pstar(\u) \geq T_{X|\u}\),
is both simple and sufficient for universal approximation. However, one can employ more general and refined selection schemes, which can also facilitate the
learning of TAC parameters and thresholds using gradient descent methods. 
For example, one can use a sigmoid function to select CPTs based on the following equations; see Figure~\ref{fig:sigmoid}:

\begin{eqnarray}
\theta_{x|\u} 
& = & \tau_{\u} \cdot \theta^+_{x|\u} + (1-\tau_{\u}) \cdot \theta^-_{x|\u} \label{eq:sigmoid} \\
\tau_\u & = & [1+\exp\{-\gamma \cdot (\Pstar(\u) - T_{X|\u})\}]^{-1} \nonumber
\end{eqnarray}
Here, \(\gamma\) is a meta-parameter that controls the sigmoid slope and \(\tau_\u \in [0,1]\). As \(\gamma\) tends to \(\infty\), this selection scheme tends
towards implementing a threshold test. Moreover, as \(\gamma\) tends to \(0\), the selection tends towards a fixed CPT as in BNs. 
Threshold tests select a CPT for node \(X\) from a finite set of \(2^n\) different CPTs, where \(n\) is the number of states for parents \(\U\).
However, the sigmoid selects a weighted average of these CPTs, which is another CPT.\footnote{If \(P_1(X)\) and \(P_2(X)\) are distributions, 
and if \(\tau \in [0,1]\), then \(\pr(X) = \tau \cdot P_1(X) + (1-\tau) \cdot P_2(X)\) is also a distribution.}
The TAC compilation algorithm in~\ref{sec:compiling tacs} can easily accommodate this more general selection scheme, leading to TACs with sigmoid units that replace testing units.

\def\myplotwidth{0.22}

\begin{table*}[tp]
\caption{Functions \(f_1, \ldots, f_5\) and their TAC and AC approximations.
\label{table:plots}}
\vspace{3mm}
\centering
\renewcommand{\arraystretch}{1.25}
\begin{tabular}{c|c|c|c}
\(i\)  & $f_i(x,y)$ & learned TAC function & learned AC function \\ \hline
\raisebox{4.5\height}{1} &
  \includegraphics[width=\myplotwidth\linewidth]{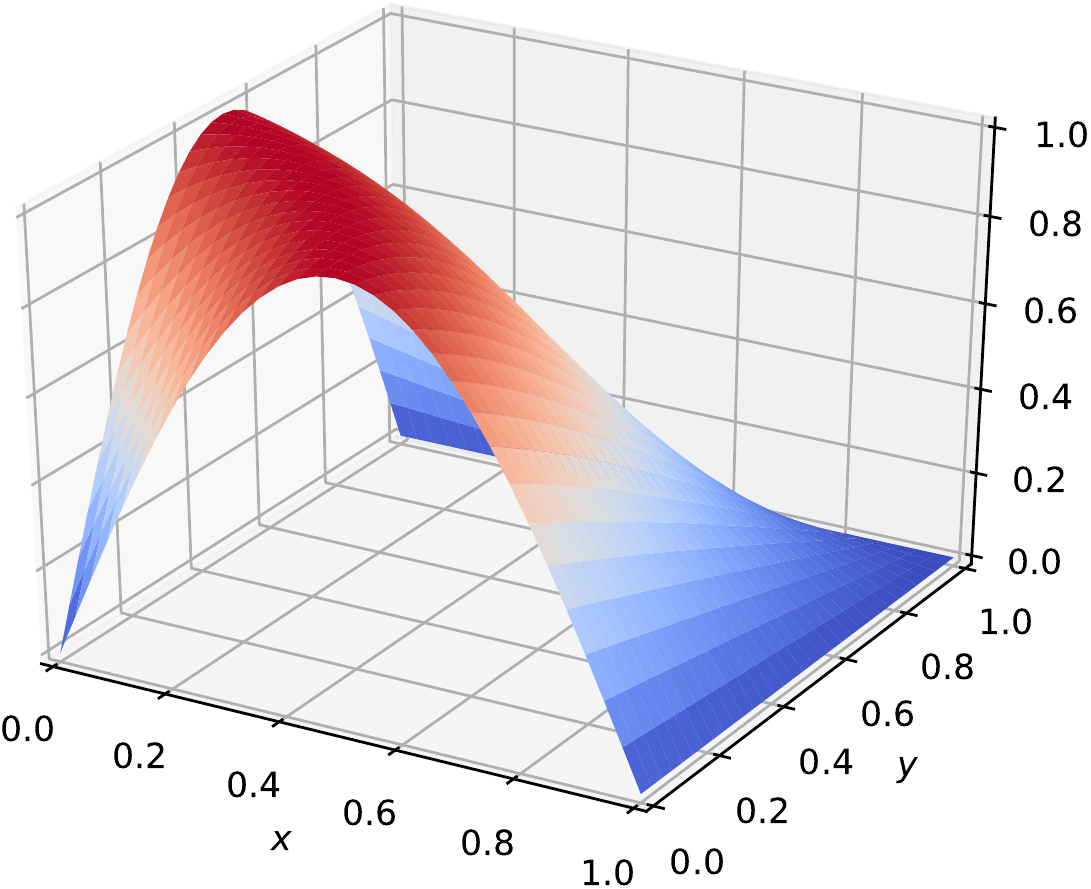} &
  \includegraphics[width=\myplotwidth\linewidth]{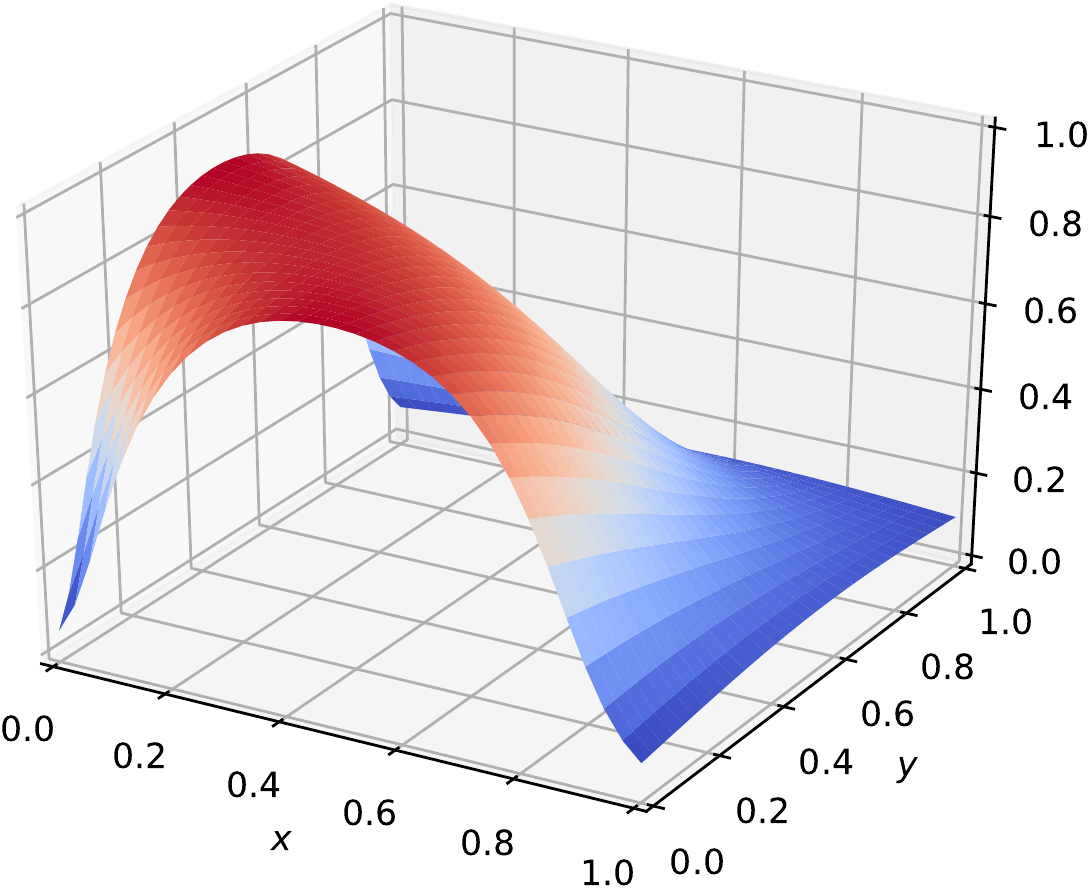} & 
  \includegraphics[width=\myplotwidth\linewidth]{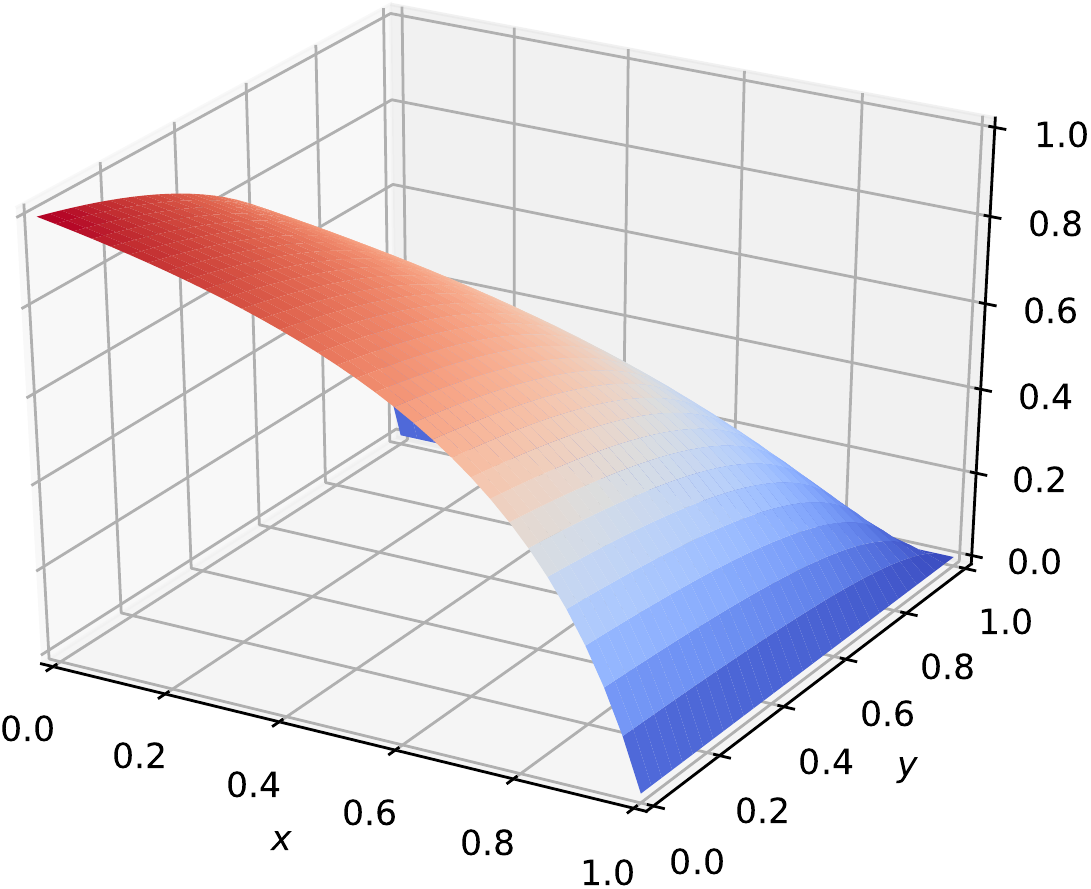} \\ \hline
\raisebox{4.5\height}{2} &
  \includegraphics[width=\myplotwidth\linewidth]{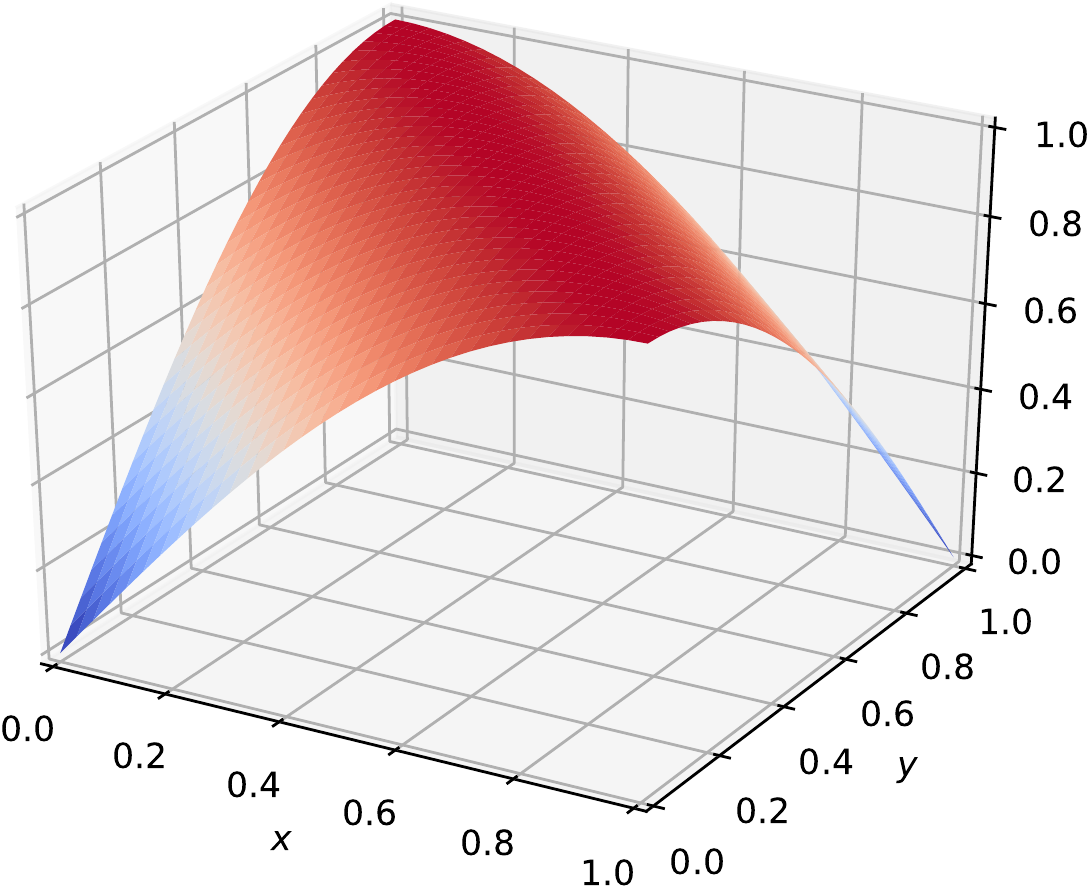} &
  \includegraphics[width=\myplotwidth\linewidth]{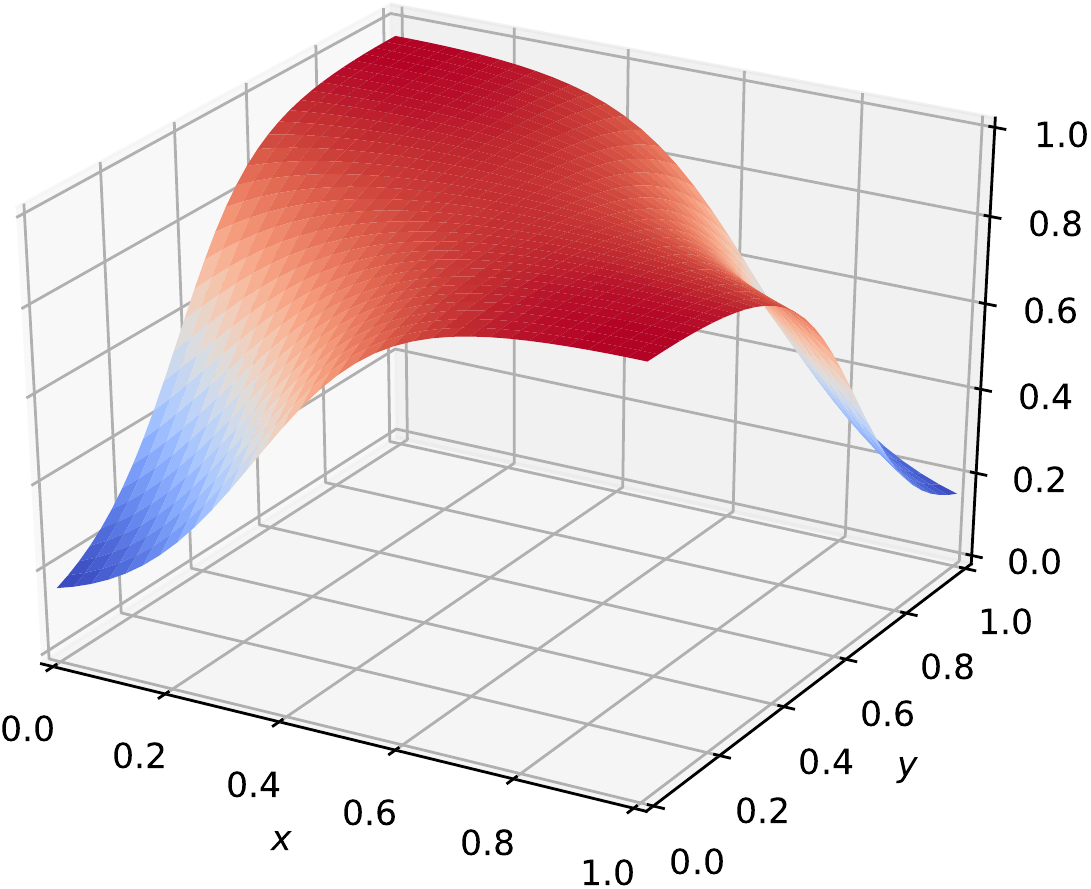} & 
  \includegraphics[width=\myplotwidth\linewidth]{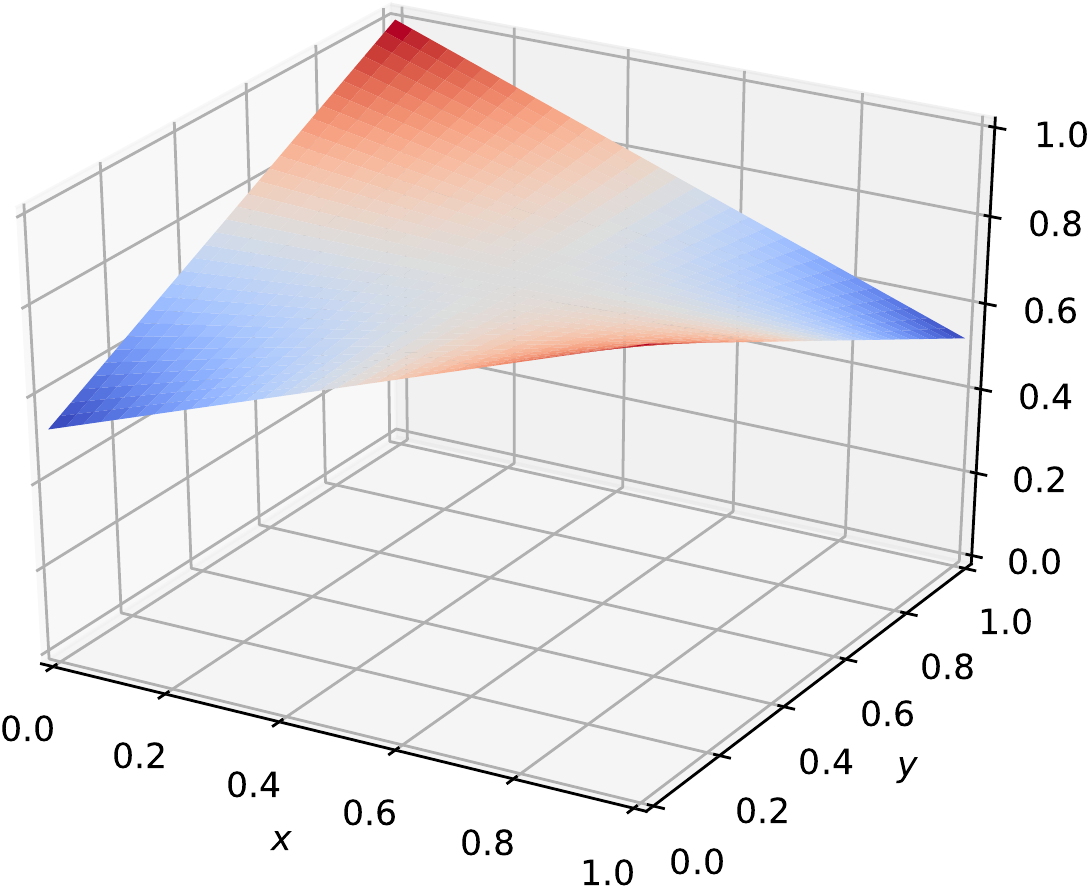} \\ \hline
\raisebox{4.5\height}{3} &
  \includegraphics[width=\myplotwidth\linewidth]{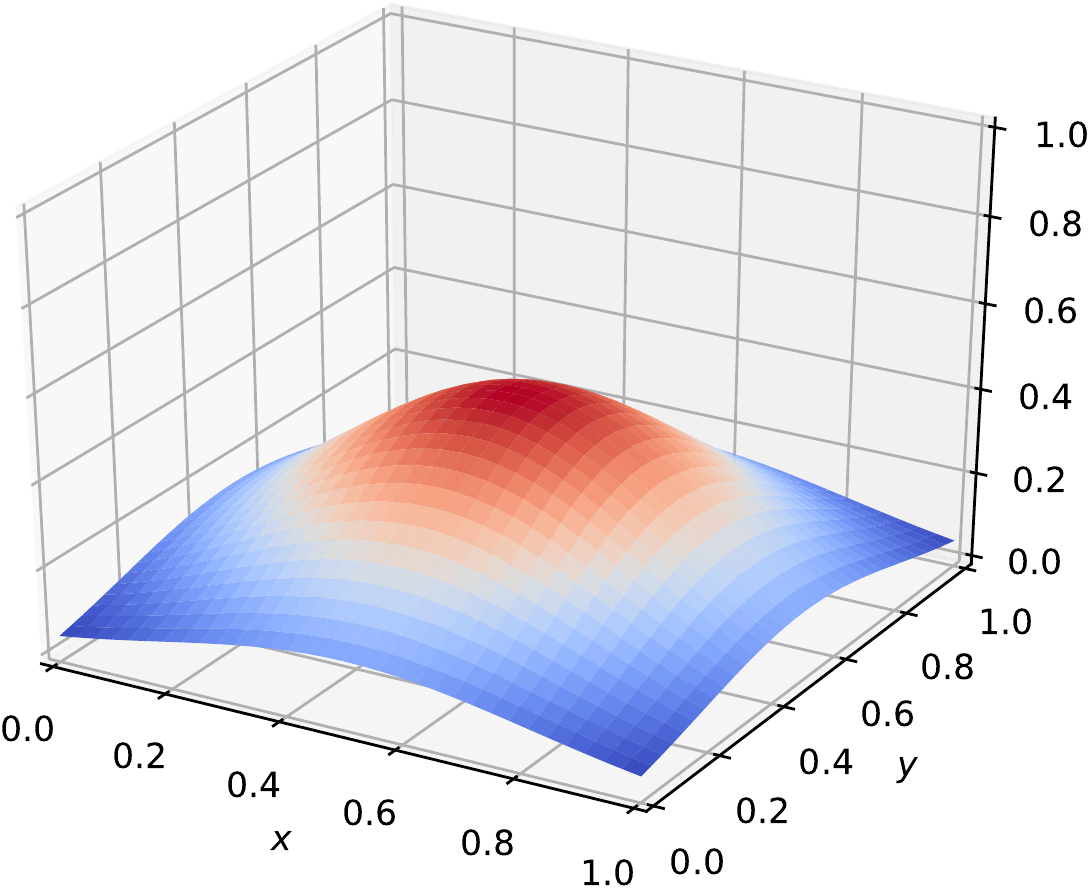} &
  \includegraphics[width=\myplotwidth\linewidth]{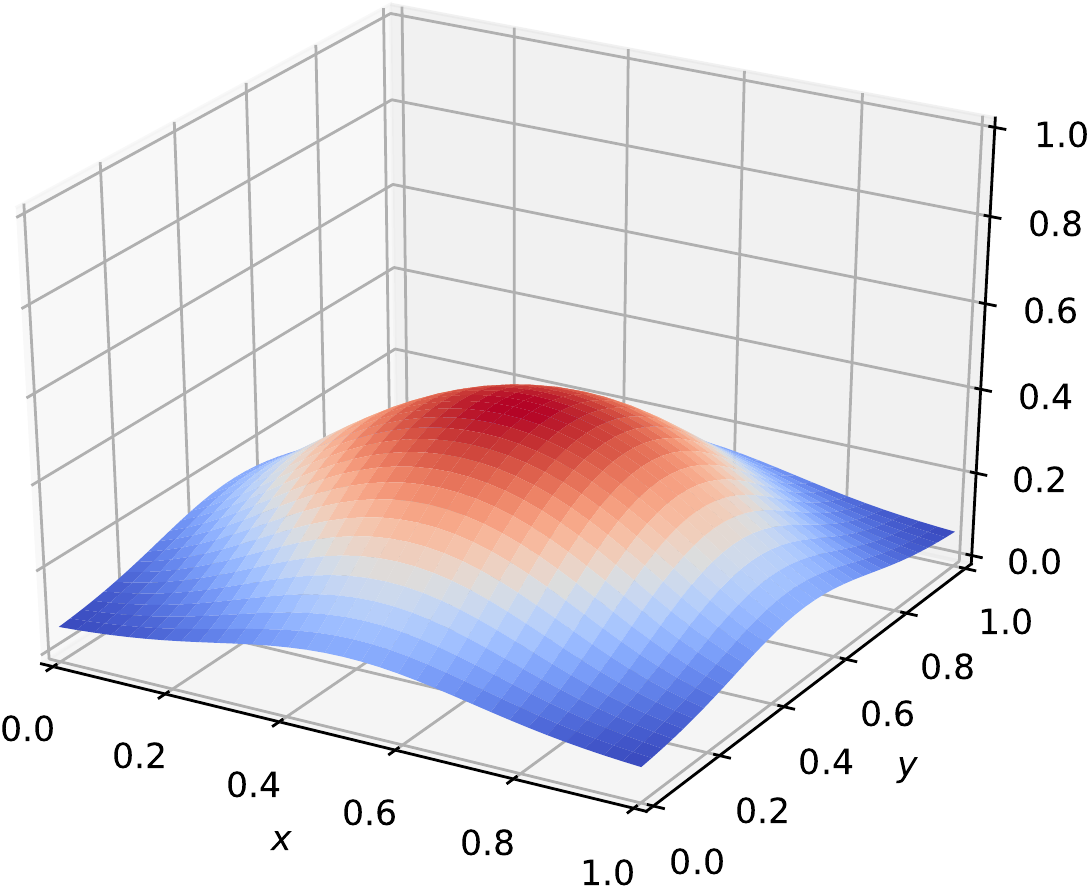} & 
  \includegraphics[width=\myplotwidth\linewidth]{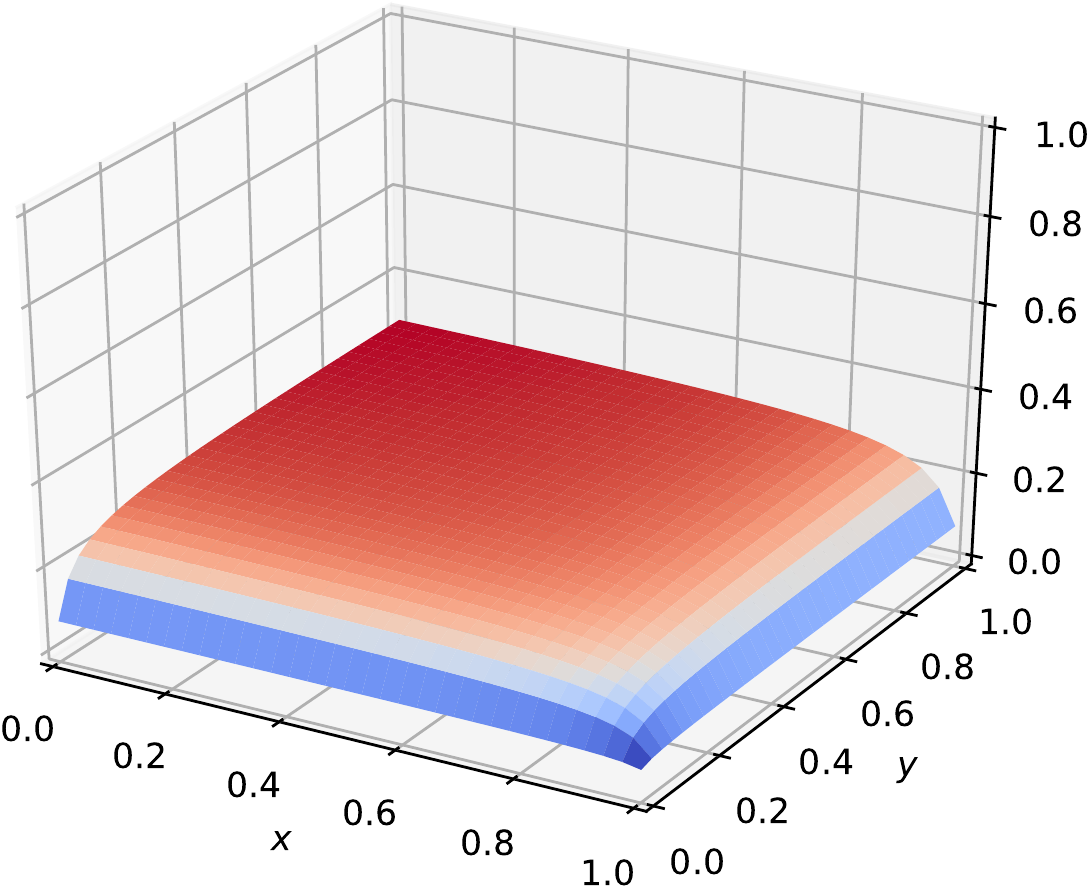} \\ \hline
\raisebox{4.5\height}{4} &
  \includegraphics[width=\myplotwidth\linewidth]{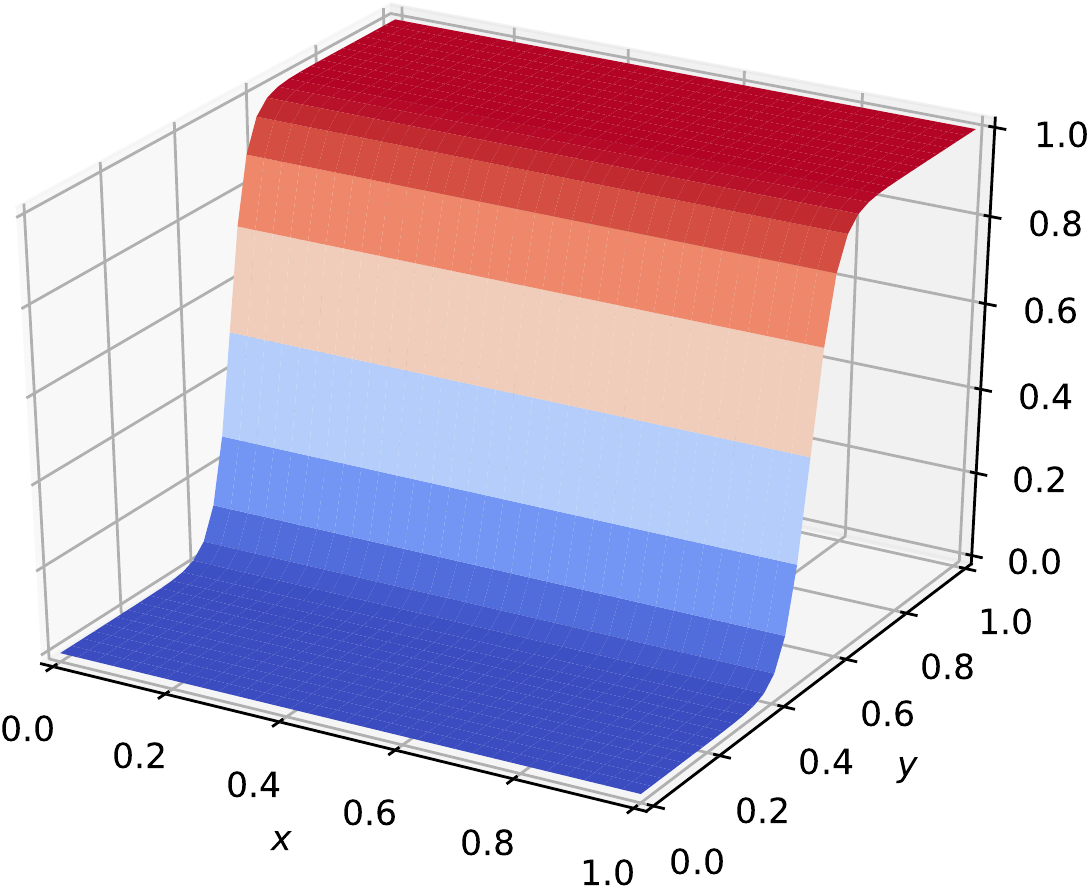} &
  \includegraphics[width=\myplotwidth\linewidth]{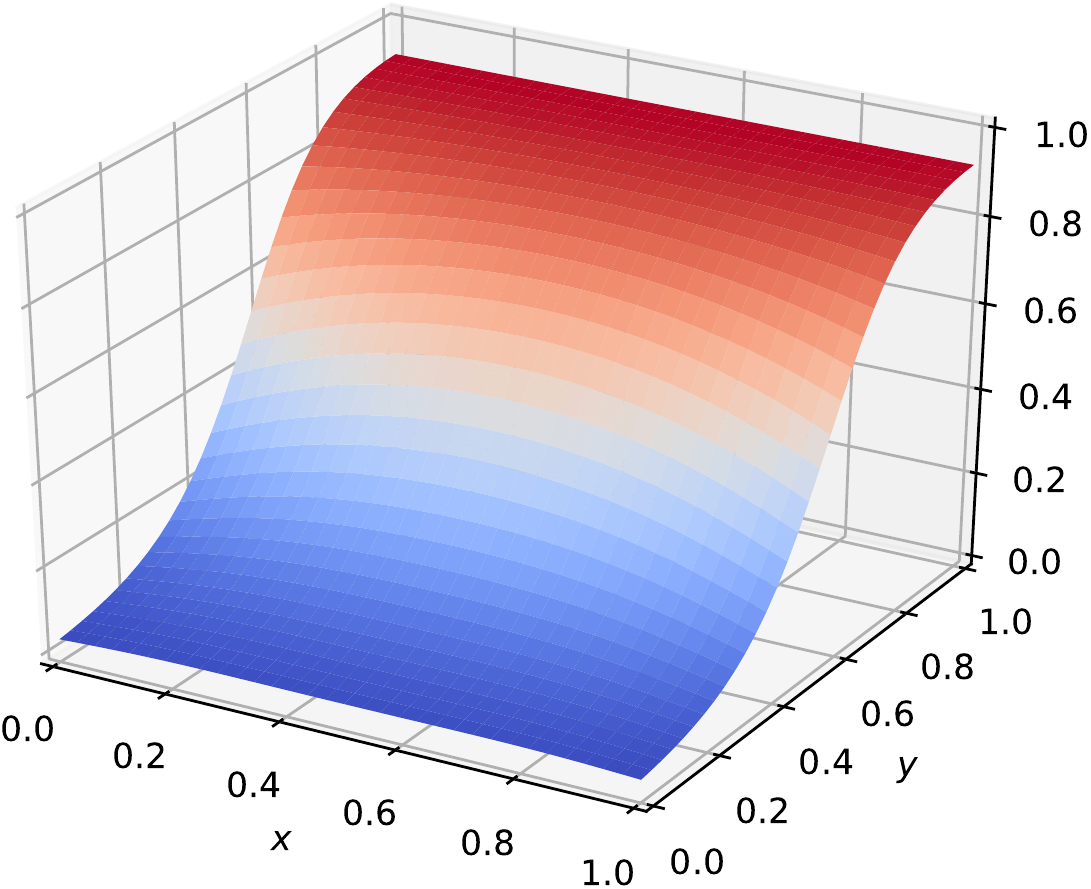} & 
  \includegraphics[width=\myplotwidth\linewidth]{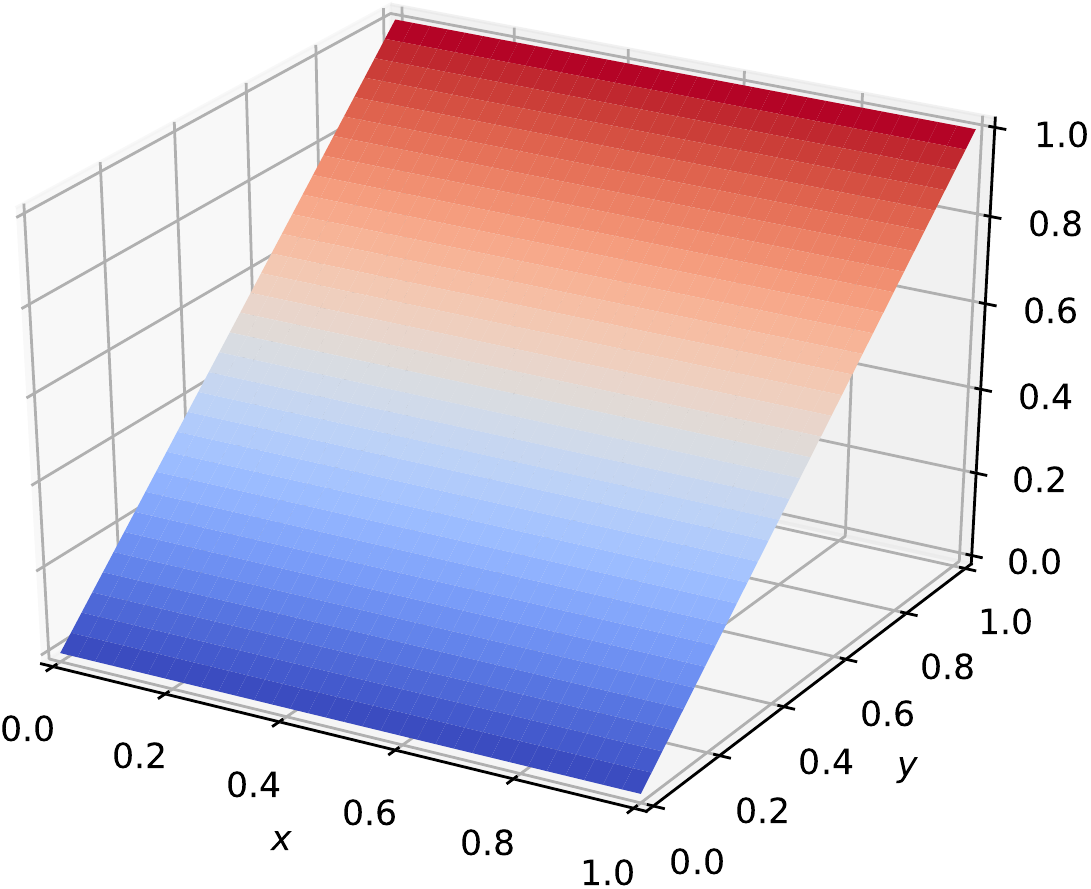} \\ \hline
\raisebox{4.5\height}{5} &
  \includegraphics[width=\myplotwidth\linewidth]{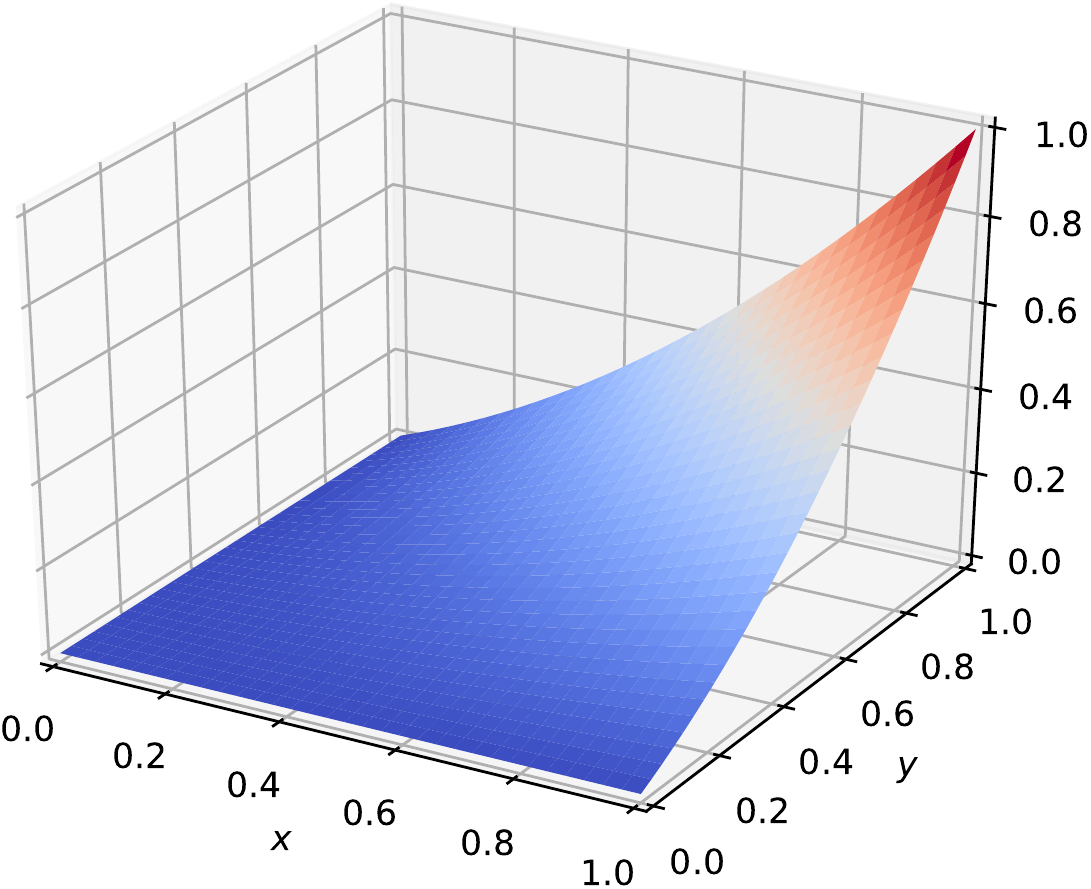} &
  \includegraphics[width=\myplotwidth\linewidth]{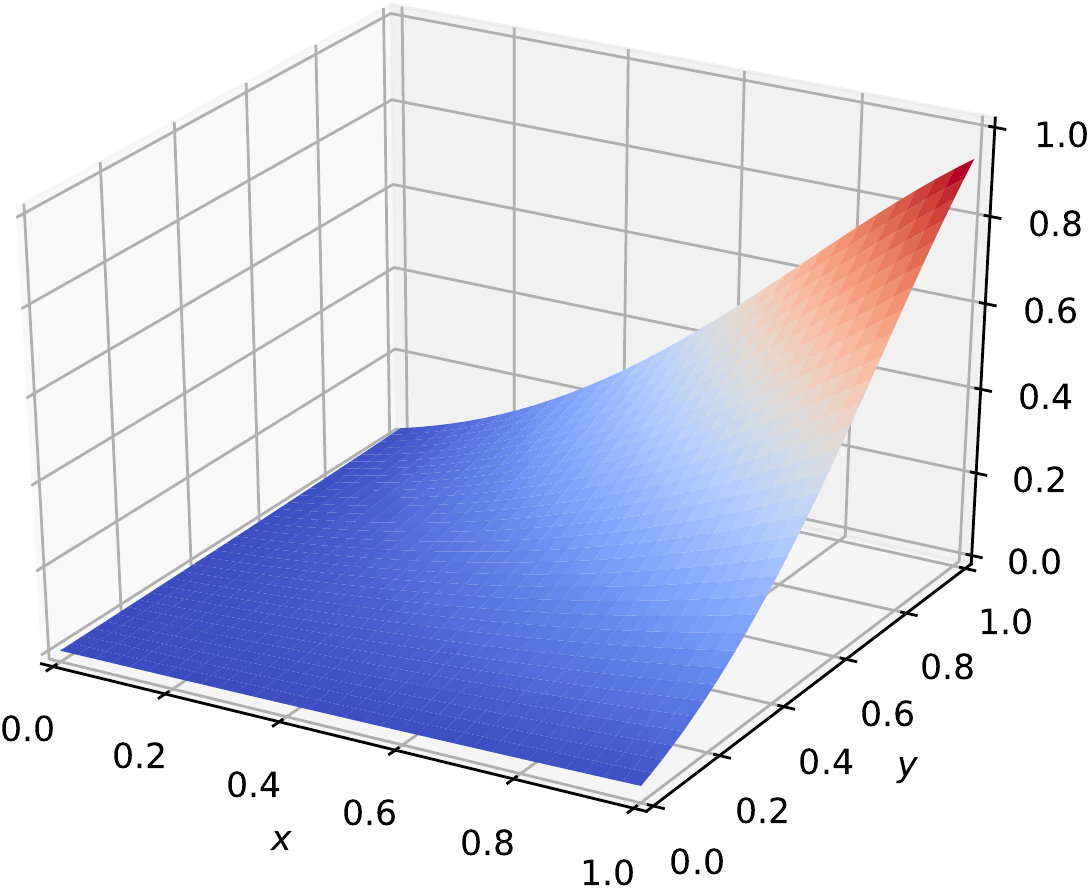} & 
  \includegraphics[width=\myplotwidth\linewidth]{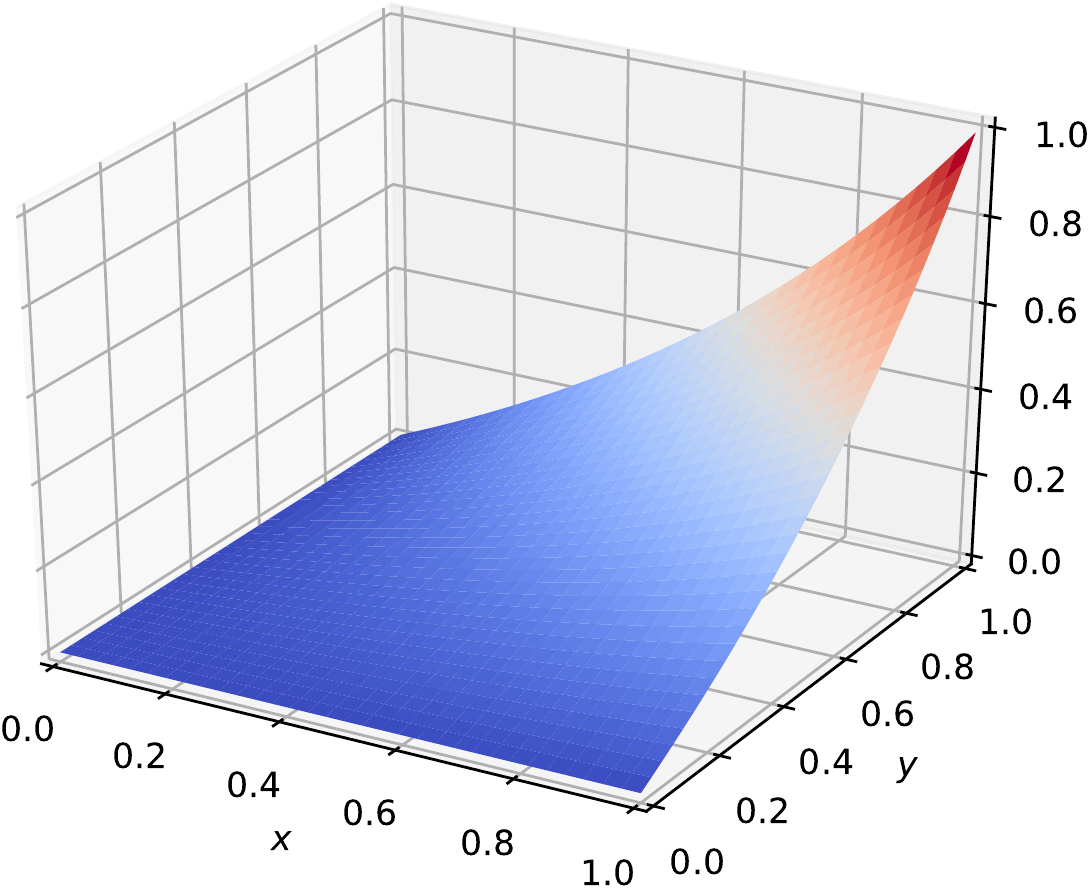} \\ \hline
\end{tabular}
\end{table*}

We next show some results on training TACs to approximate functions using labeled data and gradient descent. We considered the TBN in Figure~\ref{fig:2tbn-dag}, except that we 
used four testing nodes, \(T_1, \ldots, T_4\). We compiled a TAC assuming \(E_1\) and \(E_2\) are evidence nodes and \(Q\) is the query node.\footnote{The
TAC computes the conditional probability of \(\eql(Q,q)\) given soft evidence, \(\Pstar(q)\).}
We then generated labeled data from the following continuous functions, which are visualized in the first column of Table~\ref{table:plots}.
\begin{enumerate}
\item[] \(f_1(x,y) = \sin\left( \pi (1-x) (1-y) \right)\)
\item[] \(f_2(x,y) = \sin\left( \frac{\pi}{2} \cdot (2-x-y) \right)\)
\item[] \(f_3(x,y) = \frac{1}{2} \cdot e^{-5 \cdot (x-\frac{1}{2})^2 - 5 \cdot (y-\frac{1}{2})^2}\)
\item[] \(f_4(x,y) = \left( 1+e^{-32 (y-\frac{1}{2})} \right)^{-1}\)
\item[] \(f_5(x,y) = \frac{1}{2} \cdot xy (x+y)\) 
\end{enumerate}
For each one of these functions, each mapping values from \([0,1]^2\) to \([0,1]\), 
we selected data examples \((x,y)\) by dissecting the input space into a \(32 \times 32\) grid and then recording the resulting values of \(f(x,y)\) as labels.  
We then trained the TAC on this data using TensorFlow by minimizing mean-squared error as described in \ref{sec:train}.

The second column of Table~\ref{table:plots} depicts the TAC approximations of these functions. 
The TAC approximates well all functions except for \(f_4\), which is a sigmoid function with a steep slope.
For comparison, we also depict the AC approximations of these functions in the third column of Table~\ref{table:plots}.
Since ACs can only represent multi-linear functions or their quotients, we see that only function \(f_5\) is approximated well.
Moreover, the AC approximation of function \(f_4\) is worse than the TAC approximation (the AC effectively learns a plane). 

\def\myotherplotwidth{0.2}

\begin{table*}[t]
\caption{Functions \(g_1\) and \(g_2\) and their TAC approximations with increasing number of testing nodes.  \label{table:granularity}}
\vspace{3mm}
\centering
\setlength{\tabcolsep}{4pt}
\begin{tabular}{c|c|c|c|c}
\(i\) & \(g_i(x,y)\) & 4-test TBN & 8-test TBN & 16-test TBN \\ \hline
\raisebox{4.5\height}{1} &
\includegraphics[width=\myotherplotwidth\linewidth]{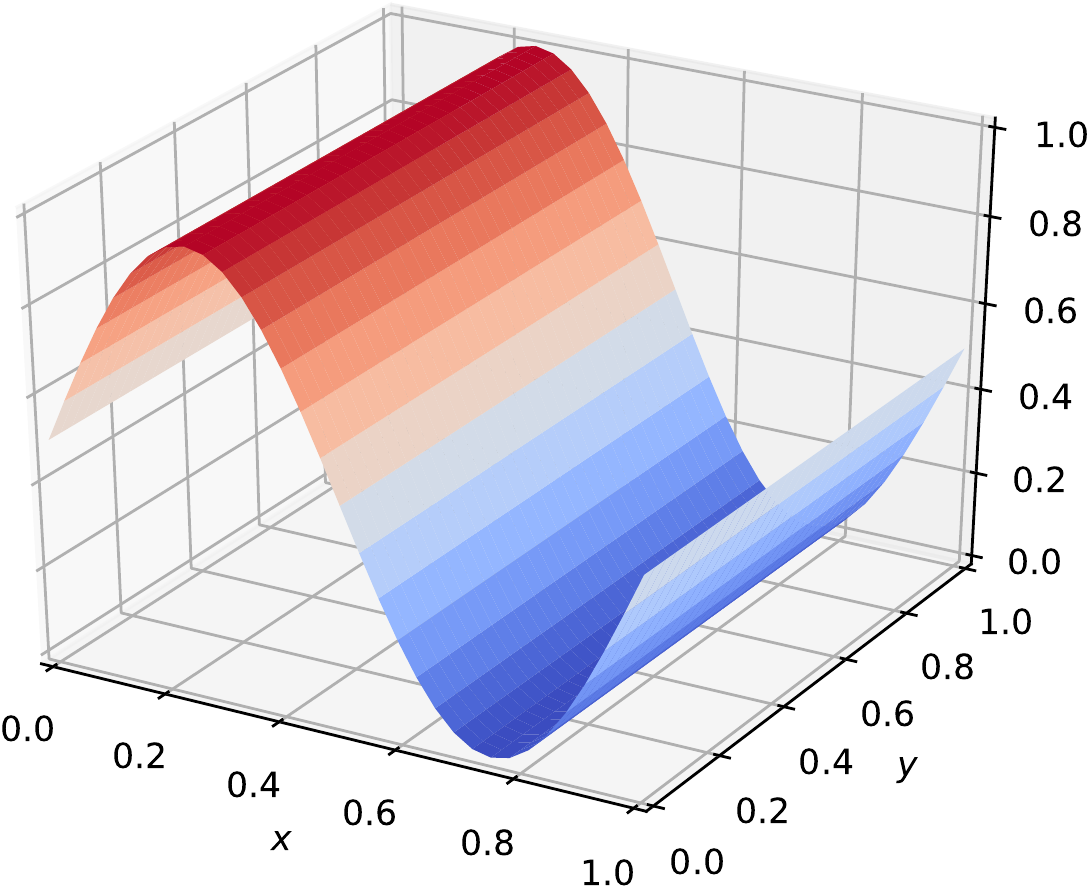} & 
\includegraphics[width=\myotherplotwidth\linewidth]{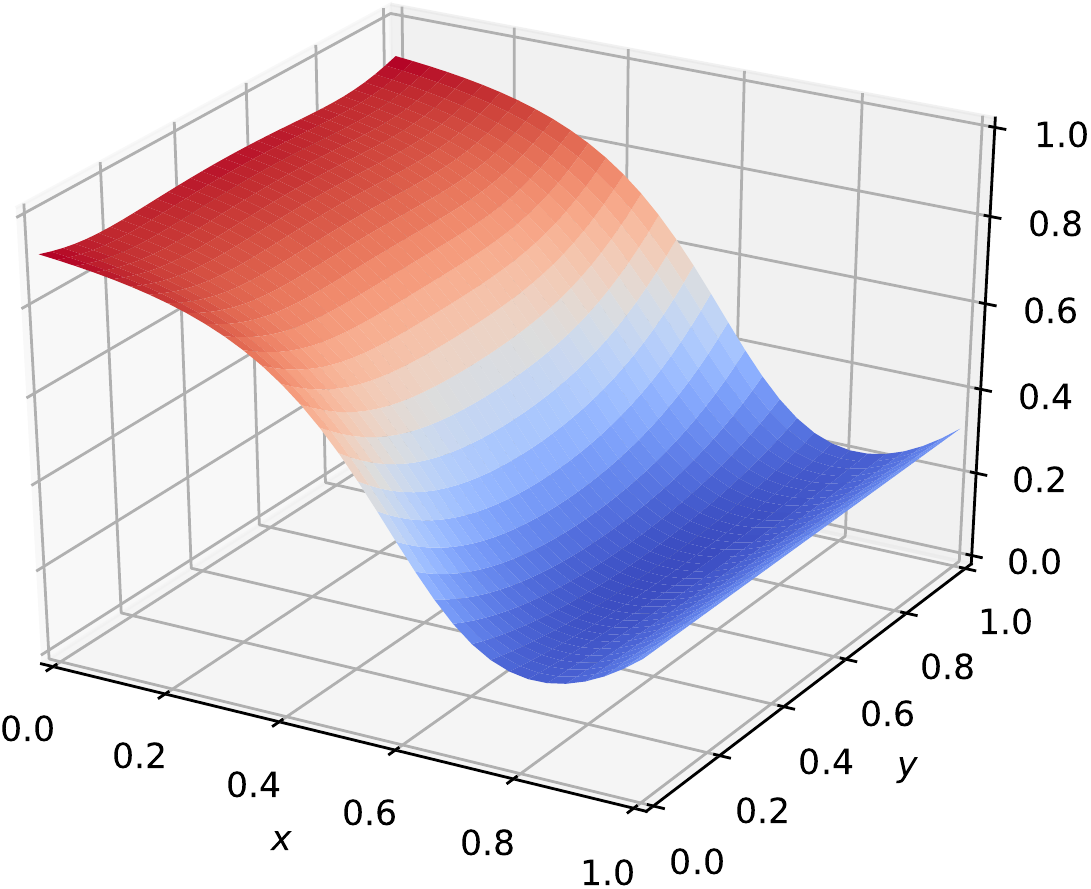} & 
\includegraphics[width=\myotherplotwidth\linewidth]{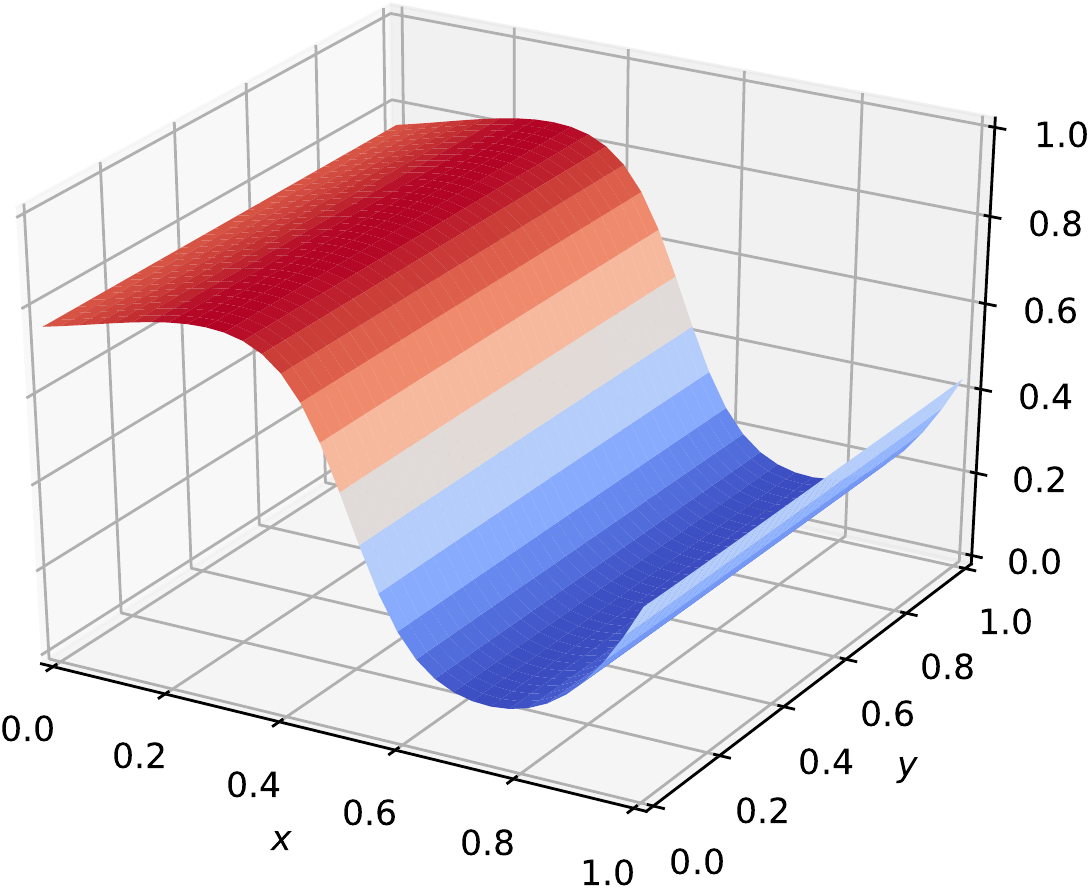} & 
\includegraphics[width=\myotherplotwidth\linewidth]{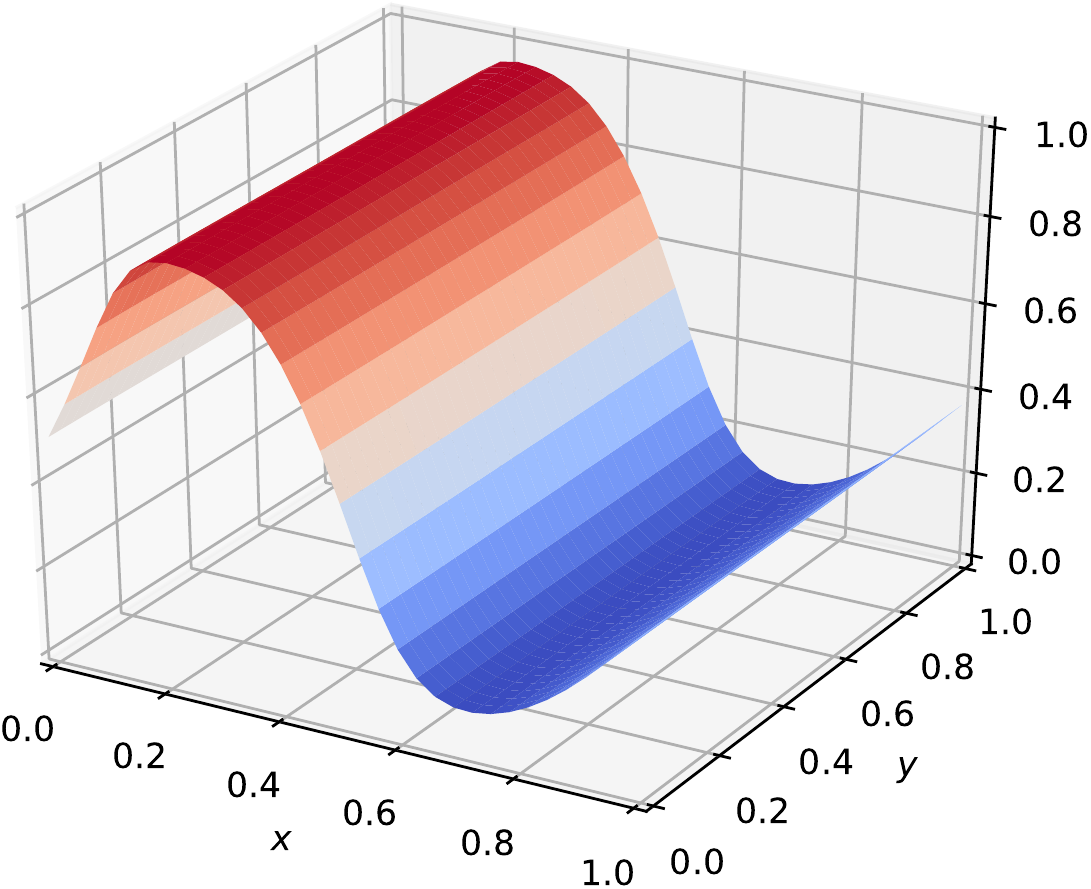} \\\hline
\raisebox{4.5\height}{2} &
\includegraphics[width=\myotherplotwidth\linewidth]{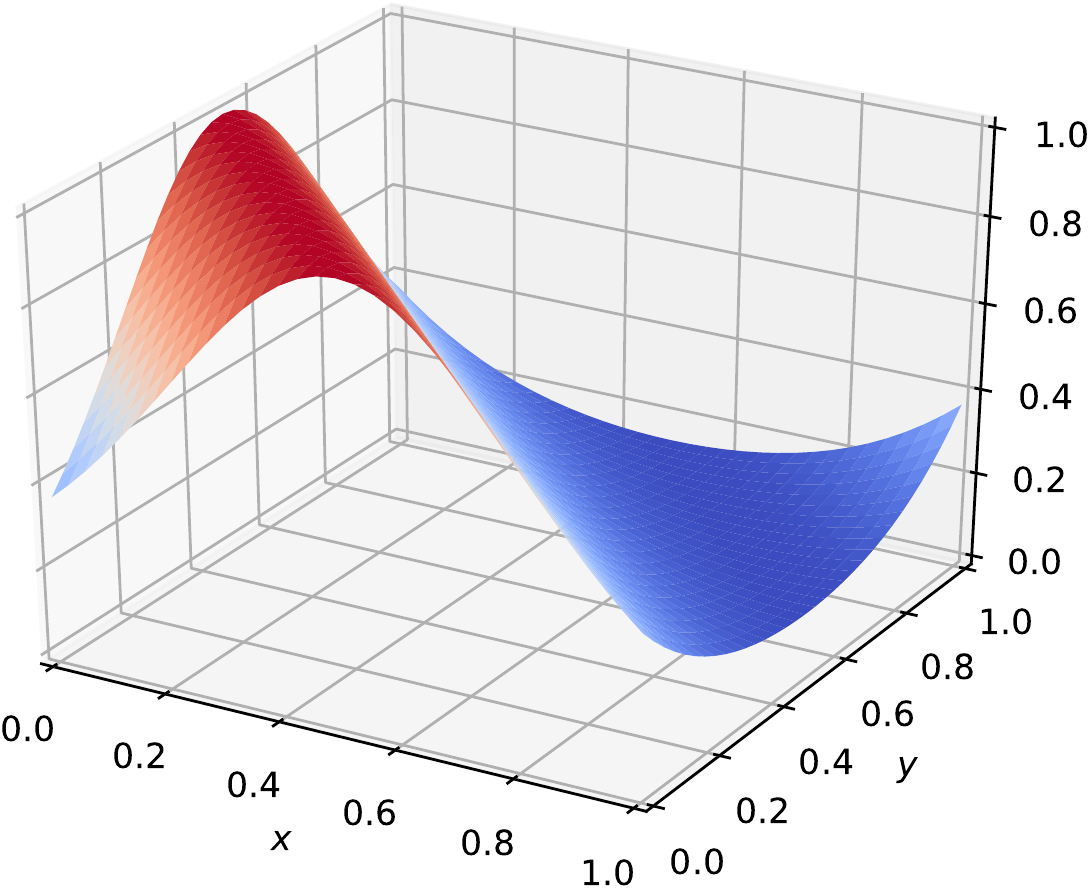} & 
\includegraphics[width=\myotherplotwidth\linewidth]{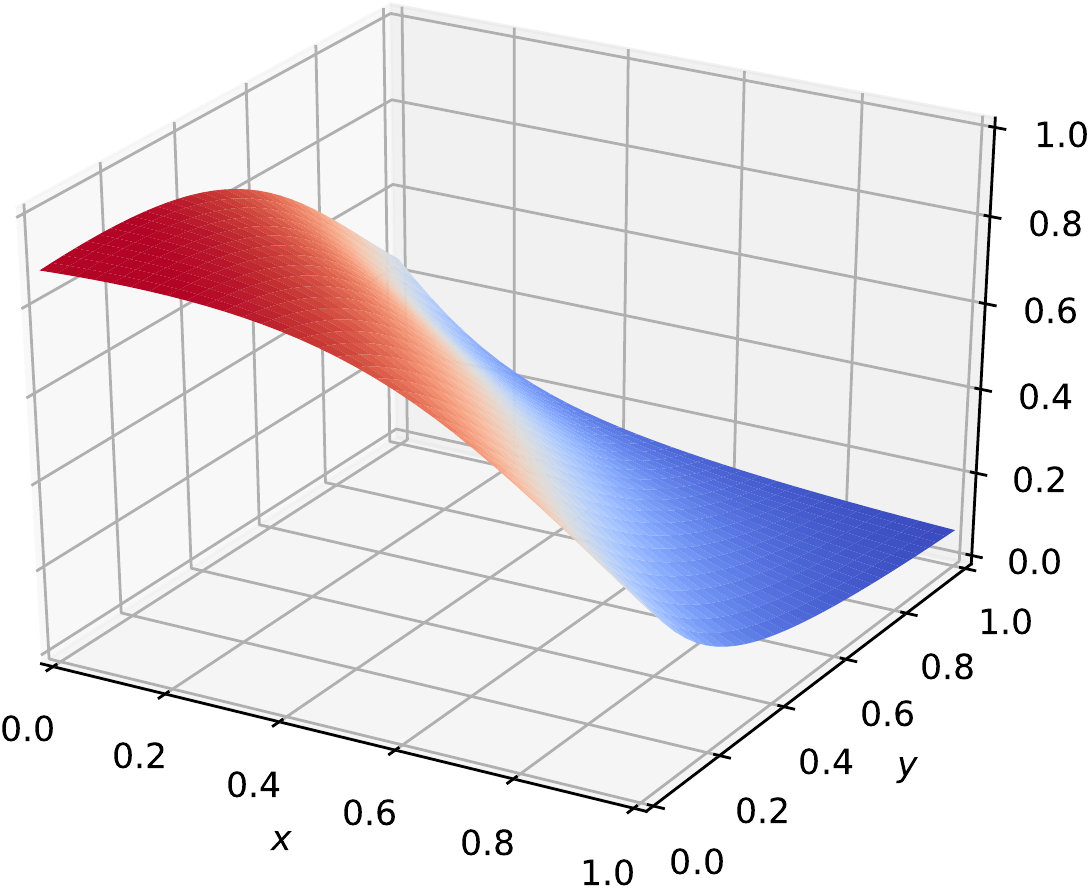} & 
\includegraphics[width=\myotherplotwidth\linewidth]{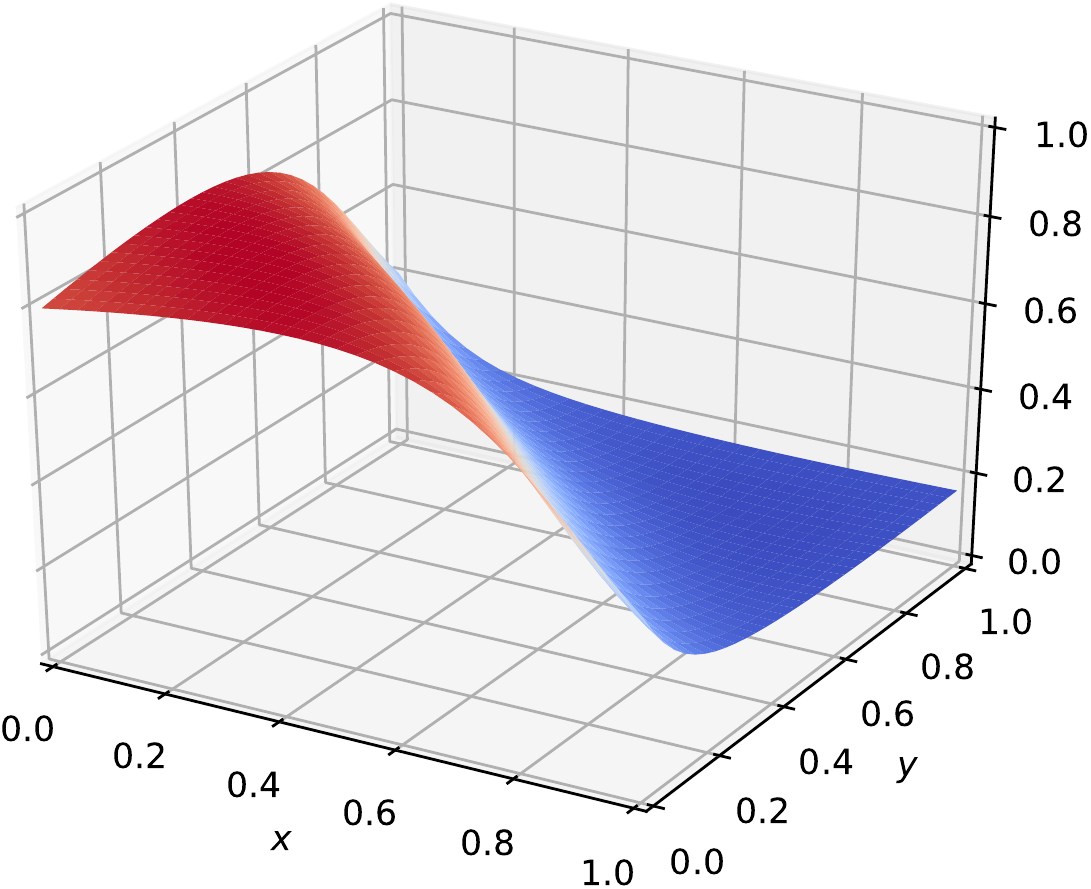} & 
\includegraphics[width=\myotherplotwidth\linewidth]{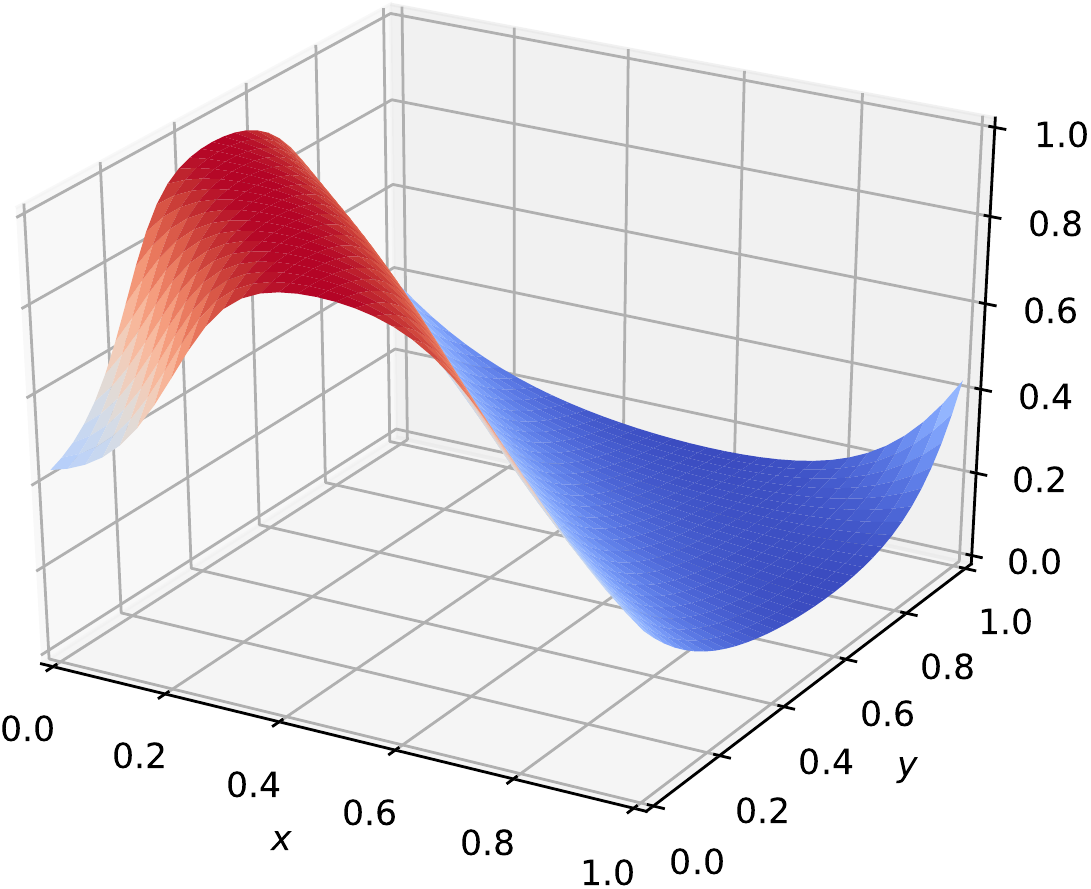} 
\end{tabular}
\end{table*}

We now consider two additional functions that are visualized in the first column of Table~\ref{table:granularity}:
\begin{enumerate}
\item[] \(g_1(x,y) = \frac{1}{2}+\frac{1}{2} \sin(2 \pi x)\)
\item[] \(g_2(x,y) = e^{\sin(\pi(x+y))-1}\)
\end{enumerate}
We approximated these functions using the same TBN query in Figure~\ref{fig:2tbn-dag}, but increased the number of testing nodes from \(4\) to \(8\) and finally to \(16\).
As the number of testing nodes increased, so did the number of sigmoid units in the compiled TACs.
The TAC approximations are depicted in the 2nd, 3rd and 4th columns of Table~\ref{table:granularity}.
As we increase the number of sigmoid units in the compiled TACs, we expect their expressiveness to also increase, leading to better approximations.  
This expectation is confirmed by Table~\ref{table:granularity}.

\section{Conclusion}
\label{sec:conclusion}

We considered the relative expressiveness of Bayesian and neural networks.  Neural networks are ``universal approximators'' of continuous functions, 
whereas marginal Bayesian network queries correspond to multi-linear functions (joint queries) or their quotients (conditional queries).  We proposed Testing Bayesian Networks (TBN) whose marginal queries are also ``universal approximators,'' and Testing Arithmetic Circuits (TAC) for computing these queries.  Moreover, we showed that marginal TBN queries and their TACs represent piecewise multi-linear functions, while highlighting that neural networks with ReLU activation functions
represent piecewise linear functions.
Finally, we generalized the concept of CPT selection that is responsible for the expressiveness of TBNs, which facilitated their training using gradient descent.  

TBNs and TACs move us a step forward towards fusing model-based and function-based approaches to AI. In particular, they provide a framework for integrating background knowledge into expressive 
functions that can be learned from labeled data. This can contribute to learning more robust functions based on less data, and to verifying, interpreting and explaining learned functions.

\section*{Acknowledgments}

We thank Yujia Shen, Andy Shih, and Yaacov Tarko for comments and
discussions on this paper.  This work has been partially supported by
NSF grant \#IIS-1514253, ONR grant \#N00014-18-1-2561 and DARPA XAI
grant \#N66001-17-2-4032.

\appendix

\section{Training TACs}
\label{sec:train}

We implemented a training algorithm for TACs using TensorFlow, which is based on a symbolic representation of the loss function to be minimized.  
In this tool, gradients are computed automatically and an optimizer (e.g., gradient descent method) can be used to minimize the given loss function.
Using TensorFlow terminology, we assume training data with real-valued {\em features} in \([0,1]\) and real-valued {\em labels} in \([0,1]\). 

TAC inputs assert soft evidence while the TAC output computes the conditional probability of some variable \(Q\) given soft evidence.  
The goal is to learn TAC parameters and thresholds that minimize the difference between TAC outputs and the labels of corresponding training examples.
We used mean squared error as our loss function:
\[
\frac{1}{N} \sum_{i=1}^N (TAC(\lambda_i) - y_i )^2.
\]
We have \(N\) training examples, with the \(i\)-th example consisting of input vector \(\lambda_i\) and label \(y_i\).  \(TAC(\lambda_i)\) is the TAC output under input vector \(\lambda_i\). 
It is the result of normalizing two TAC nodes for \(\pstar(q)\) and \(\pstar(\n(q))\).  Since the parameters of our TAC are probabilities, we use a logistic function to represent them:
\[
\theta_{x} = \frac{1}{1 + \exp\{-\tau_{x}\}}
\quad\quad
\theta_{\n(x)} = \frac{\exp\{-\tau_x\}}{1 + \exp\{-\tau_{x}\}}
\] 
and optimize instead the real-valued meta-parameters \(\tau_x\).
Our training algorithm assumed TACs with sigmoid selection units as given by Equation~\ref{eq:sigmoid} (we used \(\gamma = 16\)).\footnote{Testing 
units (see Section~\ref{sec:tac}) are challenging for gradient descent methods as they are not differentiable at \(x=T\) and have a zero 
gradient everywhere else (with respect to input \(x\)).}
Finally, while our description of the training algorithm is based on binary variables, the treatment applies directly to multi-valued variables too.

\section{Universal Approximation Theorem}
\label{sec:uat2}

\begin{figure}[t]
 \centering
 \subfigure[monotonic function]{\label{fig:monotonic}
   \includegraphics[width=0.22\linewidth]{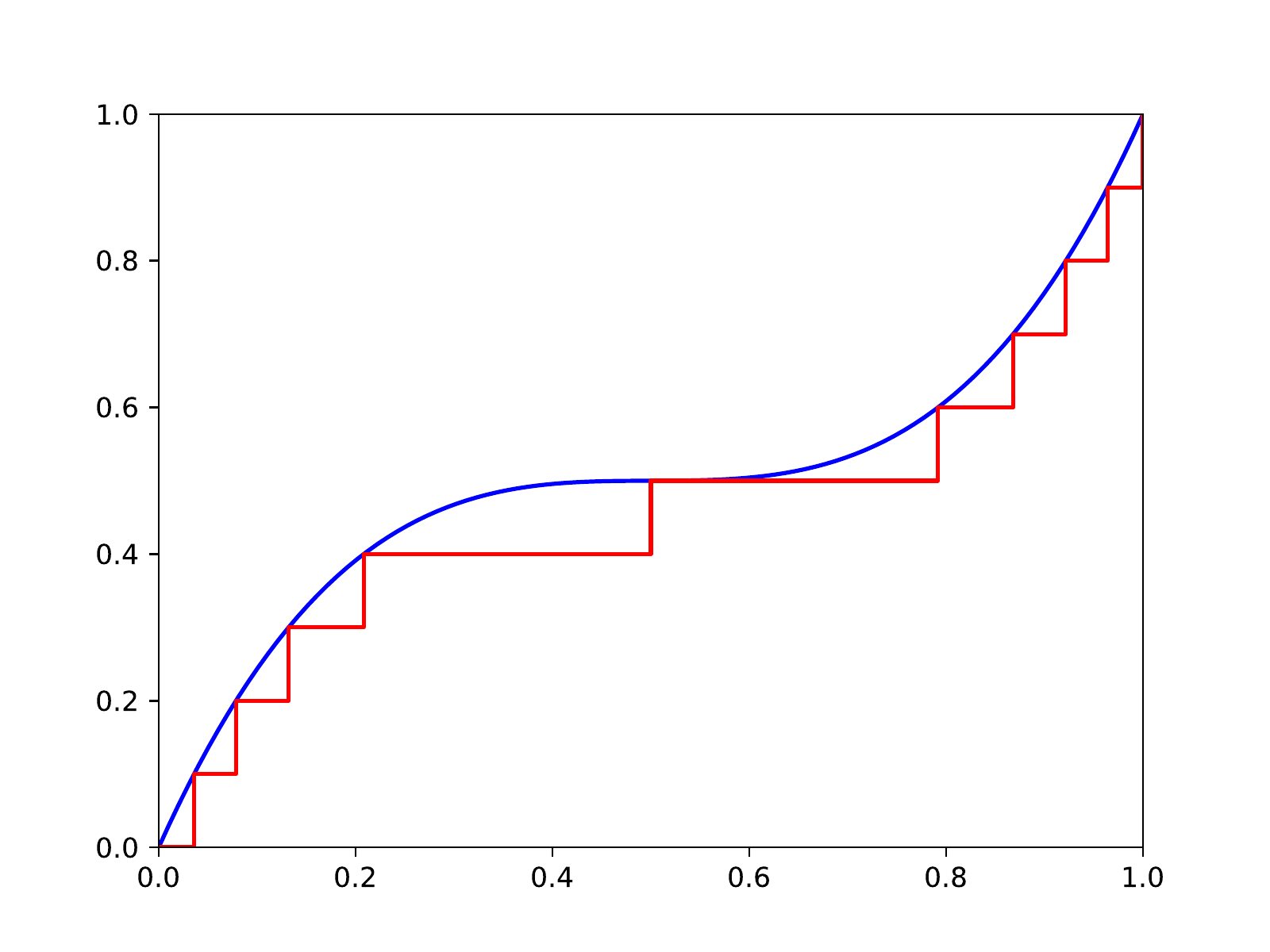}} 
\qquad
 \subfigure[non-monotonic function]{\label{fig:updown}
   \includegraphics[width=0.22\linewidth]{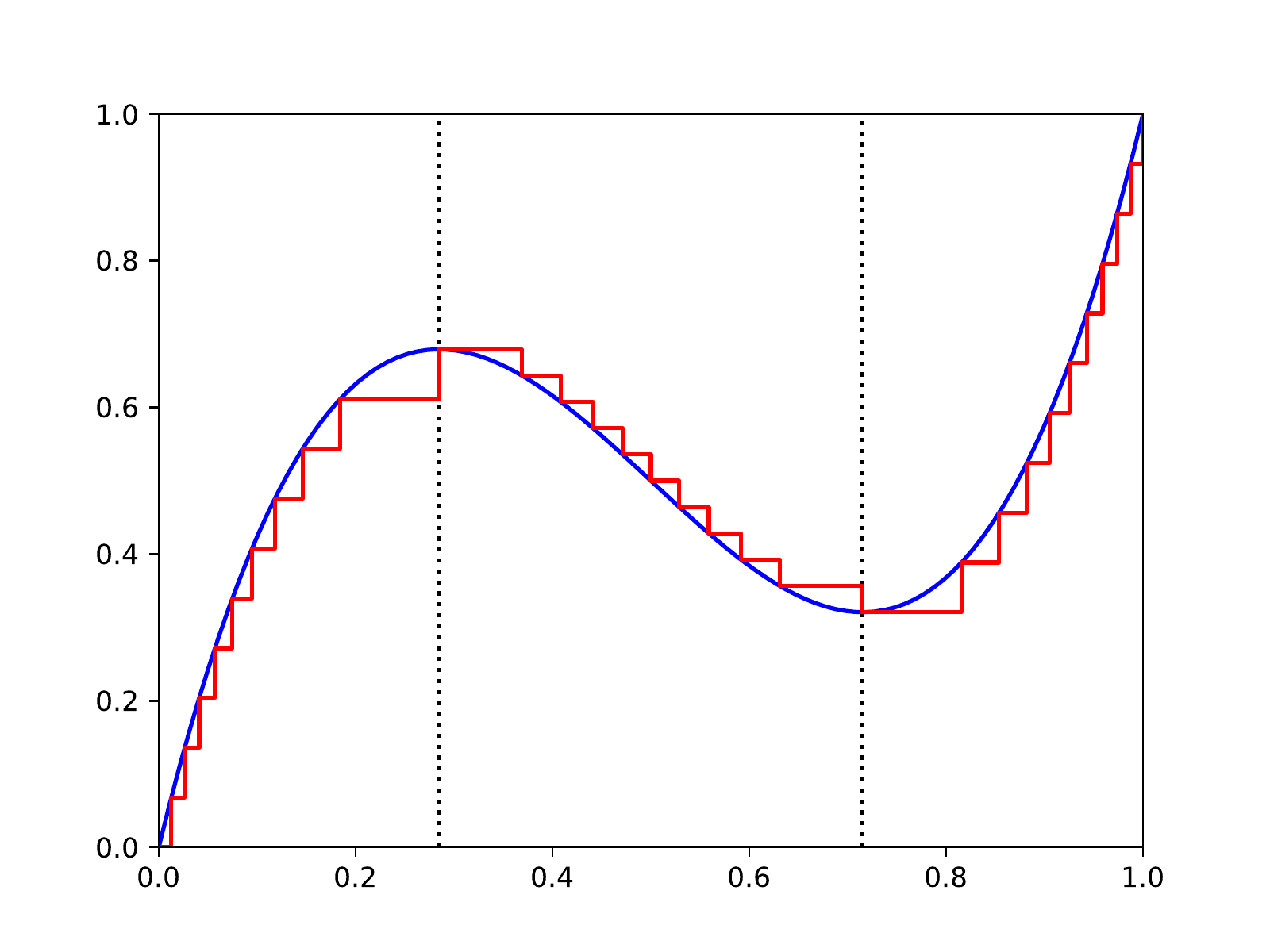}}
 \caption{Two functions and their approximations.}
\end{figure}

\begin{figure}[t]
  \centering
 \includegraphics[width=0.50\linewidth]{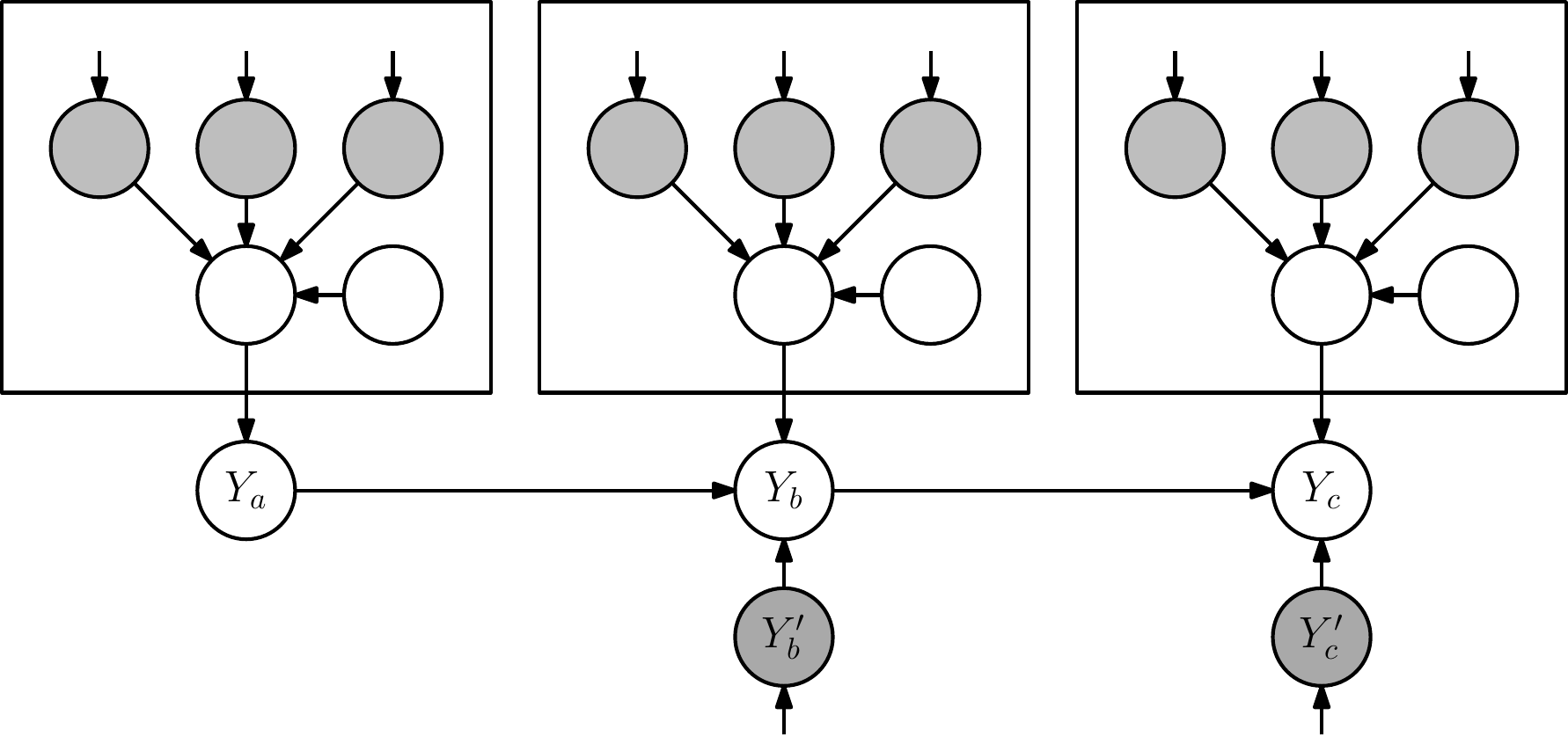}
 \caption{A chain of TBNs.}
 \label{fig:chain}
\end{figure}

We next extend Theorem~\ref{theo:uat} of Section~\ref{sec:uat} to non-monotonic functions and then to multivariate functions.

Figure~\ref{fig:updown} depicts a non-monotonic function \(f(x)\) from \([0,1]\) to \([0,1]\).  
To approximate \(f(x)\) with a TBN, we first split our function into monotonic pieces as shown in the figure (we used vertical dotted lines to mark the points 
where the sign of the first derivative \(\frac{df}{dx}\) changes).

Consider now the TBN of Figure~\ref{fig:chain} that approximates this non-monotonic function using the TBN in Figure~\ref{fig:TBN} for approximating monotonic functions as a building block (for clarity, in Figure~\ref{fig:chain}, we do not draw node \(Z\) on which soft evidence is asserted). Each of the three sub-networks enclosed in a box approximates each of the three monotonic components of our function \(f(x)\).  Let \(a,b\) and \(c\) denote the three components, and let \(T_{ab}\) and \(T_{bc}\) denote the values of \(x\) at the two borders.  We have a chain of nodes \(Y_a \to Y_b \to Y_c\) at the bottom of our network.  Node \(Y_a\) simply copies the value of its parent (its CPT is an equivalence constraint).  Node \(Y_b\) will either copy the value of \(Y_a\) or the value of its component \(b\), depending on whether the input \(x\) is below or above the threshold \(T_{ab}\).  Node \(Y_c\) will either copy the value of \(Y_b\) or the value of its component \(c\), depending on whether the input \(x\) is below are above the threshold \(T_{bc}\).  Testing nodes \(Y^\prime_b\) and \(Y^\prime_c\) perform these threshold tests and select CPTs that clamp themselves to a particular value, which determines where \(Y_b\) and \(Y_c\) copy their values from.  This sequence of threshold tests will, given an input \(x\), select the approximation of \(f(x)\) from the appropriate component, which is finally the probability of \(Y_c.\)  The size of this construction is linear in the number of times the sign of the first derivative changes.

The generalization to multivariate functions is analogous to the approximation of functions using ``ridge'' and ``bump'' functions; see, e.g., \citep{Jones90,LapedesF87}.  First, we use our construction for approximating a univariate function as a building block to approximate a function \(f(x_1,x_2)\) over two variables.  In particular, we construct a TBN for \(N\) univariate functions \(f_{x_2}(x_1) = f(x_1,x_2)\) for \(N\) values of \(x_2\) from \(0\) to \(1\).  As we did previously for approximating non-monotonic univariate functions by pieces, we construct a chain of these \(N\) TBNs and copy the output of the appropriate component based on the input value of \(x_2\).  To approximate a function \(f(x_1,\ldots,x_n)\) over \(n\) variables, we construct \(N\) TBNs that approximate functions \(f_n(x_1,\ldots,x_{n-1})\) over \(n-1\) variables, and then perform a similar construction.

The error in the approximation can be improved arbitrarily by increasing \(N\) (under some assumptions, i.e., the change in \(f\) is bounded for small changes in the input). Moreover, this construction is exponential in the number of input variables \(n\).  Related constructions for showing neural networks (with one or two hidden layers) are ``universal approximators'' are also exponential in the dimension of the function.  

\section{Compiling TBN Queries into TACs}
\label{sec:compiling tacs}

A BN query can be compiled into an AC; see Figure~\ref{fig:function}.  
Similarly, a TBN query can be compiled into a TAC; see Figure~\ref{fig:compiled tac}.
Compiling an AC can be done by keeping a symbolic trace of an elimination algorithm as described in~\cite[Chapter~12]{Darwiche09}.  
We next provide an algorithm for compiling a TAC by keeping a symbolic trace of the factor elimination algorithm described in~\citep[Chapter~7]{Darwiche09}. 
The algorithm assumes TBNs with threshold-based selection, but can be easily adjusted to handle sigmoid-based selection.

\subsection{Factor Elimination}

A {\em factor} \(f(\X)\) is a mapping from instantiations \(\x\) into positive, real numbers. Factor elimination is a variation on variable elimination \citep{zhangJAIR96a,dechterUAI96}, in which factors are systematically eliminated according to the following process. Consider a factor \(f(\X)\), where variables \(\Y \subseteq \X\) appear only in this factor. To eliminate this factor, we simply sum-out variables \(\Y\) from the factor, leading to \(\sum_\Y f(\X)\), and then multiply the result by some other factor.\footnote{Summing-out and multiplication are standard operations on factors; see, e.g., \cite[Chapter 6]{Darwiche09}.} We can use this method to compute the marginal on any query variable \(Q\) in a BN. We start with the network CPTs as our initial factors. We then identify a factor that contains \(Q\) and call it the {\em root factor.} We successively eliminate all factors, one by one, except for the root factor. At this point, we are left with a single factor \(f(\Z,Q)\), where \(g(Q) = \sum_\Z f(\Z,Q)\) is the marginal on \(Q\). Any elimination order works, but the specific order used impacts the algorithm's complexity.

We can assert soft evidence \((\lambda_1,\ldots,\lambda_k)\) on a variable \(X\) by adjusting its CPT as follows.  The entry of each row corresponding to value \(x_i\) of \(X\) is multiplied by \(\lambda_i\).  If we have evidence, the computed factor \(g(Q)\) needs to be normalized to obtain the posterior distribution on variable \(Q\).

\subsection{Symbolic Factor Elimination}

A TBN query can be compiled into a TAC by keeping a symbolic trace of the factor elimination algorithm.  This requires working with \emph{symbolic} factors, whose entries are circuit nodes instead of numbers, and is analogous to compiling a BN into an AC by keeping a symbolic trace of elimination algorithms \citep{Chavira.Darwiche.Ijcai.2007,Darwiche09}. 
We next illustrate these concepts by providing an example of compiling an AC for a BN query. We then show how the technique can be extended to compiling TACs, also using a concrete example. 
We then follow by a formal statement of the compilation algorithm and its complexity.

Consider a BN over binary variables \(A, B, C\) with edges \(A \to B\) and \(A \to C\) and the following CPTs.
\begin{center}
\small
\begin{tabular}{c|c}
\(A\)      & \(\Theta_A\) \\\hline
\(a\)       & \(\theta_a\) \\
\(\n(a)\) & \(\theta_{\n(a)}\) \\
\end{tabular}
\quad
\begin{tabular}{cc|c}
\(A\)     &     \(B\) &  \(\Theta_{B|A}\) \\\hline
\(a\)     &     \(b\)  & \(\theta_{b|a}\) \\
\(a\)     &    \(\n(b)\) & \(\theta_{\n(b)|a}\) \\
\(\n(a)\) &     \(b\) & \(\theta_{b|\n(a)}\) \\
\(\n(a)\) & \(\n(b)\) & \(\theta_{\n(b)|\n(a)}\) \\
\end{tabular}
\quad
\begin{tabular}{cc|c}
\(A\)     &     \(C\) &  \(\Theta_{C|A}\) \\\hline
\(a\)     &     \(c\)  & \(\theta_{c|a}\) \\
\(a\)     &    \(\n(c)\) & \(\theta_{\n(c)|a}\) \\
\(\n(a)\) &     \(c\) & \(\theta_{c|\n(a)}\) \\
\(\n(a)\) & \(\n(c)\) & \(\theta_{\n(c)|\n(a)}\) \\
\end{tabular}
\end{center}
Suppose we have soft evidence on variables \(A\) and \(C\) and want to compile an AC that computes the marginal \((\pstar(b),\pstar(\n(b)))\).  
We start with the following CPTs that incorporate evidence (we use \(\sumnode(n_1,n_2)\) and \(\prodnode(n_1,n_2)\) to denote sum and product circuit nodes with \(n_1\) and \(n_2\) as their children).
\begin{center}
\small
\begin{tabular}{c|c}
\(A\) & \(\Theta_A\) \\\hline
\(a\)     & \(n_1  = \prodnode(\lambda_a,\theta_a)\) \\
\(\n(a)\) & \(n_2 = \prodnode(\lambda_{\n(a)},\theta_{\n(a)})\) \\
\end{tabular}
\quad
\begin{tabular}{cc|c}
\(A\)     &     \(B\) &  \(\Theta_{B|A}\) \\\hline
\(a\)     &     \(b\)  & \(\theta_{b|a}\) \\
\(a\)     &    \(\n(b)\) & \(\theta_{\n(b)|a}\) \\
\(\n(a)\) &     \(b\) & \(\theta_{b|\n(a)}\) \\
\(\n(a)\) & \(\n(b)\) & \(\theta_{\n(b)|\n(a)}\) \\
\end{tabular}
\quad
\begin{tabular}{cc|c}
\(A\)     &     \(C\)   & \(\Theta_{C|A}\) \\\hline
\(a\)     &     \(c\)     & \(n_3 = \prodnode(\lambda_c,\theta_{c|a})\) \\
\(a\)     &    \(\n(c)\) & \(n_4 = \prodnode(\lambda_{\n(c)},\theta_{\n(c)|a})\) \\
\(\n(a)\) &     \(c\)    & \(n_5 = \prodnode(\lambda_c,\theta_{c|\n(a)})\) \\
\(\n(a)\) & \(\n(c)\)   & \(n_6 = \prodnode(\lambda_{\n(c)},\theta_{\n(c)|\n(a)})\) \\
\end{tabular}
\end{center}
Our root factor is \(\Theta_{B|A}\). We will eliminate factor \(\Theta_{C|A}\) first. 
Summing out variable \(C\) from the factor gives
\begin{center}
\small
\begin{tabular}{c|c}
\(A\) & \(\sum_C \Theta_{C|A}\) \\\hline
\(a\)      & \(n_{7}  = \sumnode(n_3,n_4)\) \\
\(\n(a)\) & \(n_{8} = \sumnode(n_5,n_6)\) \\
\end{tabular}
\end{center}
Multiplying the result by factor \(\Theta_A\), we get
\begin{center}
\small
\begin{tabular}{c|c}
\(A\) & \(\Theta_A \sum_C \Theta_{C|A}\) \\\hline
\(a\)      & \(n_{9}  = \prodnode(n_1,n_7)\) \\
\(\n(a)\) & \(n_{10} = \prodnode(n_2,n_8)\) \\
\end{tabular}
\end{center}
Our set of factors are now
\begin{center}
\small
\begin{tabular}{c|c}
\(A\) & \(\Theta_A \sum_C \Theta_{C|A}\) \\\hline
\(a\)      & \(n_{9}  = \prodnode(n_1,n_7)\) \\
\(\n(a)\) & \(n_{10} = \prodnode(n_2,n_8)\) \\
\end{tabular}
\quad
\begin{tabular}{cc|c}
\(A\)     &     \(B\) &  \(\Theta_{B|A}\) \\\hline
\(a\)     &     \(b\)  & \(\theta_{b|a}\) \\
\(a\)     &    \(\n(b)\) & \(\theta_{\n(b)|a}\) \\
\(\n(a)\) &     \(b\) & \(\theta_{b|\n(a)}\) \\
\(\n(a)\) & \(\n(b)\) & \(\theta_{\n(b)|\n(a)}\) \\
\end{tabular}
\end{center}
Eliminating the recently constructed factor, we now have
\begin{center}
\small
\begin{tabular}{cc|c}
\(A\)     &     \(B\) &  \(\Theta_{B|A}\Theta_A \sum_C \Theta_{C|A}\) \\\hline
\(a\)     &     \(b\)      & \(n_{11} = \prodnode(n_9,\theta_{b|a})\) \\
\(a\)     &    \(\n(b)\) & \(n_{12} = \prodnode(n_9,\theta_{\n(b)|a})\) \\
\(\n(a)\) &     \(b\)    & \(n_{13} = \prodnode(n_{10},\theta_{b|\n(a)})\) \\
\(\n(a)\) & \(\n(b)\)  & \(n_{14} = \prodnode(n_{10},\theta_{\n(b)|\n(a)})\) \\
\end{tabular}
\end{center}
Summing out variable \(A\) gives
\begin{center}
\small
\begin{tabular}{c|c}
\(B\) & \(\sum_A \Theta_{B|A}\Theta_A \sum_C \Theta_{C|A}\) \\\hline
\(b\)     & \(n_{15} = \sumnode(n_{11},n_{13})\) \\
\(\n(b)\) & \(n_{16} = \sumnode(n_{12},n_{14})\) \\
\end{tabular}
\end{center}
We have now compiled a circuit whose outputs, \(n_{15}\) and \(n_{16}\), give the marginal for variable \(B\); see Figure~\ref{fig:ac}.  
The inputs to this circuits are the vectors \((\lambda_a,\lambda_{\n(a)})\) and \((\lambda_c,\lambda_{\n(c)})\) representing soft evidence, 
in addition to network parameters.  Dividing nodes \(n_{15}\) and \(n_{16}\) by their sum yields the posterior \((\Pstar(b),\Pstar(\n(b)))\).

\begin{figure}[t]
  \centering
 \includegraphics[width=0.35\linewidth]{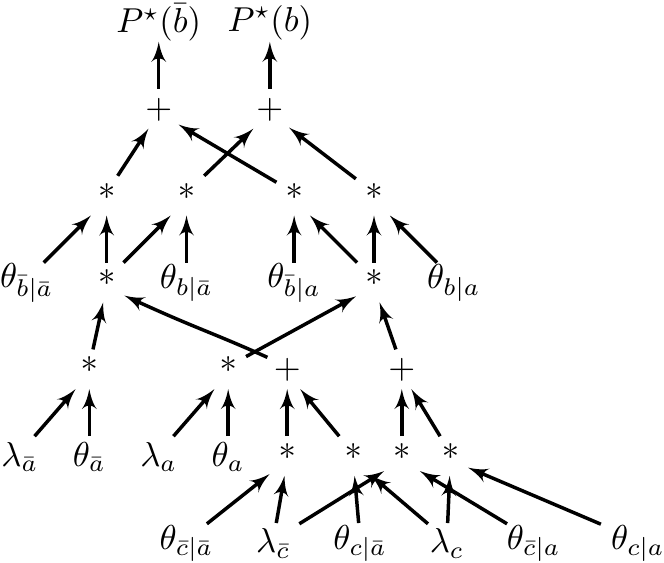}
 \caption{An AC compiled from a BN query.}
 \label{fig:ac}
\end{figure}

One can prune the Bayesian network before compiling an AC. In particular, one can successively remove leaf variables as long as they are not query or evidence variables. 

\subsection{Compiling TACs via Symbolic Factor Elimination}

Beyond sum and product nodes, TACs also use \emph{testing nodes} \(\testnode(p,t,n^+,n^-)\), which pass through the value of \(n^+\) if \(p \geq t\) and the value of \(n^-\) otherwise (like a multiplexer).

Compiling TACs is similar to compiling ACs but requires two phases. In the first phase, we {\em select} CPTs for all testing nodes that are relevant to the query. In the second phase, we perform classical {\em inference} as in the previous example. The selection phase converts testing CPTs into regular CPTs using an operation called {\em flattening.} 
The flattening of a CPT for variable \(X\) involves only its ancestral CPTs. In particular, before the CPT for variable \(X\) is flattened, all its ancestral CPTs must be flattened and eliminated. 

Consider our earlier example over binary variables \(A, B, C\) with edges \(A \to B\) and \(A \to C\), and suppose now that variable \(B\) is testing. 
We have soft evidence on variables \(A\) and \(C\) and want to compile a TAC that computes the marginal \((\pstar(b),\pstar(\n(b)))\). 

\vspace{2mm}
\noindent{\bf Selection Phase}
As we want to flatten the CPT for variable \(B\), the CPT for variable \(C\) is irrelevant, so we start with the following CPTs.
\begin{center}
\small
\begin{tabular}{c|c}
\(A\) & \(\Theta_A\) \\\hline
\(a\)     & \(n_1  = \prodnode(\lambda_a,\theta_a)\) \\
\(\n(a)\) & \(n_2 = \prodnode(\lambda_{\n(a)},\theta_{\n(a)})\) \\
\end{tabular}
\quad
\begin{tabular}{cc|c|ccc}
\(A\)     &     \(B\) & \multicolumn{4}{c}{\(\Theta^?_{B|A}\)}   \\\hline
\(a\)     &     \(b\)  & \(\)  & \(T_{B|a}\) & \(\theta^+_{b|a}\) & \(\theta^-_{b|a}\) \\
\(a\)     & \(\n(b)\) & \(\)  & \(T_{B|a}\) &  \(\theta^+_{\n(b)|a}\) &  \(\theta^-_{\n(b)|a}\) \\
\(\n(a)\) &     \(b\) & \(\)  & \(T_{B|\n(a)}\) & \(\theta^+_{b|\n(a)}\) & \(\theta^-_{b|\n(a)}\) \\
\(\n(a)\) & \(\n(b)\) & \(\) & \(T_{B|\n(a)}\) & \(\theta^+_{\n(b)|\n(a)}\) & \(\theta^-_{\n(b)|\n(a)}\) \\
\end{tabular}
\end{center}
The CPT for variable \(A\) is regular so it is notated as usual. The CPT for variable \(B\) is testing so it is notated differently.
The second column is empty initially but will contain TAC nodes when factors are multiplied into the CPT. The third
column contains thresholds and dynamic parameters and is kept until the CPT is flattened (i.e., converted into a regular CPT).

We first eliminate the CPT for variable \(A\). No variables are summed out in this case, so we just multiply
this CPT by the testing CPT for variable \(B\), leading to
\begin{center}
\small
\begin{tabular}{cc|c|ccc}
\(A\)     &     \(B\) & \multicolumn{4}{c}{\(\Theta_A \Theta^?_{B|A}\)}   \\\hline
\(a\)     &     \(b\)  & \(n_1\)  & \(T_{B|a}\) & \(\theta^+_{b|a}\) & \(\theta^-_{b|a}\) \\
\(a\)     & \(\n(b)\) & \(n_1\)  & \(T_{B|a}\) &  \(\theta^+_{\n(b)|a}\) &  \(\theta^-_{\n(b)|a}\) \\
\(\n(a)\) &     \(b\) & \(n_2\)  & \(T_{B|\n(a)}\) & \(\theta^+_{b|\n(a)}\) & \(\theta^-_{b|\n(a)}\) \\
\(\n(a)\) & \(\n(b)\) & \(n_2\) & \(T_{B|\n(a)}\) & \(\theta^+_{\n(b)|\n(a)}\) & \(\theta^-_{\n(b)|\n(a)}\) \\
\end{tabular}
\end{center}
If we define a new TAC node \(n_3 = \sumnode(n_1,n_2)\), then \(n_1/n_3\) and \(n_2/n_3\) represent the posterior on variable \(A\),
which is what we need to flatten the CPT for variable \(B\) as follows\footnote{More generally: after flattening and eliminating all ancestral CPTs of a variable \(X\),
the CPT for \(X\) will contain the marginal on its parents given evidence on its ancestors. This marginal is all we needed to
select a CPT for variable \(X\), whether we are using threshold-based or sigmoid-based selection.}
\begin{center}
\small
\begin{tabular}{cc|c}
\(A\)     &     \(B\)  & \(\Theta_{B|A}\) \\\hline
\(a\)     &     \(b\)   & \(n_4 = \testnode(n_1,{\prodnode(T_{B|a},n_3)},\theta^+_{b|a},\theta^-_{b|a})\) \\
\(a\)     & \(\n(b)\)  & \(n_5 = \testnode(n_1,{\prodnode(T_{B|a},n_3)},\theta^+_{\n(b)|a},\theta^-_{\n(b)|a})\) \\
\(\n(a)\) &     \(b\)  & \(n_6 = \testnode(n_2,{\prodnode(T_{B|\n(a)},n_3)},\theta^+_{b|\n(a)},\theta^-_{b|\n(a)})\) \\
\(\n(a)\) & \(\n(b)\) & \(n_7 = \testnode(n_2,{\prodnode(T_{B|\n(a)},n_3)},\theta^+_{\n(b)|\n(a)},\theta^-_{\n(b)|\n(a)})\) \\
\end{tabular}
\end{center}
This is a regular CPT since each entry is a circuit node. This finishes the selection phase, leading to the partial TAC in Figure~\ref{fig:tac-a}. We are now ready for the inference phase.

\begin{figure*}[t]
 \centering
  \subfigure[Selection]{\label{fig:tac-a}
    \includegraphics[width=0.42\linewidth]{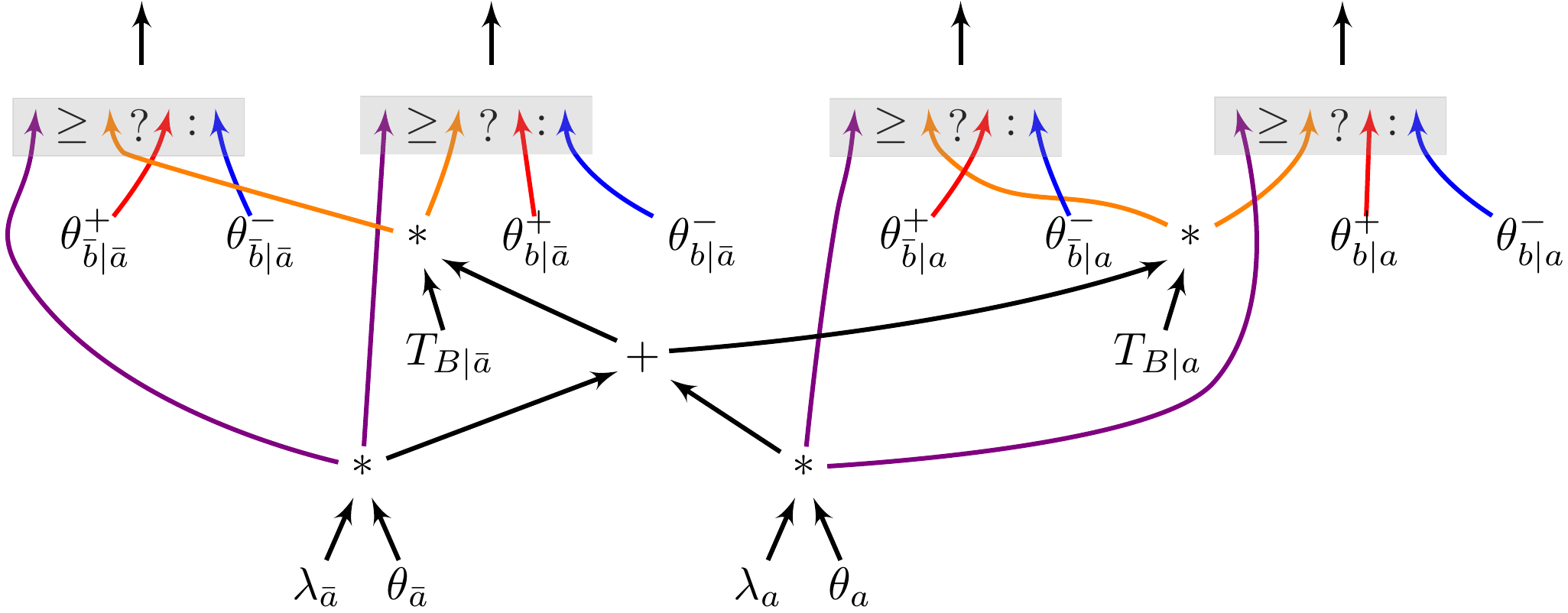}}
  \subfigure[Inference]{\label{fig:tac-b}
    \includegraphics[width=0.54\linewidth]{figs-aaai/tac2-crop-box-p}}
 \caption{A TAC compiled from a TBN query. 
 \label{fig:tacs}}
\end{figure*}

\vspace{2mm}
\noindent{\bf Inference Phase}
Our CPTs are now all regular
\begin{center}
\small
\begin{tabular}{c|c}
\(A\) & \(\Theta_A\) \\\hline
\(a\)     & \(n_1\) \\
\(\n(a)\) & \(n_2\) \\
\end{tabular}
\qquad
\begin{tabular}{cc|c}
\(A\)     &     \(B\)  & \(\Theta_{B|A}\) \\\hline
\(a\)     &     \(b\)   & \(n_4\) \\
\(a\)     & \(\n(b)\)  & \(n_5\) \\
\(\n(a)\) &     \(b\)  & \(n_6\) \\
\(\n(a)\) & \(\n(b)\) & \(n_7\) \\
\end{tabular}
\qquad
\begin{tabular}{cc|c}
\(A\)     &     \(C\)   & \(\Theta_{C|A}\) \\\hline
\(a\)     &     \(c\)     & \(n_8\) \\
\(a\)     &    \(\n(c)\) & \(n_9\) \\
\(\n(a)\) &     \(c\)    & \(n_{10}\) \\
\(\n(a)\) & \(\n(c)\)   & \(n_{11}\) \\
\end{tabular}
\end{center}
Inference can proceed as in the previous section: eliminate factor \(\Theta_{C|A}\) then factor \(\Theta_A\). 
To eliminate factor \(\Theta_{C|A}\), we sum-out variable \(C\) from the factor leading to
\begin{center}
\small
\begin{tabular}{c|c}
\(A\) & \(\sum_C \Theta_{C|A}\) \\\hline
\(a\)      & \(n_{12}  = \sumnode(n_8,n_9)\) \\
\(\n(a)\) & \(n_{13} = \sumnode(n_{10},n_{11})\) \\
\end{tabular}
\end{center}
We then multiply the above factor by factor \(\Theta_A\). Our new set of factors is
\begin{center}
\small
\begin{tabular}{c|c}
\(A\) & \(\Theta_A \sum_C \Theta_{C|A}\) \\\hline
\(a\)      & \(n_{14}  = \prodnode(n_{12},n_1)\) \\
\(\n(a)\) & \(n_{15} = \prodnode(n_{13},n_2)\) \\
\end{tabular}
\qquad
\begin{tabular}{cc|c}
\(A\)     &     \(B\)  & \(\Theta_{B|A}\) \\\hline
\(a\)     &     \(b\)   & \(n_4\) \\
\(a\)     & \(\n(b)\)  & \(n_5\) \\
\(\n(a)\) &     \(b\)  & \(n_6\) \\
\(\n(a)\) & \(\n(b)\) & \(n_7\) \\
\end{tabular}
\end{center}
To eliminate the last constructed factor, we just multiply it by the factor for variable \(B\), leading to
\begin{center}
\small
\begin{tabular}{cc|c}
\(A\)     &     \(B\)  & \(\Theta_{B|A}\Theta_A \sum_C \Theta_{C|A}\) \\\hline
\(a\)     &     \(b\)   & \(n_{16} = \prodnode(n_{14},n_4)\) \\
\(a\)     & \(\n(b)\)  & \(n_{17} = \prodnode(n_{14},n_5)\) \\
\(\n(a)\) &     \(b\)  & \(n_{18} = \prodnode(n_{15},n_6)\) \\
\(\n(a)\) & \(\n(b)\) & \(n_{19} = \prodnode(n_{15},n_7)\) \\
\end{tabular}
\end{center}
Summing-out variable \(A\), we now have the final factor
\begin{center}
\small
\begin{tabular}{c|c}
\(B\) & \(\sum_A \Theta_{B|A}\Theta_A \sum_C \Theta_{C|A}\) \\\hline
\(b\)       & \(n_{20} = \sumnode(n_{16},n_{18})\) \\
\(\n(b)\) & \(n_{21}  = \sumnode(n_{17},n_{19})\) \\
\end{tabular}
\end{center}

We now have a TAC whose two outputs, \(n_{20}\) and \(n_{21}\), compute the marginal \((\pstar(b),\pstar(\n(b)))\); see Figure~\ref{fig:tac-b}. 
Normalizing the circuit outputs gives the posterior \((\Pstar(b),\Pstar(\n(b)))\). 

\subsection{The Compilation Algorithm}
We now provide a more formal description of the compilation algorithm. 

\vspace{2mm}
\noindent{\bf Testing CPTs} 
A testing CPT for node \(X\) with parents \(\U\) will initially have rows of the following form:
\begin{center}
\begin{tabular}{|c|c|ccc|} \hline
\(\u\: x\) &  & \(T_{X|\u}\) & \(\theta^+_{x|\u}\) & \(\theta^-_{x|\u}\) \\ \hline
\end{tabular}
\end{center}
When factors are multiplied into this testing CPT, its rows will have the following form
\begin{center}
\begin{tabular}{|c|c|ccc|} \hline
\(\u\: x\) & \(n\) & \(T_{X|\u}\) & \(\theta^+_{x|\u}\) & \(\theta^-_{x|\u}\) \\ \hline
\end{tabular}
\end{center}
where \(n\) is a TAC node. When the CPT is flattened (see below), its rows will have the following form
\begin{center}
\begin{tabular}{|c|c|} \hline
\(\u\: x\) & \(\testnode(n,{\prodnode(m,T_{X|\u})},\theta^+_{x|\u},\theta^-_{x|\u})\) \\ \hline
\end{tabular}
\end{center}
where \(n\) and \(m\) are TAC nodes. It is now a regular CPT since the entry in each row is a testing TAC node. 

\vspace{2mm}
\noindent{\bf Pseudocode}
To compile a TBN query into a TAC (\(Q\) is the query node):
\begin{enumerate}
\item {\em Pruning.} Repeatedly remove a leaf TBN node if it is not an evidence or query node. \label{alg:prune}
\item {\em Initialization.} Replace each TBN parameter and threshold by a corresponding TAC node.
\item  {\em Entering Evidence.} If a node \(X\) has soft evidence and \(k\) values, construct TAC nodes for \(\lambda_1, \ldots, \lambda_k\)
and multiply each row in the CPT of \(X\) by \(\lambda_i\) if the row corresponds to value \(x_i\). 
\item {\em Selection Phase.} Visit testing nodes \(X\), ancestors before descendants. 
\begin{enumerate}
\item Let \(\Sigma\) be the set of CPTs for \(X\) and its ancestors (ancestral CPTs must be regular at this point).
\item Eliminate all factors from \(\Sigma\) except the CPT for node \(X\). \label{alg:e1}
\item Flatten the CPT for node \(X\) with parents \(\U\) as follows:
\begin{enumerate}
\item Create a TAC node \(n=\sum_\u n_\u\), where \(n_\u\) is the TAC node in some CPT row for state \(\u\).\footnote{CPT rows for the same instantiation \(\u\)
of parents \(\U\) must have the same TAC node (represents the marginal on state \(\u\)).}
\item Replace each CPT row
\begin{center}
\begin{tabular}{|c|c|ccc|} \hline
\(\u\: x\) & \(n_\u\) & \(T_{X|\u}\) & \(\theta^+_{x|\u}\) & \(\theta^-_{x|\u}\) \\ \hline
\end{tabular}
\end{center}
by
\begin{center}
\begin{tabular}{|c|c|} \hline
\(\u\: x\) & \(\testnode(n_\u,{\prodnode(n,T_{X|\u})},\theta^+_{x|\u},\theta^-_{x|\u})\) \\ \hline
\end{tabular}
\end{center}
\end{enumerate}
\end{enumerate}
(the CPT for node \(X\) is now regular).
\item {\em Inference Phase.}
\begin{enumerate}
\item Let \(\Sigma\) be the set of all TBN CPTs (whether initially regular or flattened).
\item Let \(f\) be a factor in \(\Sigma\) that includes query node \(Q\).
\item Eliminate all factors in \(\Sigma\) except for factor \(f\). \label{alg:e2}
\item Return \(\sum_\Z f\), where \(\Z\) are the variables of factor \(f\) except \(Q\).
\end{enumerate}
\end{enumerate}
Entries of the returned factor will contain TAC outputs, which compute the marginal \(\pstar(Q)\).
Normalizing these outputs gives the posterior \(\Pstar(Q)\).

\vspace{2mm}
\noindent{\bf The Order of Eliminating Factors}
The factor elimination algorithm is an abstraction of the jointree algorithm~\citep[Chapter 7]{Darwiche09}. In particular, a message sent by the jointree algorithm from cluster~\(i\) to cluster~\(j\) can be interpreted as eliminating the
factor at cluster~\(i\), leading to its multiplication by the factor at cluster~\(j\). Hence, the order in which factors are eliminated can be set using a jointree, which also determines the complexity of factor elimination. 
To eliminate a set of factors \(\Sigma\) except for some factor \(f \in \Sigma\), we first assign each factor \(g \in \Sigma\) to some jointree cluster that contains the variables of factor \(g\).
If multiple factors are assigned to the same cluster, we view this as one factor corresponding to their product. We next designate the cluster \(r\) containing factor \(f\) as the root cluster. 
We finally eliminate factors as follows. We only eliminate a factor if it is attached to a leaf cluster \(i \neq r\), multiplying it into its neighboring factor (cluster \(i\) is also removed). 
The process terminates when we are left with the root cluster \(r\).
Using a binary jointree with \(n\) clusters and width \(w\), the time and space complexity of factor elimination is \(O(n \exp(w))\).\footnote{A jointree
is {\em binary} if each cluster has at most three neighbors. The {\em width} of a jointree is the size of its largest cluster \(-1\). See~\citep[Chapter 7]{Darwiche09} for details.}

\vspace{2mm}
\noindent{\bf Complexity}
The elimination of factors in Steps~\ref{alg:e1} and~\ref{alg:e2} can be implemented using the same jointree, constructed for the TBN produced by Step~\ref{alg:prune}.
Given a binary jointree with \(n\) clusters and width \(w\), the time and space complexity of compiling a TAC (the selection and inference phases) is then \(O(m \cdot n \exp(w))\), 
where \(m\) is the number of testing nodes. This can be improved if one uses a different jointree for each iteration of the selection phase since each of these iterations involves a subset of the TBN that may be less connected.

\vspace{2mm}
\noindent{\bf Unique Nodes}
The described algorithm may construct equivalent TAC nodes. To avoid this, we maintain a {\em unique-node table,} which stores each constructed TAC node with a key based on its type and its children's identities. Before a new node is constructed, the table is checked if an equivalent node was constructed earlier.  If the test succeeds, the node is used. Otherwise, a new node is constructed and added to the table.

\vspace{2mm}
\noindent{\bf Sigmoid Selection}
The algorithm can be slightly adjusted to compile TBNs that employ other CPT selection mechanisms, as long as the selection process requires only the posterior on parents (sigmoid selection satisfies
this property). We just adjust the operation of CPT flattening to construct and add the desired selection nodes instead of testing nodes.

\bibliographystyle{model5-names}\biboptions{authoryear}

\end{document}